\documentclass{article}

\usepackage{arxiv}
\usepackage{ccicons}

\usepackage[utf8]{inputenc} 
\usepackage[T1]{fontenc}    
\usepackage{hyperref}       
\usepackage{url}            
\usepackage{booktabs}       
\usepackage{amsfonts}       
\usepackage{nicefrac}       
\usepackage{microtype}      
\usepackage{amssymb}        
\usepackage{lipsum}
\usepackage{subfigure}
\usepackage{graphicx}
\usepackage{marvosym}
\usepackage{wrapfig} 
\usepackage{amsmath}
\usepackage{empheq}
\usepackage{orcidlink}
\usepackage{subcaption}
\graphicspath{ {./images/} }

\usepackage{xcolor}
\hypersetup{hidelinks,
	colorlinks=true, linkcolor=magenta, urlcolor=cyan!60!black, citecolor=purple,
	pdfstartview=Fit,
	breaklinks=true}

\newcommand{\rbox}[1]{%
\tikz[baseline=(char.base)]{
\node[
    draw,
    rounded corners=2pt,
    inner sep=1pt
] (char) {#1};
}
}

\usepackage{amsthm}
\theoremstyle{plain}
\newtheorem{theorem}{Theorem}
\newtheorem{proposition}{Proposition}
\newtheorem{lemma}{Lemma}

\theoremstyle{definition}
\newtheorem{definition}{Definition}

\theoremstyle{remark}

\usepackage{fancybox}
\usepackage{tikz}

\usepackage{multirow}
\usepackage{multicol}
\usepackage{booktabs}
\usepackage{appendix}

\usepackage{mathtools}
\usepackage{listings}
\usepackage{caption}
\usepackage{siunitx}
\usepackage{algorithm}
\usepackage{algorithmic}
\usepackage{tikz}

\usepackage[most]{tcolorbox}
\tcbset{
    boxsep=0mm,
    boxrule=0pt,
    arc=0mm,
    outer arc=0pt,
    left=1pt,
    right=1pt,
    top=1pt,
    bottom=1pt
}

\definecolor{light}{RGB}{240, 240, 240} 
\definecolor{lightx}{RGB}{250, 250, 250}

\newtcolorbox{lightgraybox}{
    enhanced,
    frame hidden,
    sharp corners,
    colback=light,
    borderline={3pt}{-3pt}{light},
}

\definecolor{proofgray}{RGB}{190, 190, 190}

\newtcolorbox{proofbox}{
    enhanced,
    breakable,
    frame hidden,
    sharp corners,
    colback=lightx,
    borderline west={2pt}{0pt}{proofgray},
    left=8pt,
    right=8pt,
    top=4pt,
    bottom=4pt,
}

\title{On the (\emph{Generative}) Linear Sketching Problem\thanks{
\ccby{} This manuscript is licensed under a  \href{https://creativecommons.org/licenses/by/4.0/}{Creative Commons Attribution 4.0 International (CC BY 4.0) License}. \emph{Yan Qiao and Zonghui Wang are the co-corresponding authors.} This paper was supported by the National Natural Science Foundation of China (NSFC) under Grant No. 62572168, the Anhui Provincial Natural Science Fund under Grant No. 2508085MF151, and in part by the Key Research and Development Program of Zhejiang Province under Grant No. 2025C02103.
}}

\author{
 Xinyu Yuan~\orcidlink{0009-0005-6287-5712} \\
  College of Computer Science and Technology\\
  Zhejiang University\\
  Hangzhou, China\\
  \texttt{yxy5315@gmail.com} \\
   \And
 Yan Qiao~\orcidlink{0000-0002-4407-1762} \\
  School of Computer Science and Infomation Engineering\\
  Hefei University of Technology\\
  Hefei, China\\
  \Letter~\texttt{qiaoyan@hfut.edu.cn} \\
  \And
 Zonghui Wang~\orcidlink{0000-0002-8212-4618} \\
  College of Computer Science and Technology\\
  Zhejiang University\\
  Hangzhou, China\\
  \Letter~\texttt{zhwang@zju.edu.cn} \\
  \And
 Wenzhi Chen~\orcidlink{0000-0003-1674-4701} \\
  College of Computer Science and Technology\\
  Zhejiang University\\
  Hangzhou, China\\
  \texttt{chenwz@zju.edu.cn} \\
}

\begin{document}
\maketitle

\begin{abstract}
Sketch techniques have been extensively studied in recent years and are especially well-suited to data streaming scenarios, where the sketch summary is updated quickly and compactly. However, it is challenging to recover the current state from these summaries in a way that is accurate, fast, and real. In this paper, we seek a new solution that reconciles this tension, aiming for near-perfect recovery with lightweight computational procedures. Focusing on {linear sketching} problems of the form $\boldsymbol{\Phi}f \rightarrow f$, our study proceeds in three phases. First, we revisit existing techniques and show the root cause of the dilemma: an {orthogonal} information loss during the summarization process. Second, we investigate how {generative} priors can be leveraged to bridge the gap and thereby allow theoretically superior recovery. Third, we present \texttt{FLORE}, the first generative sketching paradigm to embrace these analyses.
At its core, \texttt{FLORE} integrates invertible flows with efficient sketch instances to achieve the best of two worlds.
More importantly, \texttt{FLORE} can be trained online without access to ground-truth data, which is readily deployable in practice. 
Comprehensive experiments demonstrate \texttt{FLORE}’s ability to provide high-quality recovery, and support summary with low computing overhead, outperforming previous methods by up to $10^3\times$ in error reduction and $10^2\times$ in processing speed compared to learning-based solutions. To aid future research, our experimental codebase is available at \includegraphics[height=0.9em]{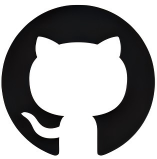}~\href{https://github.com/Y-debug-sys/FLORE}{\textsc{Github}}, and streaming datasets are available at \includegraphics[height=0.81em]{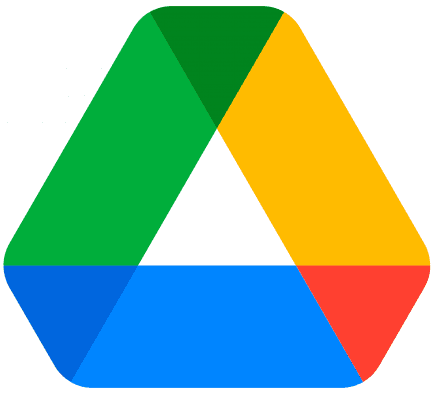}~\href{https://drive.google.com/file/d/1lEYtnUl3bDaJXiaE02mgV_8CUGpvXuz9/view}{\textsc{Google Drive}}. 
\end{abstract}


\tableofcontents


\section{Introduction}
\label{sec:introduction}

Accurately measuring the volume of water from a fire hose while being directly hit by it is nearly an impossible feat. In essence, this captures the central difficulty of streaming data analysis: the data item arrives in a continuous, unrelenting torrent, so instead we need to process each observation quickly to create a summary of the current state. 
Here’s a simple example—for a stream of numbers, obtaining the cumulative sum requires storing only a single running total rather than all past inputs.
Such a summary is referred to as a \textit{sketch} of the data, which has been extensively studied~\cite{krishnamurthy2003sketch,cormode2017data,larsen2019heavy}, since it is very helpful in any situation where people are monitoring a database or network that is being updated continuously.  
However, 
due to some (hashing) collisions, 
coping with the vast scale of information means making compromises—\textit{the description of the world via data sketching is approximate rather than exact.} A particular challenge has been to design algorithms that navigate tradeoffs effectively, like sacrificing speed for accuracy or memory consumption for reliability~\cite{bloom1970space}. The challenge gets harder, though, when the questions about the data stream get more complicated, e.g., heavy hitters~\cite{misra1982finding}, entropy~\cite{chakrabarti2006estimating}, and graph summarization~\cite{tang2016graph}. 

\begin{figure}
    \centering
    \includegraphics[width=1.\linewidth]{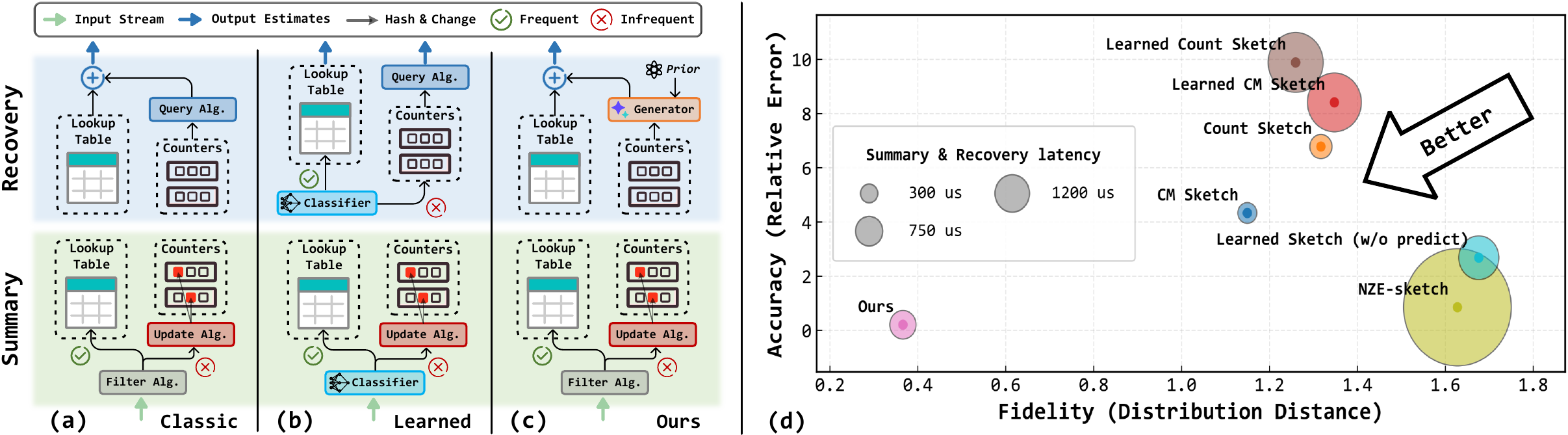}
    \caption{\textbf{Comparison of our generative solution and existing solutions.} From (a) to (c), we present an overview of the full pipelines of three sketching paradigms. In (d), our proposal achieves the best trade-off among accuracy, fidelity and speed under a 1MB sketch memory budget on the CAIDA-2018 dataset.}
    \label{fig:difference}
\end{figure}

\begin{wrapfigure}[12]{r}{0.57\textwidth}
  \vspace{-0.54cm}
  \begin{center}
  \includegraphics[width=1.\linewidth]{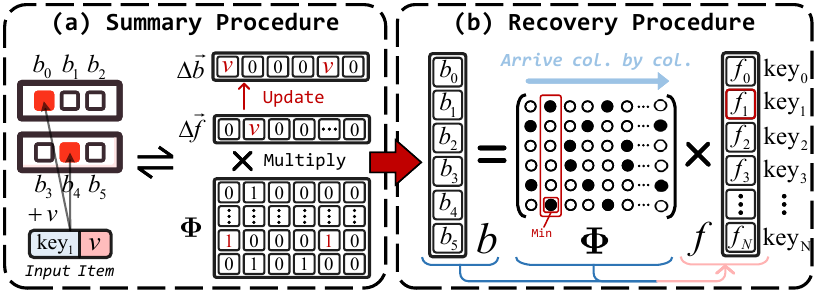}
\caption{\textbf{Matrix formulation of Count-Min.} The linear sketching problem can be analyzed via simple matrix operations.}
\label{fig:sketch_sensing}
\end{center}
\end{wrapfigure}

Drawing the mathematical form from compressive sensing (\texttt{CS})~\cite{cormode2005towards}, this work interprets sketch as a coding problem, i.e., it is encoding sparse information, such as heavy-tailed traffic, down to compact summary and trying to extract information that lets them recover what was ``squeezed in'' initially (\S\ref{subsec:formu}). Further, if the coding procedure is able to be represented as a linear transform of the input, i.e., $b=\boldsymbol{\Phi}f$, it results in what we call a linear sketching problem~\cite{cormode2011sketch,li2014turnstile}. 
From this perspective, we start by analyzing a fundamental tension between sketching cost and provable accuracy for existing techniques (\S\ref{subsec:lim}). And we reveal that the root cause lies in an inherent \textit{deficiency of orthogonal information} that cannot be recovered without any prior knowledge (\S\ref{subsec:fail}). As a result, the central challenge shifts toward compensating for the missing stream component that meets both realness and efficiency requirements. 

In order to overcome the above challenge, this work turns to a brand-new machine learning (ML) paradigm empowered by \textit{generative} models (GMs). Our key insight here is two-fold: (i) Theoretically, \textit{by enforcing alignment with the learned distributional prior knowledge, GMs may help mitigate the aforementioned information loss} from aggressive summarization. 
By formulating the learned prior assumptions as a probability distribution, the recovery procedure can be reformulated as the problem of approximating the posterior distribution $p(f \mid b)$. (ii) Systematically, since data-plane summary is on the critical path, as shown in Fig.~\ref{fig:difference}(c), \textit{deploying GMs for only recovery in the control plane may avoid burdening the overall system}.

We study the generative linear sketching problem along two lines. In the \textit{first line} (\S\ref{subsec:igp}), we attempt to examine prevalent GMs on synthetic but representative streaming scenarios. Our preliminary experiments show that all GMs fit low-cardinality stream distributions ($<$10K unique keys) well; however, it is non-trivial to scale these GMs other than flow-based generative models (FGMs) to real-world applications, due to limitations of their own (e.g., unstable training, slow inference, and poor expressiveness). This motivates the \textit{second line} (\S\ref{subsec:flore}) of our work that leverages FGM to develop a novel and practical methodology for solving problems of interest. Specifically, we propose \texttt{FLORE}, the \textit{first} deep generative sketching framework for fast summarization and accurate recovery while preserving fidelity. 
As shown in Figure~\ref{fig:difference}(d), experiments on the CAIDA-2018 dataset (d) show that, compared with both existing learning-based and classical approaches, \texttt{FLORE} simultaneously achieves high speed (throughput), fidelity (WMRE), and accuracy (ARE).

To realize \texttt{FLORE}, there are the following four core design components. (i) We adopt the data-control plane separation design to build a deployable ML system for \texttt{FLORE}. (ii) Instead of designing from scratch, we utilize the analysis results as a guideline: \texttt{FLORE} is purposefully modeled as an invertible solver to separate and compensate for the lost information from orthogonal subspaces. (iii) Since it is nearly impossible to collect ground-truth (GT) labels~\cite{cao2024learning}, we demonstrate that \texttt{FLORE} can be trained by accessing only aggregated counters, based on a well-designed Expectation Maximization (EM) algorithm. (iv) \texttt{FLORE} can easily support long-term, lightweight operation for high-speed, high-volume data streams without frequent model changes (see simulations in \S\ref{subsec:rb}), through various techniques such as stream filtering mechanism and scalable invertible neural network (INN) architecture.

\textbf{Contributions.} We summarize our contributions in this work as follows:

\begin{itemize}
    \item We dissect existing sketch techniques from a perspective of matrix analysis based on \texttt{CS} theory. As a lesson, we reveal the orthogonal irrecoverability that precludes perfect recovery from compact summary of the data stream.
    \item We investigate the feasibility of applying GMs to enhance data sketching. By testing representative GM families, we find that, except for FGMs, they either scale poorly or fall short in the desired efficiency or accuracy.
    \item We design a new linear sketching solution, \texttt{FLORE}, that efficiently utilizes the distribution modeling capability of FGMs to embrace our analyses seamlessly. First, \texttt{FLORE} theoretically supports faithful recovery, even without knowledge about the GT data. Second, \texttt{FLORE} is scalable and generic to real-world applications. To our best knowledge, the deep generative stream processing paradigm has never been explored before.
    \item We conduct a comprehensive evaluation of \texttt{FLORE} on streams simulated using 10 datasets with diverse sizes and patterns. The results demonstrate \texttt{FLORE} outperforms 10 state-of-the-art (SOTA) solutions, achieving up to a \(10^{3}\times\) reduction in per-element frequency estimation error, a \(10^{2}\times\) improvement in accurately estimating entropy and distributions, and up to an $80$+ point gain in heavy-hitter detection. It also attains a higher sketching throughput than prior art.
\end{itemize}

\textbf{Organization.} The problem formulation and preliminaries are presented in \S\ref{main_sec:pre}, followed by the analysis of existing sketch techniques in \S\ref{main_sec:ays}. Next, we benchmark mainstream generative models in \S\ref{main_sec:gensys}, then present our method and system design in \S\ref{subsec:flore}. Finally, \S\ref{main_sec:exps} conducts simulations and reports the experimental results. 


\section{Preliminaries}
\label{main_sec:pre}

\subsection{Problem Formulation}
\label{subsec:formu}

A data stream is a sequence of $n$ tuples arriving in order. The $t$-th tuple is denoted as $(\text{key}_t, v_t)$, where $\text{key}_t \in \mathcal{I}$ is a data-item key used for hashing and $v_t$ is a value associated with the item. These $n$ tuples can have duplicate keys, but algorithms running on the data stream must be \textit{one-pass}. Sketch techniques utilize a compact data structure to summarize each incoming item with constant-time complexity, e.g., $\mathcal{O}(k)$ when applying $k$ independent hash functions to (flattened) counters $b \in \mathbb{R}^m$. Upon receiving a query, the goal is to recover the required statistics from the sketch summary, most notably frequencies $f \in \mathbb{C}^N$ in this paper. 

\textbf{Tailored problem.} Our study focuses on the \textit{linear sketching} problem, where the summary is defined to have a linear part such that each counter is updated linearly by incoming items. If we consider $f$ to be sparse or heavy-tailed, a property justified and utilized in many prior works~\cite{roy2016augmented,huang2017sketchvisor}, the problem falls into a classical \texttt{CS} framework:
\( \min_{\boldsymbol{\Phi}f=b} \left\| f \right\|_1\), 
where $\boldsymbol{\Phi}$ is an indicator matrix produced by multiple hashing operations, and its columns are presented one by one over time. In Figure~\ref{fig:sketch_sensing}, we take the Count-Min Sketch as an example to illustrate the linear matrix formulation with two procedures: summary and recovery. In this setting, \(\boldsymbol{\Phi}_{i,j}\) = $1$ if and only if \(\text{key}_j\) is hashed to counter \(i\); otherwise, \(\boldsymbol{\Phi}_{i,j}\) = $0$. This problem is NP-hard~\cite{foucart2022mathematical} and often highly ill-posed, i.e., $m \ll N$.

\begin{tcolorbox}[
    colback=blue!3,
    colframe=blue!50,
    boxrule=0.5pt,
]
{\large$\diamond$}\ \textbf{Note:} We assume that the possible keys set $\mathcal{I}$ is known at the start of data streams, following the recent study~\cite{da2022bayesian,da2025bfes}. In fact, the online collection of $\mathcal{I}$ can be easily achieved by employing a Bloom Filter as detailed in~\cite{sheng2021pr,li2016flowradar,huang2021toward} (see \S\ref{subsec:structure} also). For untracked keys, we can report zeros as their final estimates
\footnote{Aamand et al.\cite{aamand2023improved} proved that truncating low frequency estimates to zeros will not impact the sketch’s overall accuracy.},
and extend our formulation to \( \min_{\boldsymbol{\Phi}f=b + \boldsymbol{\epsilon}} \left\| f \right\|_1\) where $\boldsymbol{\epsilon}$ captures the error induced by them. We will soon show in \S\ref{main_sec:exps} that the considered framework maintains overall accuracy under incomplete $\mathcal{I}$.
\end{tcolorbox}

\begin{figure}
    \centering
    \includegraphics[width=0.84\linewidth]{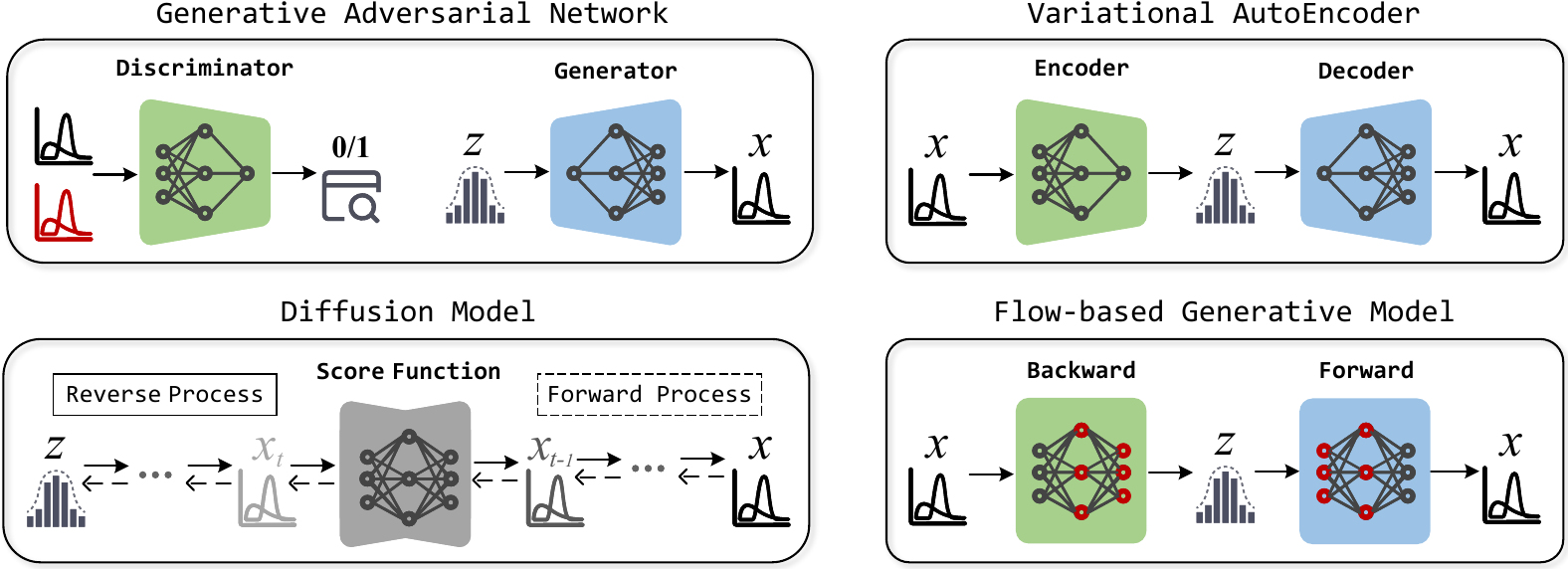}
    \caption{Overview of four types of GMs considered in this work.}
    \label{fig:gm_overview}
\end{figure}  

\subsection{Generative Models}
\label{subsec:gms}

Deep generative models are a class of ML models designed to learn the underlying distribution of data and generate new samples from it. The key idea of them is to develop a latent vector space, where any point satisfies a fixed prior distribution, e.g., an isotropic Gaussian $\mathcal{N}(0, \mathbf{I})$, can be mapped to a ``valid'' sample. 

There is a range of different strategies for learning such latent spaces, each with its own characteristics. As illustrated in Figure~\ref{fig:gm_overview}, the most common types of GMs are \textcircled{\scriptsize 1} Variational Autoencoders (VAEs)~\cite{kingma2013auto}, which learn probabilistic latent representations by optimizing a variational lower bound; \textcircled{\scriptsize 2} Generative Adversarial Networks (GANs)~\cite{goodfellow2014generative}, which synthesize data through an adversarial game between a generator and a discriminator; \textcircled{\scriptsize 3} Diffusion Models (DMs)~\cite{ho2020denoising}, which generate data by iteratively reversing a noise injection process; and \textcircled{\scriptsize 4} Flow-based Generative Models (FGMs)~\cite{rezende2015variational}, which establish exact and invertible mappings between data and latent variables, enabling tractable likelihood evaluation and efficient sampling under suitable architectural constraints. In this work, their underlying principles are explained in detail in Appendix~\ref{app:bg}, and we will restrict our analysis to these four GMs.

\subsection{Notations \& Definitions}

\begin{table}[t]
\centering
\caption{List of Main Symbols}
\label{tab:notation}
\small
\setlength{\tabcolsep}{6pt}
\renewcommand{\arraystretch}{1.15}
\begin{tabular}{
    c p{0.37\linewidth}
    @{\hspace{1.2em}}
    c p{0.37\linewidth}
}
\toprule
\textbf{Symbol} & \textbf{Description}
& \textbf{Symbol} & \textbf{Description} \\
\midrule
$f$ & Ground-truth (frequency) vector
& $\boldsymbol{\Phi}$ & Linear sketching matrix \\
$b$ & Measurement (counters) vector
& $k$ & Number of hash functions in sketch \\
$z$ & Prior (Gaussian) vector
& $N$ & Number of distinct items \\
$\mathcal{I}$ & Key set of data stream
& $m$ & Number of total counters in sketch \\
$T_i$ & $i$-th invertible transformation function
& $(\varepsilon, \varsigma)$ & Coefficient \& error probability of sketch \\
$\boldsymbol{\epsilon}$ & Error introduced by missing elements
& $m_{bf}$ & Number of bits in Bloom Filter \\
$s$ & Non-zero count upper bound 
& $k_{bf}$ & Number of hash functions in Bloom Filter \\
$p_B(\cdot)$ & Distribution of counters
& $p_Z(\cdot)$ & Distribution of priors \\
$p_F(\cdot)$ & Distribution of ground truth
& $q(\cdot)$ & Joint distribution of counters and priors \\
\bottomrule
\end{tabular}
\end{table}

\textbf{Notations.} In Table~\ref{tab:notation}, we list the notational conventions used in this paper.

\begin{definition}
    \textit{(Support)} For a vector $\mathbf{x} \in \mathbb{R}^n$, the set of all indices of non-zero elements in $\mathbf{x}$ is called the support of $x$, denoted as $\operatorname{supp}(\mathbf{x}): \operatorname{supp}(\mathbf{x}) = \left\{ i : \mathbf{x}_i \ne 0 \right\}$. 
\end{definition}

\begin{definition}
    \textit{(Sparsity)} The sparsity of $\mathbf{x}$ = the number of non-zero element in $\mathbf{x}$ = the cardinality of $\operatorname{supp}(\mathbf{x})$. It can be denoted as $|\operatorname{supp}(\mathbf{x})|$ or $\| x \|_0$. A vector $x$ is $s$-sparse if $\| x \|_0 \leq s$.
\end{definition}

\begin{definition}
    \textit{($ s $-term Approximation)} For $ p > 0 $, the $ \ell_p $-error of best $ s $-term approximation to a vector $ \mathbf{x} \in \mathbb{C}^N $ is defined by $\sigma_s(\mathbf{x})_p := \inf \left\{ \left\| \mathbf{x} - \mathbf{z} \right\| _p, \ \mathbf{z} \in \mathbb{C}^N \text{ is } s\text{-sparse} \right\}$.
    \label{def:sa}
\end{definition}

\begin{definition}
    \textit{(Restricted Isometry Constant)} The $s$-th restricted isometry constant $\delta_s = \delta_s(\mathbf{A})$ of a matrix $\mathbf{A} \in \mathbb{C}^{m \times N}$ is the smallest number $\delta_s \geq 0$, such that 
    \(
        (1 - \delta_s)\| \mathbf{x}\| _2^2 \leq \| \mathbf{A}\mathbf{x}\| _2^2 \leq (1 + \delta_s)\| \mathbf{x}\| _2^2 
    \)
    holds for all $s$-sparse vectors $\mathbf{x} \in \mathbb{C}^N$. Equivalently, it is given by
    \(
        \delta_s = \max_{S \subseteq [N],\, \mathrm{card}(S) \leq s} \| \mathbf{A}_S^* \mathbf{A}_S - \mathbf{Id} \| _{2 \to 2}, 
    \)
    where $\mathbf{A}_S^*$ is the hermitian transpose of $\mathbf{A}_S$ and $\mathrm{card}(S)$ is the cardinality of set $S$.
    \label{def:ric}
\end{definition}

\begin{definition}
    \textit{(Kullback–Leibler Divergence)} The Kullback–Leibler (KL) divergence is a type of statistical distance: a measure of how much an approximating probability distribution $q$ is different from a true probability distribution $p$. Mathematically, it is defined as $\mathcal{D}_{\text{KL}}(p \parallel q) = \sum_{x \in \mathcal{X}} p(x) \log \frac{p(x)}{q(x)}$. Obviously, $\mathcal{D}_{\text{KL}}(p \parallel q) \ge 0$ and equals zero if and only if $p=q$ as measures.
    \label{def:kld}
\end{definition}

\begin{definition}
    \textit{(Maximum Mean Discrepancy)} Let $x$ and $y$ be random variables defined on a topological space $X$, with respective Borel probability measures $p$ and $q$. And let $\mathcal{F}$ be a class of functions $F: \mathcal{X} \to \mathbb{R}$. We define the maximum mean discrepancy (MMD) as $\text{MMD}[\mathcal{F}, p, q] := \sup_{F \in \mathcal{F}} \left( \mathbb{E}_x[F(x)] - \mathbb{E}_y[F(y)] \right)$.
\end{definition}

\section{Motivating Analyses}
\label{main_sec:ays}

\subsection{Limitations of Existing Techniques}
\label{subsec:lim}

\textbf{\rbox{T1} Randomized approximation.} Typically, $\hat f_i$ is carried out by extracting information associated with $\text{key}_i$ from the sketch summary, which means each element’s estimate is computed individually. Representative sketch techniques include Count-Min Sketch (CM)~\cite{cormode2005improved} and Count Sketch (CS\footnote{This abbreviation should be distinguished from the compressive sensing (\texttt{CS})})~\cite{charikar2002finding}. Essentially, their recovery of \(\hat f_i\) is based on observing only \(k\) entries of \(\boldsymbol{\Phi}\), ignoring the influence of all other keys. While their processes are very fast (requiring only $\mathcal{O}(k)$ time) and aggregation by minimum or mean suppresses partial noise, the per-element accuracy, e.g., $\varepsilon \|f\|_1$ for CM and $\varepsilon \|f\|_2$ for CS where $\varepsilon$ denotes the sketch coefficient, is heavily constrained by the space budget, i.e., the number of rows in \(\boldsymbol{\Phi}\).

\textbf{\rbox{T2} Disentangled augmentation.} The second class splits the sketch into two smaller components: a heavy part for frequent elements and a light part for infrequent elements. The Augmented Sketch (AG)~\cite{roy2016augmented} and Elastic Sketch (ES)~\cite{yang2018elastic} both fall into this category. With $\|f\| \Downarrow$ after filtering, this design more effectively handles skewed distributions by mitigating hash collisions. From our perspective, such methods can be interpreted as a decomposition of the sketching matrix $\boldsymbol{\Phi}$, which allocates nonzero entries into separate components. By clearing columns of $\boldsymbol{\Phi}$, the degree of underdetermination is thus decreased. However, when resources are constrained, the capacity of the heavy part is limited. Additionally, the light part continues to suffer from large error bounds across all algorithms, an issue exacerbated under a fixed budget.

In particular, the learning-augmented sketches~\cite{kraska2018case,hsu2019learning}, a growing body of research closely related to ours, attempt to optimize partitioning by learning an ``oracle'' (i.e., a classifier). 
However, they inherit the aforementioned limitations and are often not system-friendly, which brings painstaking deployment overhead, as performing a neural network forward pass before hashing operations each time may result in substantial stream processing efficiency degradation for various scenarios.

\begin{figure}
    \centering
    \subfigure[RIP (accuracy)~\cite{huang2021toward} \& complexity (speed) \label{subfig:tc}]{
        \centering
        \includegraphics[width=0.4\linewidth]{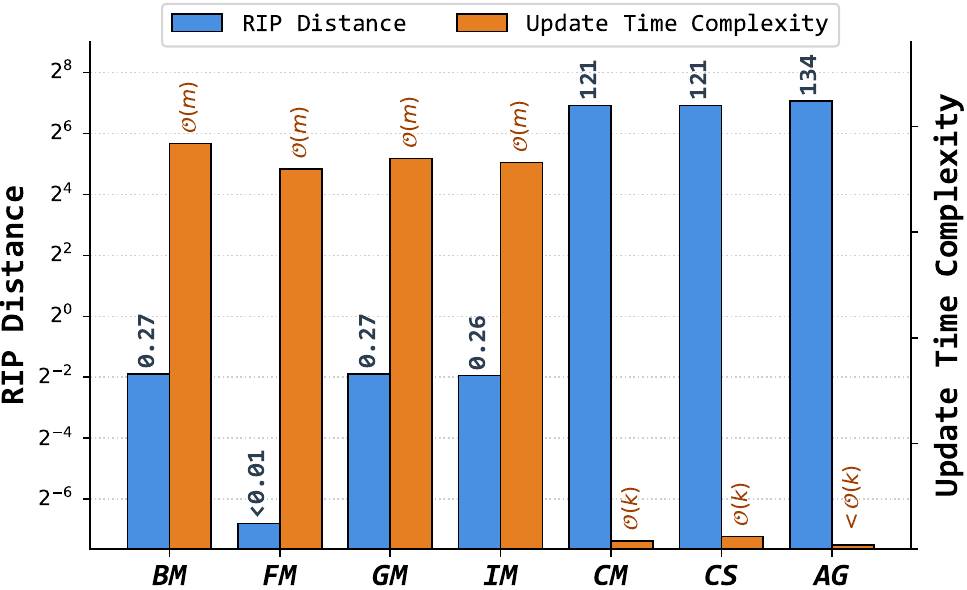}
    }
    \quad
    \subfigure[Measurement matrix\label{subfig:gs}]{
        \centering
        \includegraphics[width=0.29\linewidth]{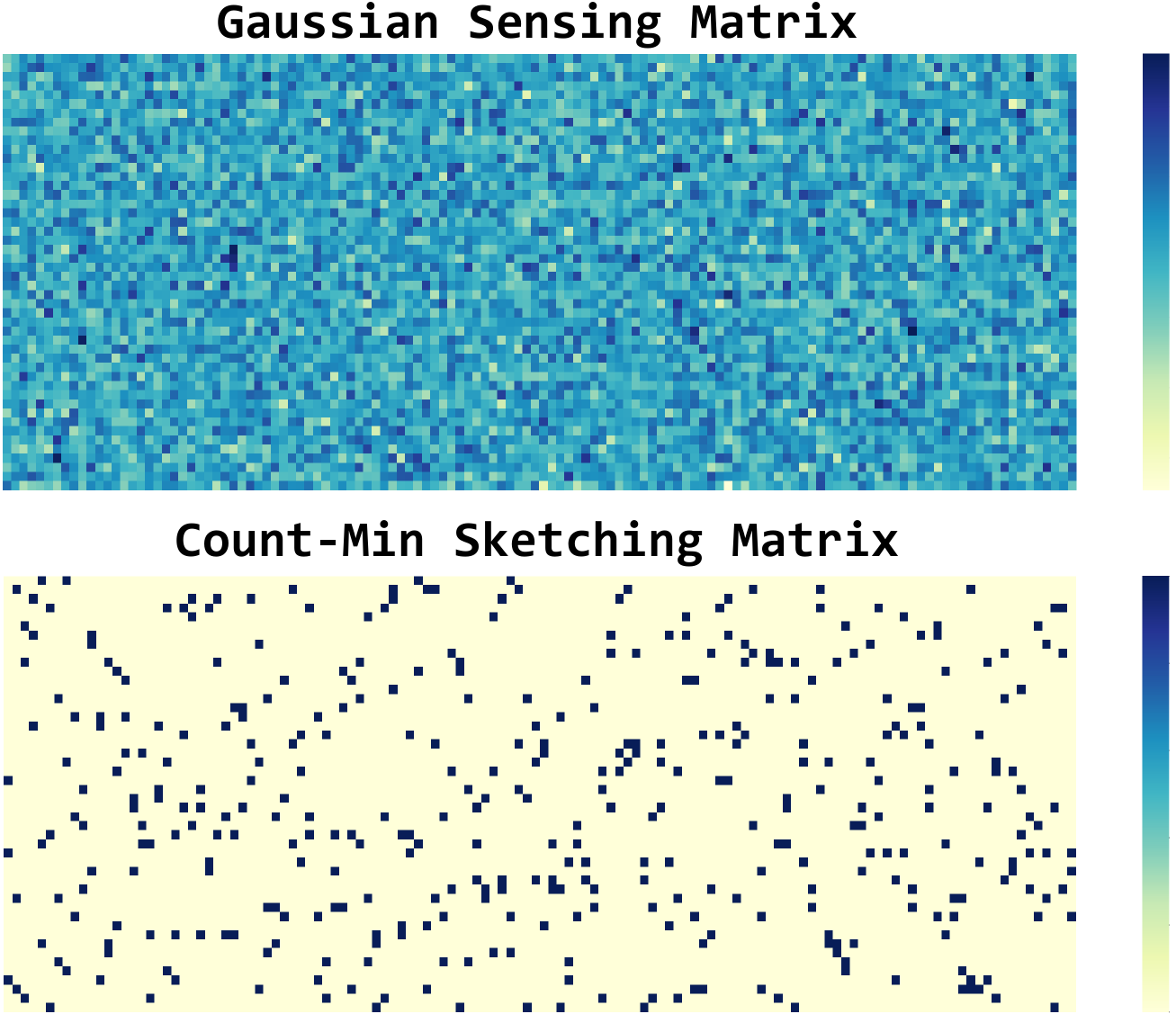}
    }
    \caption{Compare compressive sensing with sketch techniques.}
\end{figure}

\textbf{\rbox{T3} Compressive sensing.} Recently, the sketch techniques have begun to compute estimates for all elements simultaneously rather than independently. They usually decode all information by solving a system of linear equations using \texttt{CS} algorithms. Among these sketches, the PR-sketch (PR)~\cite{sheng2021pr} and NZE-sketch (NZE)~\cite{huang2021toward} represent the most mature examples. Before proceeding with our analysis, we need to introduce a notion called restricted isometry property (RIP)~\cite{candes2008restricted}, which characterizes how well $\boldsymbol{\Phi}$ preserves the norm of sparse signals. Since $\boldsymbol{\Phi}$ must operate uniformly over all sparse vectors, this property is quantified by a sequence of restricted isometry constants (RICs), denoted by $\left\{\delta_s\right\}$, where
\( \delta_s = \sup \left\{\frac{\left| \|\boldsymbol{\Phi} x\|_2 - \|x\|_2 \right|}{\|x\|_2}: \ x \text{ is an } s\text{-sparse vector} \right\}. \)
We say that $\boldsymbol{\Phi}$ satisfies the RIP if $\delta_s$ is small for reasonably large $s$. Under this condition, the following error bound of \texttt{CS}-based sketch techniques holds:

\begin{lightgraybox}
\begin{proposition}
    \textit{Suppose that for $p>0$, the error of best $s$-term approximation is denoted by $\sigma_s(f)_p \coloneqq \inf_{\|z\|_0 \leq s} \| f-z \|_p$. Then, for any vector $f \in \mathbb{C}^N$, a \texttt{CS} solution $f^{\star}$ with $b = \boldsymbol{\Phi} f + \boldsymbol{\epsilon}$ and $\|\boldsymbol{\epsilon}\|_2 \leq \eta$ approximates $f$ with $\ell_2$-error }
    \begin{equation}
        \|f^{\star}-f\|_2 \leq C s^{-1/2} \sigma_s(f)_1 + 2C \eta 
    \label{ineq:cs_bound}
    \end{equation}
    \textit{where the constant $C = \frac{2(1+\rho)}{1-\rho}$ and $\rho \propto \delta_{2s}$.}
\label{prop:error_bound}
\end{proposition}
\end{lightgraybox}

\begin{proofbox}
\begin{proof}
    We first partition the vector $f$ into a sequence of vectors $f_{\mathcal{T}_k}$ based on the magnitudes of its entries. Specifically, let $f_{\mathcal{T}_0}$ retain the $s$ largest entries of $f$ in magnitude (with all others set to $0$), $f_{\mathcal{T}_1}$ retain the next $s$ largest entries, and so forth. Given $f = f^{\star} + h$, at a high level, our proof separately provides upper bounds for \( \left\| h_{\mathcal{T}_{0\cup1}} \right\|_2\) and \( \left\| h_{\mathcal{T}_{0\cup1}^c} \right\|_2\), thereby bounding our target $\left\| f^{\star} - f \right\|_2 = \left\| h \right\|_2$ (i.e., the sum of these two terms). According to the above definition, we can write
    \[
        \left\| h_{\mathcal{T}_j} \right\|_2 = \sqrt{\sum_i h_{\mathcal{T}_j}^2(i)} \leq \sqrt{s \max_i h_{\mathcal{T}_j}(i) ^2} \leq \sqrt{\frac{1}{s} \sum_i h_{\mathcal{T}_{j-1}}(i) } = \sqrt{\frac{1}{s}} \left\| h_{\mathcal{T}_{j-1}} \right\| _1,
    \]
    which holds for $ j \geq 2 $. Based on this, we have
    \(
        \sum_{j \geq 2} \left\| h_{\mathcal{T}_j} \right\| _2 \leq \frac{1}{\sqrt{s}} \sum_{j \geq 1} \left\| h_{\mathcal{T}_j} \right\|_1 = \frac{1}{\sqrt{s}} \left\| h_{\mathcal{T}_0^c} \right\|_1.
    \)
    Then the following useful inequality can be further derived:
    \begin{equation}
        \left\| h_{\mathcal{T}^c_{0 \cup 1}} \right\|_2 = \left\|\sum\limits_{j \ge 2} h_{\mathcal{T}_j} \right\|_2 \le \sum\limits_{j \ge 2} {{{\left\| {{h_{{\mathcal{T}_j}}}} \right\|}_2}} \le \frac{1}{{\sqrt s }} \sum\limits_{j \ge 1} {{{\left\| {{h_{{\mathcal{T}_j}}}} \right\|}_1}} = \frac{1}{{\sqrt s }} {{{\left\| {{h_{\mathcal{T}_0^c}}} \right\|}_1}}
        \label{ineqn:useful1}
    \end{equation}

    Since \( f^{\star} \) minimizes \( \| f^{\star} \| _1 \) subject to the constraint \( \boldsymbol{\Phi} f^{\star} = \boldsymbol{\Phi} f + \boldsymbol{\epsilon} \), i.e., $\min_{\boldsymbol{\Phi}f^{\star}=b} \left\| f^{\star} \right\|_p$, it is intuitively clear that \( \| f^{\star} \| _1 \) cannot exceed \( \| f \| _1 \). Hence, we have
    $\left\| f \right\|_1 \geq \left\| f^{\star} \right\| _1$, which leads to
    \[
        \begin{aligned}
            \left\| f_{\mathcal{T}_0} \right\|_1 + \left\| f_{\mathcal{T}_0^c} \right\|_1 = \left\| f \right\| _1 \geq \left\| f^{\star} \right\|_1 = \left\| f - h \right\|_1 &= \left\| (f - h)_{\mathcal{T}_0} \right\|_1 + \left\| (f - h)_{\mathcal{T}_0^c} \right\|_1 \\
            &\geq \left\| f_{\mathcal{T}_0} \right\|_1 - \left\| h_{\mathcal{T}_0} \right\|_1 + \left\| h_{\mathcal{T}_0^c} \right\|_1 - \left\| f_{\mathcal{T}_0^c} \right\|_1.
        \end{aligned}
    \]
    To this end, we get $\left\| h_{\mathcal{T}_0^c} \right\|_1 \leq 2 \left\| f_{\mathcal{T}_0^c} \right\|_1 + \left\| h_{\mathcal{T}_0} \right\|_1$. Combining with Ineq.~(\ref{ineqn:useful1}), it gives the first upper bound
    \begin{equation}
        \left\|h_{\mathcal{T}^c_{0 \cup 1}}\right\|_2 \le \frac{1}{{\sqrt s }} {{{\left\| {{h_{\mathcal{T}_0^c}}} \right\|}_1}} \le \frac{1}{{\sqrt s }}\left\|h_{\mathcal{T}_0}\right\|_1 + \frac{2}{{\sqrt s }}\left\|f_{\mathcal{T}_0^c}\right\|_1
        \label{ineqn:useful2}
    \end{equation}

    We now continue to derive an upper bound for the other term, \( \left\|h_{\mathcal{T}_{0 \cup 1}}\right\|_2 \). By applying the Cauchy–Schwarz inequality, we obtain the following inequality:
    \[
        \begin{aligned}
            \left\| \boldsymbol{\Phi}\left( f^{\star} - f \right) \right\|^2_2 & = \left\langle \boldsymbol{\Phi}f^{\star} - b + b - \boldsymbol{\Phi}f, \boldsymbol{\Phi}f^{\star} - b + b - \boldsymbol{\Phi}f \right\rangle \\
            & = \left\| \boldsymbol{\Phi}f^{\star} - b \right\|_2^2 + \left\| b - \boldsymbol{\Phi}f \right\|_2^2 + 2 \left| \left\langle \boldsymbol{\Phi}f^{\star} - b, b - \boldsymbol{\Phi}f \right\rangle \right|  \\
            & \le \left\| \boldsymbol{\Phi}f^{\star} - b \right\|_2^2 + \left\| b - \boldsymbol{\Phi}f \right\|_2^2 + 2\left\| \boldsymbol{\Phi}f^{\star} - b \right\|_2\left\| b - \boldsymbol{\Phi}f \right\|_2 \\
            & = \left( \left\| \boldsymbol{\Phi}f^{\star} - b \right\|_2 + \left\| b - \boldsymbol{\Phi}f \right\|_2 \right)^2
        \end{aligned}
    \]
    Therefore,
    \[
        \left\| \boldsymbol{\Phi}h \right\|_2 = \left\| \boldsymbol{\Phi}\left( f^{\star} - f \right) \right\|_2 \le \left\| \boldsymbol{\Phi}f^{\star} - b \right\|_2 + \left\| b - \boldsymbol{\Phi}f \right\|_2 \le 2\left\| \boldsymbol{\epsilon} \right\|_2 \le 2\eta
    \]
    Here we again apply the Cauchy–Schwarz inequality together with Definition~\ref{def:ric} to get:
    \[
       \left| \left\langle \boldsymbol{\Phi}h_{\mathcal{T}_{0 \cup 1}}, \boldsymbol{\Phi}h \right\rangle \right| \le \left\| \boldsymbol{\Phi}h_{\mathcal{T}_{0 \cup 1}} \right\|_2 \left\| \boldsymbol{\Phi}h \right\|_2 \le 2\eta \sqrt{1 + \delta_{2s}}\left\| h_{\mathcal{T}_{0 \cup 1}} \right\|_2
    \]
    By Lemma~\ref{lem:rip_ineqn}, we have
    \[\begin{cases} 
        \left| \left\langle \boldsymbol{\Phi}h_{\mathcal{T}_0}, \boldsymbol{\Phi}h_{\mathcal{T}_j} \right\rangle \right| \le \delta_{2s} \left\| h_{\mathcal{T}_0} \right\|_2 \left\| h_{\mathcal{T}_j} \right\|_2, \qquad \forall j \in \{1, 2, \dots\}, \\
        \left| \left\langle \boldsymbol{\Phi}h_{\mathcal{T}_1}, \boldsymbol{\Phi}h_{\mathcal{T}_j} \right\rangle \right| \le \delta_{2s} \left\| h_{\mathcal{T}_1} \right\|_2 \left\| h_{\mathcal{T}_j} \right\|_2, \qquad \forall j \in \{0, 2, \dots\}. 
    \end{cases}\]
    Thus, using the inequality \(\left\|  h_{\mathcal{T}_0}  \right\|_2 + \left\| h_{\mathcal{T}_1}  \right\|_2 \le \sqrt{2} \left\| h_{\mathcal{T}_{0 \cup 1}} \right\|_2 \), we obtain
    \[
        \begin{aligned}
            \left\| \boldsymbol{\Phi}h_{\mathcal{T}_{0 \cup 1}} \right\|_2^2 & = \left\langle \boldsymbol{\Phi}h_{\mathcal{T}_{0 \cup 1}}, \boldsymbol{\Phi}l \right\rangle - \left\langle \boldsymbol{\Phi}h_{\mathcal{T}_{0 \cup 1}}, \sum\limits_{j \ge 2} \boldsymbol{\Phi}h_{\mathcal{T}_j} \right\rangle \\
            & \le 2\eta \sqrt{1 + \delta_{2s}}\left\| h_{\mathcal{T}_{0 \cup 1}} \right\|_2 - \left\langle \boldsymbol{\Phi}h_{\mathcal{T}_0} + \boldsymbol{\Phi}h_{\mathcal{T}_1}, \sum\limits_{j \ge 2} \boldsymbol{\Phi}h_{\mathcal{T}_j} \right\rangle \\
            & \le 2\eta \sqrt{1 + \delta_{2s}}\left\| h_{\mathcal{T}_{0 \cup 1}} \right\|_2 + \delta_{2s} \left(\left\|  h_{\mathcal{T}_0}  \right\|_2 + \left\| h_{\mathcal{T}_1}  \right\|_2\right) \sum\limits_{j \ge 2} \left\| h_{\mathcal{T}_j} \right\|_2 \\
            & \le 2\eta \sqrt{1 + \delta_{2s}}\left\| h_{\mathcal{T}_{0 \cup 1}} \right\|_2 + \sqrt{2} \delta_{2s}  \left\| h_{\mathcal{T}_{0 \cup 1}} \right\|_2 \sum\limits_{j \ge 2} \left\| h_{\mathcal{T}_j} \right\|_2.
        \end{aligned}
    \]
    Moreover, since \( \left( 1 - \delta_{2s} \right) \left\| h_{\mathcal{T}_{0 \cup 1}} \right\|_2^2 \le \left\| \boldsymbol{\Phi}h_{\mathcal{T}_{0 \cup 1}} \right\|_2^2 \), revisiting Ineq.~(\ref{ineqn:useful1}) yields:
    \[
        \left\| h_{\mathcal{T}_{0 \cup 1}} \right\|_2 \le \frac{2\sqrt{1 + \delta_{2s}}}{1 - \delta_{2s}} \eta + \frac{ \sqrt{2} \delta_{2s} }{1 - \delta_{2s}} \sum\limits_{j \ge 2} \left\| h_{\mathcal{T}_j} \right\|_2 \le \frac{2\sqrt{1 + \delta_{2s}}}{1 - \delta_{2s}} \eta + \frac{ \sqrt{2} \delta_{2s} }{\sqrt{s} \left(1 - \delta_{2s}\right)} \left\| h_{\mathcal{T}^c_{0}} \right\|_1
    \]
    Combining this further with inequality $\left\| h_{\mathcal{T}_0^c} \right\|_1 \leq 2 \left\| f_{\mathcal{T}_0^c} \right\|_1 + \left\| h_{\mathcal{T}_0} \right\|_1$ established earlier, we have
    \[
        \left\| h_{\mathcal{T}_{0 \cup 1}} \right\|_2 \le \frac{2\sqrt{1 + \delta_{2s}}}{1 - \delta_{2s}} \eta + \frac{ \sqrt{2} \delta_{2s} }{\sqrt{s} \left(1 - \delta_{2s}\right)} \left\| h_{\mathcal{T}_{0}} \right\|_1 + \frac{ 2\sqrt{2} \delta_{2s} }{\sqrt{s} \left(1 - \delta_{2s}\right)} \left\| f_{\mathcal{T}^c_{0}} \right\|_1.
    \]
    It should also be noted that $\left\| h_{\mathcal{T}_0} \right\| _1 \leq \sqrt{s}\, \left\| h_{\mathcal{T}_0} \right\| _2 \leq \sqrt{s} \left\| h_{\mathcal{T}_{0 \cup 1}} \right\| _2$, which also follows from the Cauchy–Schwarz inequality, as shown in the following derivation:
    \[
        \frac{1}{{\sqrt s}}{\left\| {{h_{{\mathcal{T}_0}}}} \right\|_1} = \sqrt {{{\left( {\sum\limits_{i} {\frac{1}{{\sqrt s }}|{h_{{\mathcal{T}_0}}}\left( i \right)|} } \right)}^2}} \le \sqrt {{\left({\sum\limits_{j = 1}^s {{\frac{1}{{s }}}} }\right)}\sum\limits_{i} {h_{{\mathcal{T}_0}}^2\left( i \right)} }  = {\left\| {{h_{{\mathcal{T}_0}}}} \right\|_2} \le {\left\| {{h_{{\mathcal{T}_{0 \cup 1}}}}} \right\|_2}
    \]
    Thus, further, we have
    \begin{equation}
       \begin{aligned}
           \left\| h_{\mathcal{T}_{0 \cup 1}} \right\|_2 & \le \frac{2\sqrt{1 + \delta_{2s}}}{1 - \delta_{2s}} \eta + \frac{ \sqrt{2} \delta_{2s} }{1 - \delta_{2s}} \left\| h_{\mathcal{T}_{0 \cup 1}} \right\|_2 + \frac{ 2\sqrt{2} \delta_{2s} }{\sqrt{s} \left(1 - \delta_{2s}\right)} \left\| f_{\mathcal{T}^c_{0}} \right\|_1 \\
           & \le \frac{2\sqrt{1 + \delta_{2s}}}{1 - \left( {\sqrt 2  + 1}  \right) \delta_{2s}} \eta + \frac{ 2\sqrt{2} \delta_{2s} }{\sqrt{s} - \sqrt{s} \left( {\sqrt 2  + 1}  \right) \delta_{2s}} \left\| f_{\mathcal{T}^c_{0}} \right\|_1
       \end{aligned}
        \label{ineqn:useful3}
    \end{equation}
    Now using the first upper bound as in Ineq.~(\ref{ineqn:useful2}), we can write 
    \[
        \left\| f^{\star} - f \right\|_2 = \left\| h \right\|_2 \le \left\| h_{\mathcal{T}_{0 \cup 1}} \right\|_2 + \left\| h_{T^c_{0 \cup 1}} \right\|_2 \le 2 \left\| h_{\mathcal{T}_{0 \cup 1}} \right\|_2 + \frac{2}{\sqrt{s}} \left\| f_{\mathcal{T}_0^c} \right\|_1.
    \]
    Finally, Substituting into the Ineq.~(\ref{ineqn:useful3}), we see that
    \[
        \left\| f^{\star} - f \right\|_2 = \left\| h \right\|_2 \le \frac{{ {{2 + \left( {2\sqrt 2  + 2} \right){\delta_{2s}}}}}}{{\sqrt s - \sqrt s \left( {\sqrt 2  + 1} \right){\delta_{2s}}}} \left\| f_{\mathcal{T}_0^c} \right\|_1 + \frac{4 \sqrt{1+\delta_{2s}}}{1 - \left( \sqrt{2}+1 \right)\delta_{2s}} \eta
    \]
    Let $\rho = \left( \sqrt{2} + 1 \right) \delta_{2s}$; then the above expression can be rewritten as
    \[
        \left\| f^{\star} - f \right\|_2 \le \frac{2\left(1+\rho\right)}{\sqrt{s}\left(1-\rho\right)} \left\| f_{\mathcal{T}_0^c} \right\|_1 + \frac{4\sqrt{1+\delta_{2s}}}{1-\rho} \eta \le \frac{2\left(1+\rho\right)}{\sqrt{s}\left(1-\rho\right)} \left\| f_{\mathcal{T}_0^c} \right\|_1 + \frac{4 \left(1+\rho\right)}{1-\rho} \eta, 
    \]
    where the second inequality is by Lemma~\ref{lem:c_ineqn}. And by Definition~\ref{def:sa}, \(\sigma_s(f)_1\) is the minimum \(\ell_1\)-error over all \(s\)-sparse approximations of \(f\). Since \(f_{\mathcal{T}_0}\) retains the \(s\) largest entries of \(f\) in magnitude, \( \sigma_s(f)_1 = \left\| f - f_{\mathcal{T}_0} \right\|_1 = \left\| f_{\mathcal{T}_0^c} \right\|_1 \). Hence, $\|f^{\star}-f\|_2 \leq C s^{-1/2} \sigma_s(f)_1 + 2C \eta$, where $C = \frac{2(1+\rho)}{1-\rho}$. This concludes the proof.
\end{proof}
\end{proofbox}

Proposition~\ref{prop:error_bound} indicates that their per-element accuracy is bounded on the order of around $\mathcal{O}\left( {C {s}^{-1/2} \sigma_s(f)_1} / {N}  \right)$, which is tighter than the bounds of previous techniques. However, this comes at a greatly increased query cost~\cite{chakrabarti2025streaming} and does not imply better estimates in practice. Since the coefficient $C$ is positively correlated with RICs, we measure the RIP distance (i.e., the difference between $\|\boldsymbol{\Phi} f\|_2$ and $\| f\|_2$) of the classical \texttt{CS} matrices~\cite{szarek1991condition,candes2006near}, i.e., Bernoulli Matrix (BM), Fourier Matrix (FM), Gaussian Matrix (GM), Incoherence Matrix (IM), and the matrices induced by sketch techniques. In Figure~\ref{subfig:tc} (borrowed from~\cite{huang2021toward}), we can see that all three sketch techniques suffer from extremely high RIP distance ($>$120), while \texttt{CS} matrices achieve much smaller ones ($<$0.3). This discrepancy stems from the sparse characteristic of real-world sketching matrices, as visualized in Figure~\ref{subfig:gs}, because only updating a small number, e.g., $\mathcal{O}(k)$, of counters each time can keep pace with high-speed data streams. Thus, \texttt{CS} is still not enough for efficiently solving the problem.
Theorem~\ref{thm:size} formally shows that sketching matrices (e.g., CM) require substantially more rows (or counters here) that optimal RIP matrices. 

\begin{lightgraybox}
\begin{theorem}
    \textit{Let $f \in \mathbb{C}^N$ be $s$-sparse and $\boldsymbol{\Phi}$ be an $m \times N$ matrix that satisfies RIP for perfect recovery of the vector $f$. Then, the necessary number of counters}
    \begin{equation}
        m \geq \begin{cases}
                \mathcal{O}\!\left( s \log (N/s) \right), 
                & \text{\textit{if} } \boldsymbol{\Phi} \text{ \textit{is random dense}}, \\[4pt]
                \mathcal{O}\left( s^{2} \right), 
                & \text{\textit{if} } \boldsymbol{\Phi} \text{ \textit{is 0-1 sparse}}.
            \end{cases}
    \end{equation}
\label{thm:size}
\end{theorem} 
\end{lightgraybox}

\begin{proofbox}
\begin{proof}
    Clearly, our proof is divided into two parts, establishing the lower bounds for the 0-1 sparse and random dense cases separately. Without loss of generality, we further emphasize that these two cases specifically refer to the count–min and Gaussian settings; that is, in the 0-1 matrix, each column contains exactly \(k\) (determined by the number of hash functions) 1's. Before we prove the first case, we need to introduce the following variant of RIP: there exists $ c > 0 $ such that for all $ s $-sparse vector $ f \in \mathbb{R}^N $, $c \|\boldsymbol{\Phi}f\|_2 \leq \|f\|_2 \leq cD \|\boldsymbol{\Phi}f\|_2$.

    \textbf{\textcircled{\scriptsize 1} $\boldsymbol{\Phi}$ is 0-1 sparse matrix.}
    To start, we square the new RIP inequality to obtain \(c^2 \|\boldsymbol{\Phi}f\|^2_2 \leq \|f\|^2_2 \leq c^2D^2 \|\boldsymbol{\Phi}f\|^2_2\). We can bound $ c $ by considering the 1-sparse vector $ f = e_i $, where $ e_i $ is a standard basis vector with a 1 in a coordinate corresponding to a column of $ \boldsymbol{\Phi} $ with $ k $ 1's. Then, the inequalities above give $ c^2 k \leq 1 \leq c^2 D^2 k $. We will only need the right inequality, which we rewrite as $\frac{1}{c^2} \leq D^2 k$.

    Now, consider the vector $ f $ with 1's in the first $ s $ positions and 0's elsewhere, i.e., $ f = \sum_{i=1}^s e_i $. Let $ w_i $ denote the number of 1's in row $ i $ and in the first $ s $ columns of $ \boldsymbol{\Phi} $. Then, since the first $ s $ columns of $\boldsymbol{\Phi} $ contain $ ks $ 1's, we can write $ \sum_i w_i \geq ks$. Applying the Cauchy-Schwarz inequality, we obtain
    \[
        \left\|\boldsymbol{\Phi}f\right\|_2^2 = \sum_{i=1}^m w_i^2 \geq \frac{(\sum_{i=1}^m w_i)^2}{m} \geq \frac{(ks)^2}{m}, \ \text{but} \  \left\|\boldsymbol{\Phi}f\right\|_2^2 \leq \frac{\|f\|_2^2}{c^2} \leq D^2 ks,
    \]
    so putting the two inequalities together gives $ \frac{(ks)^2}{m} \leq D^2 ks $. Cancelling out $ ks $ from both sides gives $ \frac{k}{m} \leq \frac{D^2}{s} $.

    Also, let $ r $ be the maximum number of 1's in any row of $ \boldsymbol{\Phi} $. By permuting the rows and columns of $ \boldsymbol{\Phi} $, we may assume that $ \boldsymbol{\Phi}_{1,i} = 1 $ for $ 1 \leq i \leq r $, i.e., that the first row of $ \boldsymbol{\Phi} $ has a 1 in the first $ r $ entries. In the following proof, we assume that this reordering has already been done. We can now obtain a second bound on $ \frac{k}{m} $ from the obvious inequality: \( k N \leq \text{number of 1's in } \boldsymbol{\Phi} \leq r m \), and thus, $ \frac{k}{m} \leq \frac{r}{n} $.

    For notational convenience, let $ \tau = \min\{r, s\} $. Consider the vector $ f $ with 1's in the first $ \tau $ positions and 0's elsewhere, i.e., $ f = \sum_{i=1}^\tau e_i $. For this choice of $ f $, we see that \(\|\boldsymbol{\Phi}f\|_2^2 \geq \tau^2\) because the first entry of $ \boldsymbol{\Phi}f $ is $ \tau $. Note that because $ \tau \leq s $, the inequalities defining the new RIP apply to $ f $. Thus, we can apply $\frac{1}{c^2} \leq D^2 k$ to get \(\tau^2 \leq \|\boldsymbol{\Phi}f\|_2^2 \leq \frac{\|f\|_2^2}{c^2} \leq D^2 k \tau \).

    We now plug in our two bounds on $ \frac{k}{m} $. If $ r > s $, then $ \tau = s $, and we use the bound $ \frac{k}{m} \leq \frac{D^2}{s} $. This gives \( s^2 \leq D^2 \frac{D^2}{s} m s \), so $ m \geq \frac{s^2}{D^4} $. Similarly, if $ r \leq s $, then $ \tau = r $, so using our second bound on $ \frac{k}{m} $ gives $r^2 \leq D^2 \frac{r}{N} m r$. Therefore, in this case $ m \geq \frac{N}{D^2} $. Putting the two bounds together, $ m \geq \min\left\{ \frac{s^2}{D^4}, \frac{N}{D^2} \right\}$, or we can say that $m \geq \mathcal{O}\left(s^2\right)$.

    \textbf{\textcircled{\scriptsize 2} $\boldsymbol{\Phi}$ is random dense matrix.} 
    Based on Lemma~\ref{lem:gaum}, we consider the asymptotic regime where $N, s, m \to \infty$. The left-hand side of Ineq.~(\ref{ineq:gaum}) satisfies $\frac{m^2}{m+1} \approx m$ for large $m$. On the right-hand side, since $N/s$ is large, we have $\ln(2.34 N/s) = \ln(N/s) + O(1) \approx \ln(N/s)$. Furthermore, as $s$ grows, the term $\frac{\ln(\epsilon^{-1})}{s}$ becomes negligible and vanishes in the limit. Substituting these approximations back into the inequality yields $m \gtrsim 2s (\sqrt{\ln(N/s)})^2$, which simplifies to the lower bound $m \ge 2s \ln(N/s)$, or we can say that $\mathcal{O}\!\left( s \log (N/s) \right)$, ending the proof.
\end{proof}
\end{proofbox}

\subsection{The Anatomy of Failure}
\label{subsec:fail}

\begin{wrapfigure}[14]{l}{0.54\textwidth}
  \vspace{-0.63cm}
  \begin{center}
  \includegraphics[width=1.\linewidth]{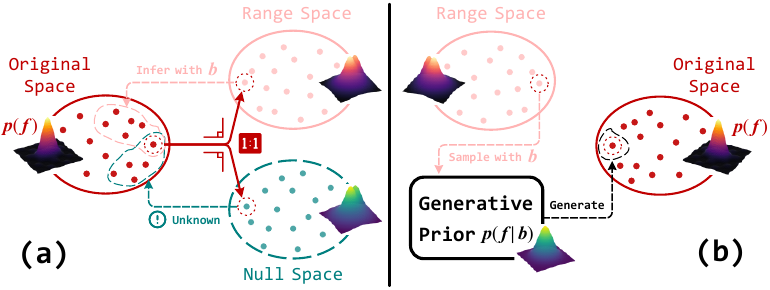}
\caption{\textbf{Illustration of our insights.} (a) The GT $f$ cannot be recovered due to orthogonal information loss in the null space. (b) We leverage GMs to generate the null-space component which is in harmony with the range-space component.}
\label{fig:orth}
\end{center}
\end{wrapfigure}

Acknowledging the intrinsic limitations of existing techniques, we next analyze the root causes of their poor performance. Our methodology is to examine whether compressed $b$ admits perfect recovery of $f$, and range-null space decomposition is at the core of our motivation. Here, we first introduce the definitions of null space and range space. A vector \( f_N \) is said to lie in the null space of \( \boldsymbol{\Phi} \) if and only if \( \boldsymbol{\Phi} f_N = 0 \), and $f_{\Phi}$ lies in the range space of \( \boldsymbol{\Phi} \) if and only if \( \boldsymbol{\Phi} f_{\Phi} = b \). See related works~\cite{da2025bfes,fang2025alphaedit} for more details. For any linear system, an arbitrary signal $f$ can be decomposed into two parts: $f = f_{\Phi} + f_N$. This decomposition can be obtained in various ways, e.g., the singular value decomposition (SVD) used below. In the context of the linear sketching problem in data streams, however, this unfortunately leads to the following negative result.

\begin{lightgraybox}
\begin{proposition}
     \textit{The sketch summary $b$ captures only the range-space component $f_{\Phi}$, while the null-space component $f_N$ is orthogonal to $f_{\Phi}$ and irrecoverable from $b$.}
    \label{prop:failure_guarantee}
\end{proposition}
\end{lightgraybox}

\begin{proofbox}
\begin{proof}
    Before decomposing the vector \( f \), we first perform a singular value decomposition (SVD) of the sketching matrix \( \boldsymbol{\Phi} \): \( \boldsymbol{\Phi} = \boldsymbol{U} \boldsymbol{\Sigma} \boldsymbol{V^{\top}} \).
    Here, \( \boldsymbol{U} \in \mathbb{R}^{m \times m} \) and \( \boldsymbol{V} \in \mathbb{R}^{N \times N} \) are orthogonal matrices, and \( \boldsymbol{\Sigma} \in \mathbb{R}^{m \times N} \) is a rectangular diagonal matrix whose diagonal entries are the singular values of $\boldsymbol{\Phi}$, arranged in descending order. Let the rank of $\boldsymbol{\Phi}$ be $r$. According to the properties of SVD, the orthogonal matrix \( \boldsymbol{V} \) can be written as \( \boldsymbol{V} = [\boldsymbol{\psi}, \boldsymbol{\phi}] \), where \( \boldsymbol{\psi} = {\psi_1, \psi_2, \dots, \psi_r} \in \mathbb{R}^{N \times r} \) forms an orthonormal basis for the range-space of \( \boldsymbol{\Phi} \), and \( \boldsymbol{\phi} = {\phi_1, \phi_2, \dots, \phi_{N-r}} \in \mathbb{R}^{N \times (N-r)} \) forms an orthonormal basis for the null-space of \( \boldsymbol{\Phi} \).

    Now define \( f_{\Phi} = \boldsymbol{\psi} \boldsymbol{\psi}^{\top} f \) and \( f_N = \boldsymbol{\phi} \boldsymbol{\phi}^{\top} f \). Note that 
    \[
        f_{\Phi} + f_N = \boldsymbol{\psi} \boldsymbol{\psi}^{\top} f + \boldsymbol{\phi} \boldsymbol{\phi}^{\top} f = (\boldsymbol{\psi} \boldsymbol{\psi}^{\top} + \boldsymbol{\phi} \boldsymbol{\phi}^{\top}) f = \boldsymbol{V} \boldsymbol{V}^{\top} f = f
    \]
    Hence, \( f \) can be decomposed into two components, \( f_{\Phi} \) and \( f_N \). Since \( f = f_{\Phi} + f_N \), left-multiplying both sides by \( \boldsymbol{\Phi} \) yields \( \boldsymbol{\Phi}f = \boldsymbol{\Phi}f_{\Phi} + \boldsymbol{\Phi}f_N \). And because \( f_N = \boldsymbol{\phi} \boldsymbol{\phi}^{\top} f \) and \( \boldsymbol{\phi} \) is a basis of the null-space of \( \boldsymbol{\Phi} \), we have \( \boldsymbol{\Phi}f_N = \boldsymbol{\Phi} \boldsymbol{\phi} \boldsymbol{\phi}^{\top} f = 0 \), and therefore \( \boldsymbol{\Phi}f_{\Phi} = \boldsymbol{\Phi}f = b \). Thus, \( f_{\Phi} \) and \( f_N \) correspond to the components of $f$ lying in the range-space and the null-space of \( \boldsymbol{\Phi} \), respectively.
    Moreover,
    \[
        f_{\Phi}^{\top} f_N = \left(\boldsymbol{\phi} \boldsymbol{\phi}^{\top} f \right)^{\top}, \quad \boldsymbol{\phi} \boldsymbol{\phi}^{\top} f = f^{\top} \boldsymbol{\psi} (\boldsymbol{\psi}^{\top} \boldsymbol{\phi}) \boldsymbol{\phi}^{\top} f = 0
    \]
    Therefore, \( f_{\Phi} \) and \( f_N \) are orthogonal and uncorrelated, which completes the proof.
\end{proof}
\end{proofbox}

Figure~\ref{fig:orth}(a) illustrates the implication of Proposition~\ref{prop:failure_guarantee}, which we term \textit{orthogonal irrecoverability}. To guarantee exact recovery, it basically requires that all $f_N$ are nearly-zero vectors, as they are entirely discarded during compression. This requirement is especially hard to meet for sparse matrices induced by sketch techniques, motivating the usage of external prior knowledge to compensate for the missing information, which we will get to later.

\section{Toward Generative Linear Sketching}
\label{main_sec:gensys}

\subsection{From Lessons to Motivation}

\textbf{On \emph{generative} optimization and why it might do better.} As outlined in \S\ref{main_sec:ays}, relying solely on the summary $b$ is insufficient to achieve nearly-zero-error data sketching. To break the ``curse of compression'', we seek solutions that conform to two important constraints: \textit{consistency} and \textit{fidelity}. For the consistency constraint, we can simply resort to the range-space $\boldsymbol{\Phi}f \equiv b$. For the fidelity constraint, as discussed in \S\ref{subsec:lim}, while the sparsity assumption of $f$ is often natural and makes sense in \texttt{CS} theory, it is neither necessary nor sufficient for our goal. 
An enhanced approach that avoids presuppositions regarding $f$ and also iterative optimizations of \texttt{CS}, is training a generator on recent realizations of the data stream to directly produce samples that closely follow $f$'s distribution, hence effectively ``infilling'' the null space. In Fig.~\ref{fig:orth}(b), we illustrate the key idea of which we refer to \textit{generative prior} through the maximization of the likelihood $\log p(f \mid b)$ via sampling. It can outperform existing techniques in scenarios where streaming data is volatile and hard to predict under limited memory budgets. Directly inferring $f$ also obviates the need for solving a lengthy linear program, and thus significantly accelerates runtimes for large-scale data streams (see \S\ref{subsec:decoding}). The pre-trained generator is expected to learn the optimum in Fig.~\ref{fig:orth}(b). 
Indeed, our proposed approach, which we will get to later, quickly converges to this outcome.

\textbf{Applicability of ISR to sketch.} For computer vision, image super-resolution (ISR) is a similar and well-studied task~\cite{dong2015image}. Recent advances~\cite{ledig2017photo,bora2017compressed,ardizzone2019analyzing,chung2023diffusion,wang2023zero} in generative AI for ISR also give us hope, but intuitively, not every ISR algorithm makes sense for data sketching, due to fundamentally different domains of data (super-high-dimensional counters v.s. grids of pixels). In addition, translating the measurement system from concept into practice is an open frontier yet, due to a host of challenges such as evolving data streams and training data scarcity. Thus, rather than simply porting ISR into a new domain, this work aims to pioneer a deployable generative sketching paradigm.

\subsection{Integrating Generative Priors}
\label{subsec:igp}

There are two ways to incorporate distributional prior information into our algorithm design. The first is a straightforward strategy that adds distributional regularization during training, such as the Dirichlet process~\cite{cai2018bayesian} and Zipf’s law~\cite{powers1998applications}. For example, the mapping $b \rightarrow f$ may be pre-trained on Zipfian datasets exhibiting similar skewness~\cite{cao2024learning}, or $f$ may be constrained to satisfy prescribed structure~\cite{yuan2025learning}. However, this way has a clear shortcoming: the manually imposed, deterministic constraints are often overly restrictive, resulting in a narrow range of applicability. Therefore, here it is not a good fit.

In this work, we adopt the second way to build a generative mapping between $\left[b, z\right]$ and $f$, where $z$ is a latent vector drawn from a Gaussian distribution. Unlike the previous strategy, the latent variable $z$ explicitly introduces external information that can be actively leveraged for diverse generation. Moreover, for the guarantee of learning the mapping unbiasedly, we prove the following theorem: 

\begin{lightgraybox}
\begin{theorem}
    \textit{For a sketching matrix \( \boldsymbol{\Phi} \in \mathbb{R}^{m \times N} \) of rank \( r \), ground-truth vector \( f \in \mathbb{C}^N \) and \( b = \boldsymbol{\Phi} f\), there exists a gaussian vector \( z \in \mathbb{R}^{N - r} \) and learnable mapping \( G : \mathbb{R}^{m + N - r} \to \mathbb{R}^N \) such that \( G(\left[b, z\right]) \) uniquely determines \( f \).}
    \label{thm:mapping}
\end{theorem}
\end{lightgraybox}

\begin{proofbox}
\begin{proof}
    This proof continues the proof of Proposition~\ref{prop:failure_guarantee}; thus, there is conceptual overlap, and unless otherwise stated, the used notation is consistent with the preceding discussion. 
    
    We first notice that \( b = \boldsymbol{\Phi} f_{\Phi} = \boldsymbol{U} \boldsymbol{\Sigma} \boldsymbol{V^{\top}} f_{\Phi} \). Recalling that \( \boldsymbol{V^{\top}} = [\boldsymbol{\psi}, \boldsymbol{\phi}] \), and using the facts that \( \boldsymbol{\psi}^{\top} \boldsymbol{\psi} = \mathbf{Id} \) and \( \boldsymbol{\phi}^{\top} \boldsymbol{\phi} = 0 \), this expression can be further expanded as 
    \(
        \left(
            \begin{array}{c}
                 \boldsymbol{\Sigma}_{\leq r}^{-1} \boldsymbol{U}_{\leq r}^{\top} b \\ 
                 \mathbf{0}
            \end{array}
        \right)
        = \boldsymbol{V^{\top}} f_{\Phi},
    \)
    where \( \boldsymbol{\Sigma}_{\leq r} \in \mathbb{R}^{r \times r} \) is the diagonal matrix consisting of the top \( r \) singular values of \( \boldsymbol{\Phi} \), \( \boldsymbol{\Sigma}_{\leq r}^{-1} \) is its inverse, \( \boldsymbol{U}_{\leq r}^{\top} \in \mathbb{R}^{r \times m} \) denotes the first \( r \) rows of \( \boldsymbol{U}^{\top} \), and \( \mathbf{0} \) is the zero vector.
    Since \( \boldsymbol{V}^{\top} \) is full rank, there exists a one-to-one correspondence between \( f_{\Phi} \) and \( \bar{b} = \boldsymbol{U}_{\leq r}^{\top} b \). 
    
    Next, we show that \( b \) and \( \bar{b} \) are also in one-to-one correspondence. Analogous to \( \boldsymbol{V} \), we can rewrite \( \boldsymbol{U}^{\top} \) as \( \boldsymbol{U}^{\top} = [\boldsymbol{\psi}', \boldsymbol{\phi}'] \), where \( \boldsymbol{\psi}' = \boldsymbol{U}_{\leq r}^{\top} \) forms an orthonormal basis for the column space of \( \boldsymbol{\Phi} \), and \( \boldsymbol{\phi}' \) spans the null space of \( \boldsymbol{\Phi}^{\top} \). Since for all possible \( f \), the vector \( b \) always lies in the column space of \( \boldsymbol{\Phi} \), we have 
    \( 
        \boldsymbol{U}^{\top} b = 
        \left(\begin{array}{c}
            \boldsymbol{\psi}' b \\
            \mathbf{0}
        \end{array}\right) = 
        \left(\begin{array}{c}
            \bar{b} \\
            \mathbf{0}
        \end{array}\right).
    \)
    Since \( \boldsymbol{U}^{\top} \) is full rank, there exists a one-to-one correspondence between \( b \) and \( \bar{b} \). To this end, \( b \) is uniquely associated with both \( \bar{b} \) and \( f_{\Phi} \).

    Similarly, left-multiplying the identity \( f = f_{\Phi} + f_N \) by \( \boldsymbol{\phi}^{\top} \) yields 
    \(
        \boldsymbol{\phi}^{\top} f = \boldsymbol{\phi}^{\top} f_{\Phi} + \boldsymbol{\phi}^{\top} f_N = \boldsymbol{\phi}^{\top} f_N =: Z \in \mathbb{R}^{N-r}.
    \)
    We then perform an SVD of \( \boldsymbol{\phi}^{\top} \), denoted by \( \boldsymbol{\phi}^{\top} = U' \Sigma' V'^{\top} \). Since the rank of \( \boldsymbol{\phi}^{\top} \) is \( N - r \), its singular value matrix \( \Sigma' \) contains \( N - r \) non-zero singular values. Moreover, because the null space of \( \boldsymbol{\phi}^{\top} \) is spanned by \( \operatorname{span}(\boldsymbol{\psi}) \), we have \( V' = [\boldsymbol{\phi}, \boldsymbol{\psi}] \). Consequently, the above equation can be rewritten as
    \(
        \left(\begin{array}{c}
            \Sigma_{\leq N-r}^{'-1} U'^{\top} Z \\
            \mathbf{0}
        \end{array}\right)
        = V'^{\top} f_N,
    \)
    where \( \Sigma'_{\leq N-r} \) consists of the first \( N - r \) columns of \( \Sigma' \), and \( \Sigma_{\leq n-r}^{'-1} \) denotes its inverse. Since both \( V' \) and \( U' \) are full-rank matrices, the above equation establishes a one-to-one correspondence between \( Z \) and \( f_N \).

    Since any fixed vector can be regarded as the mean of some Gaussian distribution, we can efficiently approximate $Z$ using a gaussian vector $z$. The vector $z$ can be standard gaussian if passed through a single-layer MLP by leveraging Universal Approximation Theorem~\cite{pinkus1999approximation}. Therefore, we can say that the pair \( [b, z] \) uniquely determines \( f = f_{\Phi} + f_N \). Equivalently, there exists a learnable mapping \(G : \mathbb{R}^{m + N - r} \rightarrow \mathbb{R}^{N} \) such that \(G([b, z]) = f \).
\end{proof}
\end{proofbox}

\subsection{On the Choice of GMs}

Given the high-level view of the mapping, the next challenge is to maximize the log-likelihood of its outputs. To this end, we consider four classes of deep generative paradigms as detailed in \S\ref{subsec:gms}. Further, to support real-world streaming applications, our model selection is guided by three key principles: (i) The selected GM should support \textbf{easy adaptation}, because after prolonged usage, the model inevitably encounters generalization issues, and any subsequent fine-tuning must minimize disruption as much as possible; (ii) The selected GM should support \textbf{fast inference}, because sketch recovery and downstream analytics are often performed under limited latency constraints in high-speed data stream processing; and (iii) The selected GM should support \textbf{high accuracy \& realness}, because faithful reconstruction of the underlying data distribution is essential for reliable estimation and decision-making in measurement-based applications.

\subsubsection{Unconditional Distribution Modeling Capability}

We start with evaluating the GM's ability to fit the data stream distribution, as this ability forms the foundation of any downstream generative task. To this end, we construct synthetic data streams drawn from two different distributions. Specifically, \textit{Gaussian} streams are used to assess the models’ generalization ability under smooth and light-tailed distributions, while \textit{Zipfian} streams are adopted to emulate natural data stream characteristics with strong skewness and heavy tails. For both distributions, we vary the streaming data dimensionality, i.e., the number of distinct keys, from 1K to 100K, thereby systematically examining the robustness and scalability of each model across a wide range of stream sizes. 
Then, we train unconditional GMs on these synthetic streams using standard training protocols. All models adopt lightweight architectures composed of simple linear layers. VAEs, GANs, and FGMs are implemented with conventional adversarial or log-likelihood objectives, while the Wasserstein distance regularization is additionally included to improve adversarial training stability (WGAN)~\cite{arjovsky2017wasserstein}. For DMs, we directly train the underlying network to predict the clean signal from noisy inputs~\cite{yuan2024diffusion,li2025back}, in which 1,000 diffusion and sampling steps is used. 

\begin{figure}
    \subfigure[\# of distinct keys $\approx$ 1K\label{subfig:1k}]{
        \centering
        \includegraphics[width=0.233\linewidth]{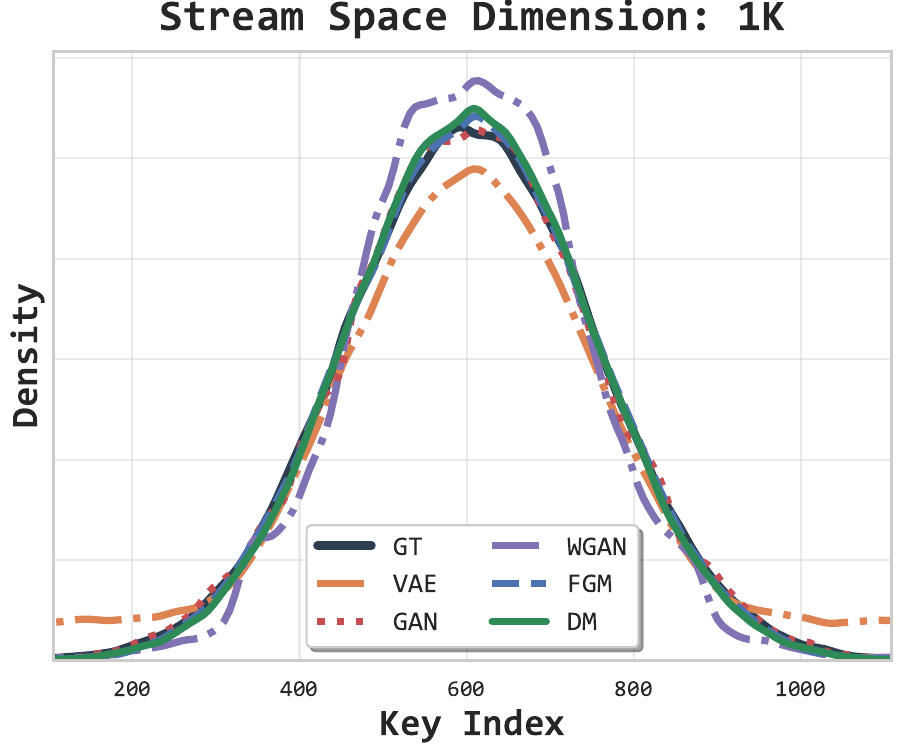}
    }
    \hfill
    \subfigure[\# of distinct keys $\approx$ 10K\label{subfig:10k}]{
        \centering
        \includegraphics[width=0.233\linewidth]{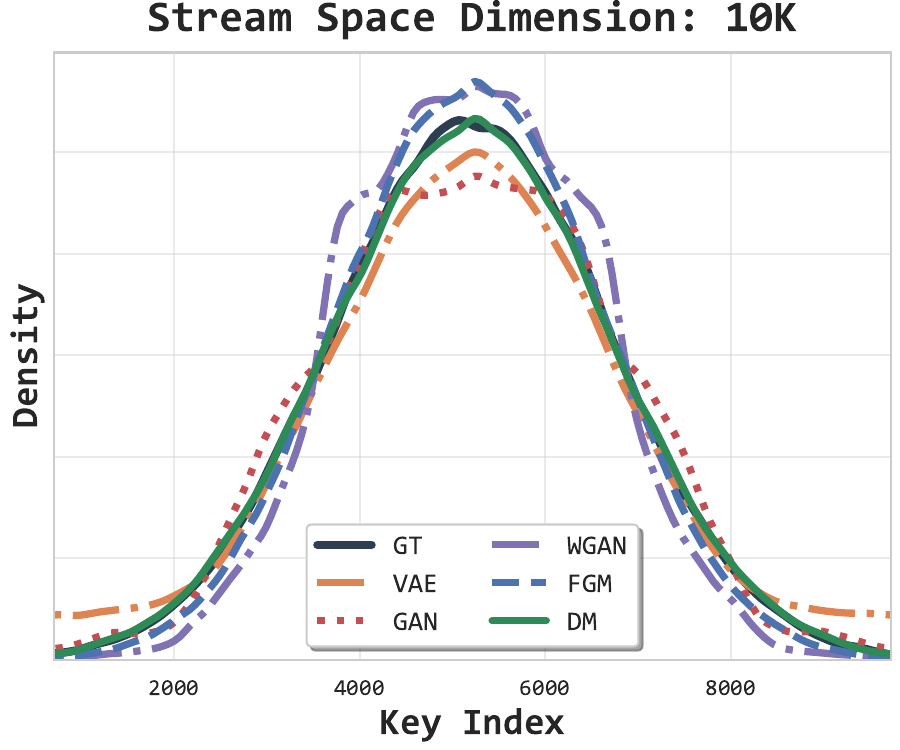}
    }
    \hfill
    \subfigure[\# of distinct keys $\approx$ 50K\label{subfig:50k}]{
        \centering
        \includegraphics[width=0.233\linewidth]{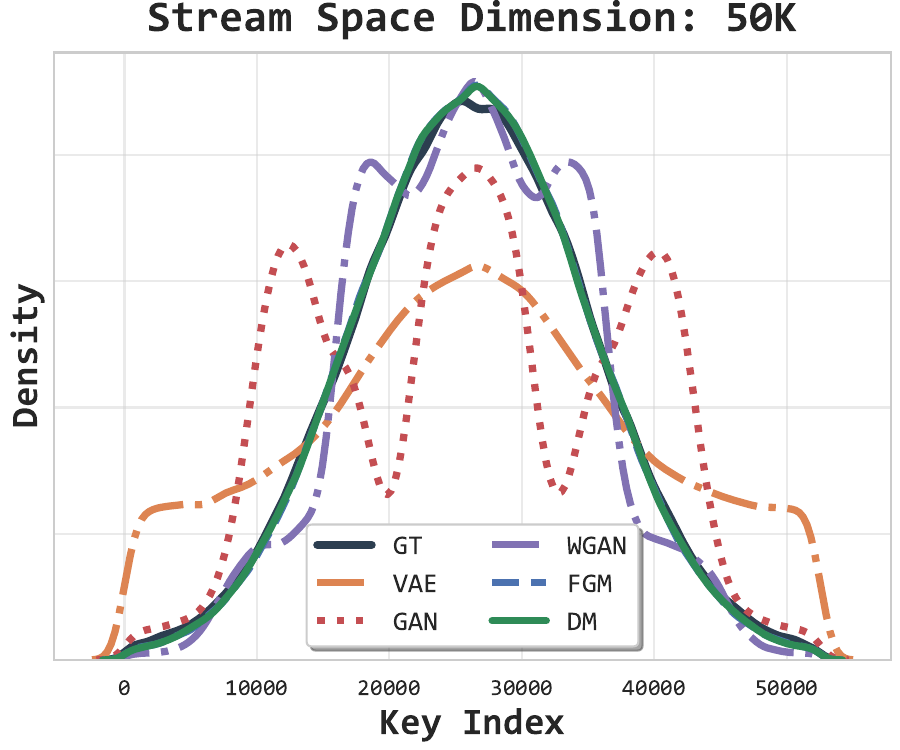}
    }
    \hfill
    \subfigure[\# of distinct keys $\approx$ 100K\label{subfig:100k}]{
        \centering
        \includegraphics[width=0.233\linewidth]{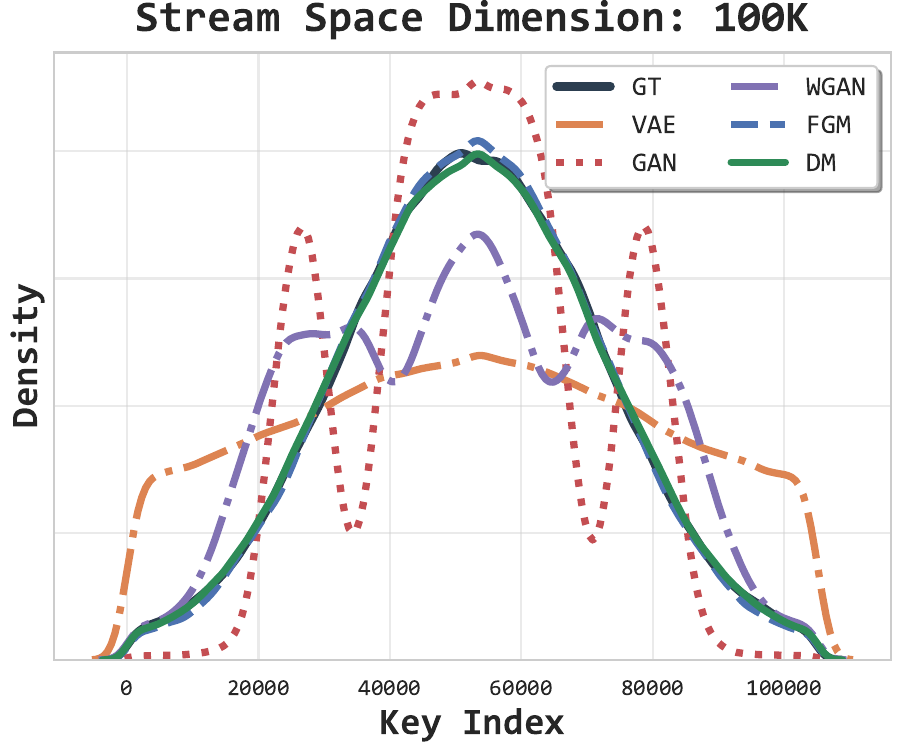}
    }\\
    \subfigure[\# of distinct keys $\approx$ 1K\label{subfig:zipf_1k}]{
        \centering
        \includegraphics[width=0.233\linewidth]{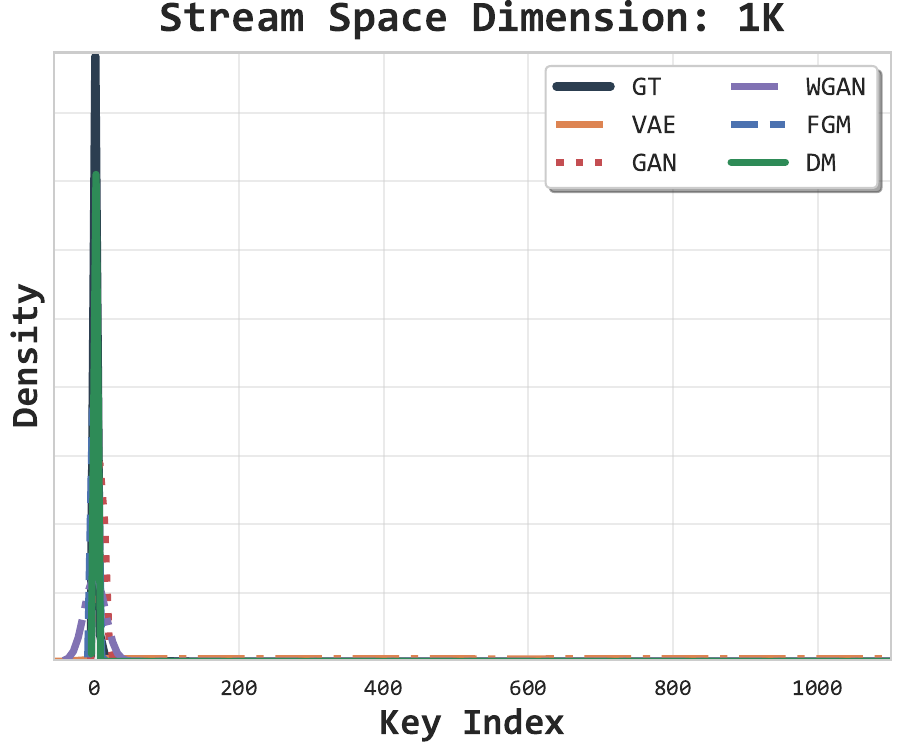}
    }
    \hfill
    \subfigure[\# of distinct keys $\approx$ 10K\label{subfig:zipf_10k}]{
        \centering
        \includegraphics[width=0.233\linewidth]{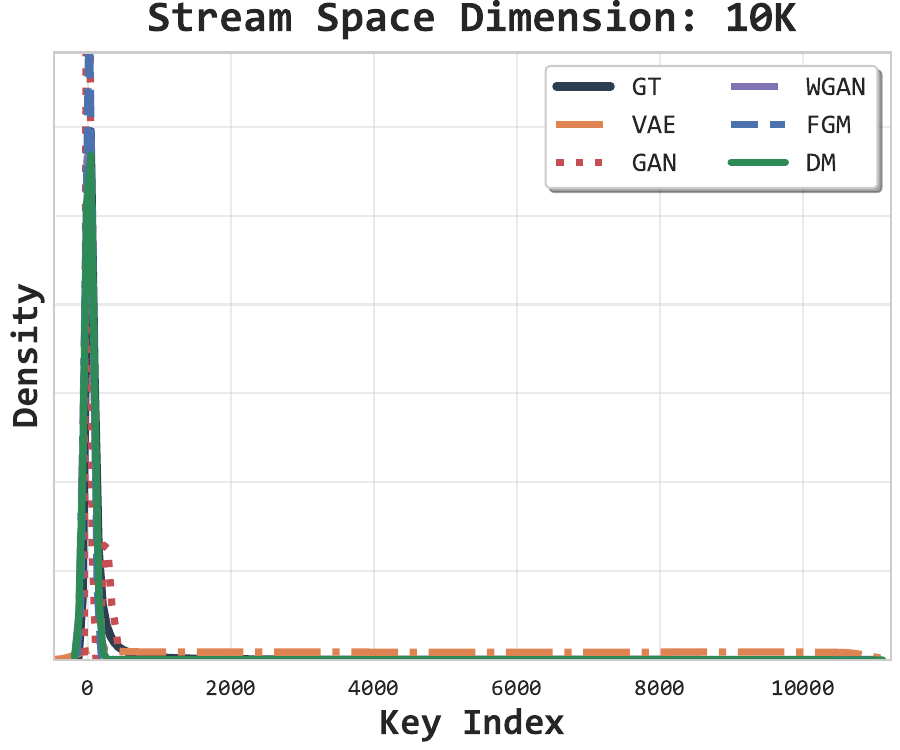}
    }
    \hfill
    \subfigure[\# of distinct keys $\approx$ 50K\label{subfig:zipf_50k}]{
        \centering
        \includegraphics[width=0.233\linewidth]{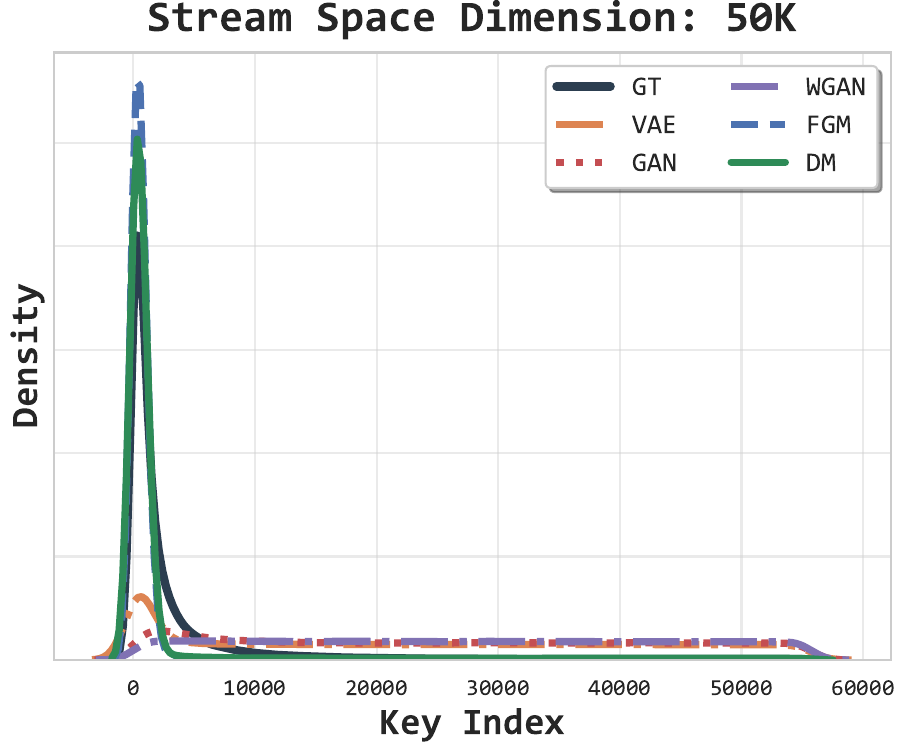}
    }
    \hfill
    \subfigure[\# of distinct keys $\approx$ 100K\label{subfig:zipf_100k}]{
        \centering
        \includegraphics[width=0.233\linewidth]{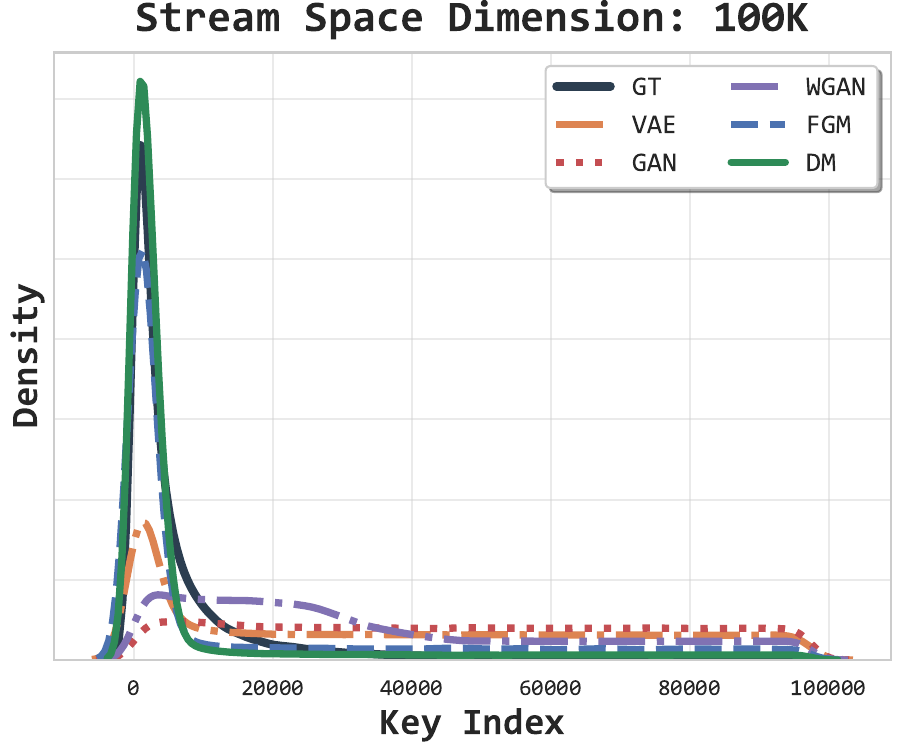}
    }
    \caption{Visualization of Streaming data distribution fitting. The first row shows results on Gaussian distributions, while the second row shows results on Zipfian distributions. From left to right, the scale of the data stream increases.}
    \label{fig:unc_syn}
\end{figure}

Figure~\ref{fig:unc_syn} compares the sampling results of these five methods on the two datasets, in which the closer the curve is to the GT curve, the better the distribution fitting capability. The figure shows that all GMs perform very well under the 1K and 10K settings. Among them, FGM and DM exhibit the most accurate fits, whereas VAE and GAN perform slightly worse. Compared to GAN, WGAN yields more confident generations, with its curve consistently exhibiting a sharper peak around the GT mode. 
However, as the dimensionality of the data stream increases to 50K and 100K, the performance landscape changes markedly. VAE fails to approach the target distribution, and GAN-based methods suffer from severe instability (see 1st row), from which even WGAN cannot escape this ``fate''. In contrast, FGM and DM remain robust under high-dimensional settings. Even worse, in heavy-tailed distributions (see 2nd row), the aforementioned methods cannot even capture the overall ``shape''. Notably, DM demonstrates superior capability in accurately shaping the underlying distribution. Hence, from the perspective of unconditional distribution learning, FGM and DM emerge as the most suitable candidates.

\subsubsection{Conditional Generation for Stream Recovery}

\begin{figure}[ht]
    \centering
    \begin{minipage}[t]{0.49\textwidth} 
        \begin{algorithm}[H]             
        \caption{Estimation Algorithm for Unconditional GMs}
        \label{alg:gms}
        \begin{algorithmic}[1]
            \ENSURE Volume estimation using an unconditional GM
            \REQUIRE GM $f_G$, strength $\zeta$, matrix $\boldsymbol{\Phi}$, \# of loops $K$
            \STATE $\{z_0,\dots,z_i,\dots\} \gets$ sample priors from $\mathcal{N}(0,I)$
            \STATE Select a good $z$ such that ${\arg \min}_{z_i} \|b - \boldsymbol{\Phi} f_G(z_i) \|_2^2$
            \STATE $b \gets$ sample counters by linear sketching
            \FOR{$j=K$ to $0$}
                \STATE $z \gets z - \zeta \nabla_{z} \left(\| b - \boldsymbol{\Phi} f_G(z) \|_2^2 + \|z\|_1 \right)$
            \ENDFOR
            \STATE $x \gets f_G(z)$
            \STATE return $x$
        \end{algorithmic}
        \end{algorithm}
    \end{minipage}%
    \hfill
    \begin{minipage}[t]{0.49\textwidth} 
        \begin{algorithm}[H]
        \caption{Estimation Algorithm for Unconditional DMs}
        \label{alg:dms}
        \begin{algorithmic}[1]
            \ENSURE Volume estimation using an unconditional DM
            \REQUIRE DM $f_G$, strength $\zeta$, matrix $\boldsymbol{\Phi}$, \# of steps $T$
            \STATE $x_T \gets$ sample prior from $\mathcal{N}(0,I)$
            \STATE $b \gets$ sample counters by linear sketching
            \FOR{$t=T$ to $0$}
                \STATE $x'_{t-1} \gets $ \textsc{ReverseSampleOneStep}$(x_t, t, f_G)$
                \STATE $\hat x_{0} \gets $ \textsc{PredcitX0withXt}$(x'_{t-1}, t)$
                \STATE $x_{t-1} \gets x'_{t-1} - \zeta \nabla_{x_t} \| b - \boldsymbol{\Phi} \hat x_{0} \|_2^2$
            \ENDFOR
            \STATE return $x_0$
        \end{algorithmic}
        \end{algorithm}
    \end{minipage}
\end{figure}

Next, we evaluate the transferability of the (unconditional) generative capabilities of different models to downstream data stream (conditional) recovery.To elicit the GM's intrinsic distribution-shaping capability and minimize the influence of the underlying architecture, we directly extend the unconditional GM to a conditional version. Our approach is based on gradient-based search in the latent space of the GM for a ``good'' latent variable $z$ that produces content consistent with the observed data. By following prior works~\cite{bora2017compressed,yuan2023traffic}, for all GMs except DMs, we first identify a "starting point," and then perform backward updates on it using the gradient of $\| b - \boldsymbol{\Phi} f_G(z) \|_2^2 + \|z\|_1$ with regard to $z$, where the former term enforces consistency with the observations, while the latter term promotes plausibility (or likelihood). Regarding DMs, we directly integrate this gradient adjustment into the reverse sampling process, as applying the aforementioned iterative refinement to them would be too time-consuming. See DPS~\cite{chung2023diffusion} for more details. Algorithm~\ref{alg:gms} and Algorithm~\ref{alg:dms} present our two implementation strategies, respectively. 

\begin{figure}
    \centering
    \includegraphics[width=0.99\linewidth]{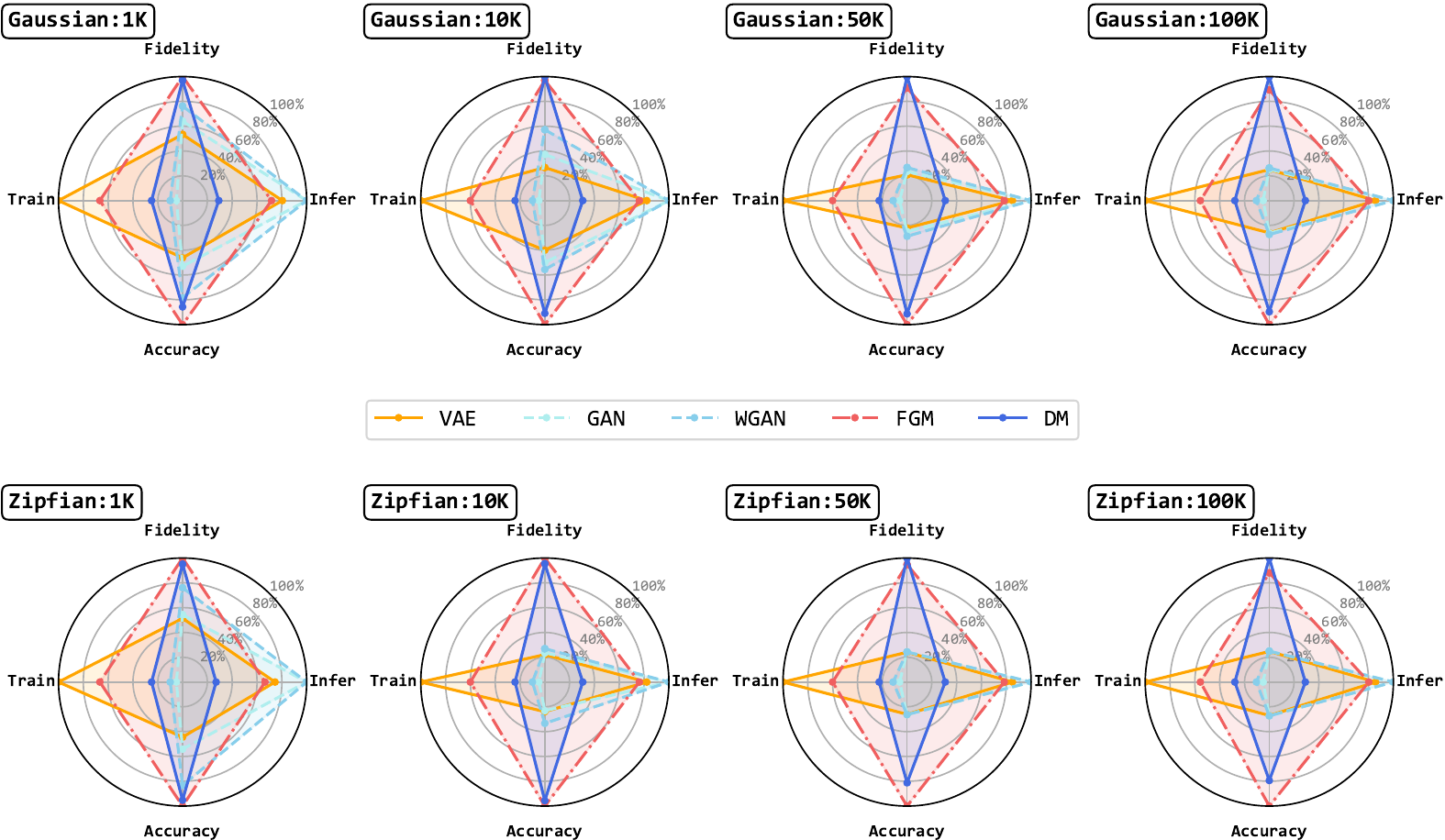}
    \caption{Performance comparison of different generative models under Gaussian and Zipfian distributions. For better visualization, metrics across different dimensions are normalized by the best result (higher the better).}
    \label{fig:gm_radar}
\end{figure}   

We still experiment on the two synthetic datasets described above, employ Count-Min as the underlying sketch, insert the GT vector into the sketch, construct the projection matrix $\boldsymbol{\Phi}$, and force the trained GM to recover the GT vector from the sketch counters $b$. As Figure~\ref{fig:gm_radar} shows, we test their sketching performance along four dimensions to embrace our principles: \textit{fidelity} measured by WMRE, \textit{accuracy} measured by ARE, \textit{infer}ence time measured by the logarithmic time (ms) taken for a single generation pass, and \textit{train}ing time (s) measured by the time taken for training the GM. On one hand, consistent with their unconditional generation results, VAEs and GANs perform adequately on small-scale data streams, but their accuracy and fidelity drop sharply as data dimensionality increases, especially under Zipfian distributions, whereas FGM and DM remain robust. On the other hand, what is not evident from unconditional generation results is that: First, despite DM consistently achieving the highest fidelity, its accuracy does not follow suit. For instance, under the Zipfian distribution, DM’s accuracy gradually falls behind that of FGM. Second, GANs suffer from extremely long training times, roughly $5\sim10$ times longer than other methods, due to their notoriously unstable adversarial training; as dimensionality increases, convergence to an optimal generator requires painstaking and delicate parameter tuning. Meanwhile, DM also faces severe inference inefficiency, as it requires lengthy iterative sampling along a Markov chain. Concretely, generating a single sample (or batch) of 100K-dimensional vectors with 1,000 diffusion steps consumes over 10 GB of GPU memory and takes tens of seconds, making it roughly $10^3\times$ slower than alternative approaches.
Given these two critical limitations, we exclude both GANs and DM from further consideration. Finally, due to the limited expressive power of VAEs, we find that FGMs offer a more favorable balance.

\begin{table}[t]
\centering
\caption{Additional comparison with flow matching on Zipfian-100K.}
\label{tab:gm_compare}
\begin{tabular}{lccc}
\toprule
GM (\# of steps) & ARE $\downarrow$ & WMRE $\downarrow$ & Gen. Time (s) \\
\midrule
DDPM (500)           & $0.22 \pm 0.04$ & $\mathbf{0.25 \pm 0.05}$ & 7.58   \\
Flow Matching (15)  & $0.32 \pm 0.11$ & $0.30 \pm 0.10$ & 0.92   \\
FGM (1)              & $\mathbf{0.18 \pm 0.07}$ & $0.28 \pm 0.06$ & \textbf{0.001} \\
\bottomrule
\end{tabular}
\end{table}

\begin{wrapfigure}[11]{r}{0.57\textwidth}
  \vspace{-0.63cm}
  \begin{center}
  \includegraphics[width=1.\linewidth]{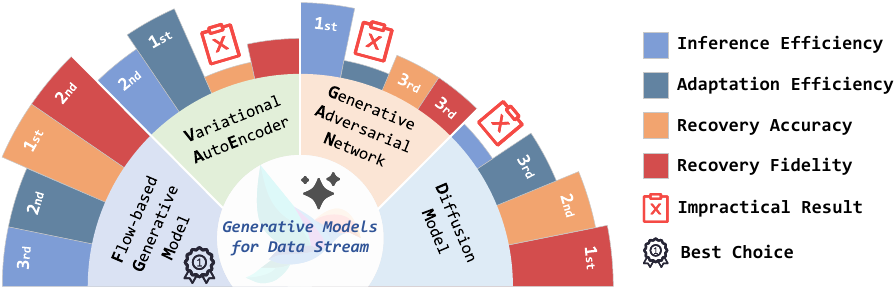}
\caption{\textbf{Performance overview across different GMs.} FGM achieves the most favorable trade-off in sketching tasks.}
\label{fig:radar}
\end{center}
\end{wrapfigure}

Compared to classical diffusion models, Flow Matching~\cite{lipman2022flow} enables efficient, simulation-free training for a broad class of transport problems using only samples from the endpoints, so it only requires a small number of sampling steps for generation. To evaluate the uncertainty it introduces to our GM selection, we conduct a simple comparative experiment on Zipfian-100K. Flow Matching adopts exactly the same model configuration as the diffusion model previously used in our work (i.e., a DDPM~\cite{ho2020denoising}), while requiring fewer than one-tenth of the sampling steps (only 15 steps). The corresponding metrics are reported in Table~\ref{tab:gm_compare}. It shows the inference latency remains non-negligible (nearly 1s) for our setting, where ms-level inference is required. Meanwhile, due to the reduced number of sampling refinement steps (i.e., $T$ in Algorithm~\ref{alg:dms}), its performance also degrades significantly compared with the 500-step diffusion model, with roughly 30\% lower accuracy and 20\% lower fidelity. We consider this an interesting direction for future work.

\textbf{Summary.} Figure~\ref{fig:radar} presents a synopsis of the examination results (each bar represents the quality achieved, with higher being better). To summarize, we have three main observations: \textcircled{\scriptsize 1} GANs suffer from notorious sensitivity to adversarial training and a brittle optimization process even if we introduce the Wasserstein optimization~\cite{arjovsky2017wasserstein}, i.e., WGANs. \textcircled{\scriptsize 2} VAEs rely on a loose surrogate objective, leading to poor performance as the size of the data stream increases (e.g., beyond $10$K). \textcircled{\scriptsize 3} Although DMs are robust and enjoy stable training, their sampling complexity is prohibitively high, typically requiring several seconds, which makes them unsuitable for real-world applications. 
In fact, it is unnecessary to perform multi-round reverse sampling computations by encoding them into the model itself.
Based on these observations, we select FGM as our underlying model due to its best trade-off, as discussed later. 

\section{\texttt{FLORE}: A Flow-based Sketching Framework}
\label{subsec:flore}

In what follows, we present and elaborate the algorithmic implementation of our proposal, \texttt{FL}ow-based \texttt{O}rthogonal \texttt{RE}covery (or \texttt{FLORE} for short). 

\begin{figure}
    \centering
    \includegraphics[width=0.9\linewidth]{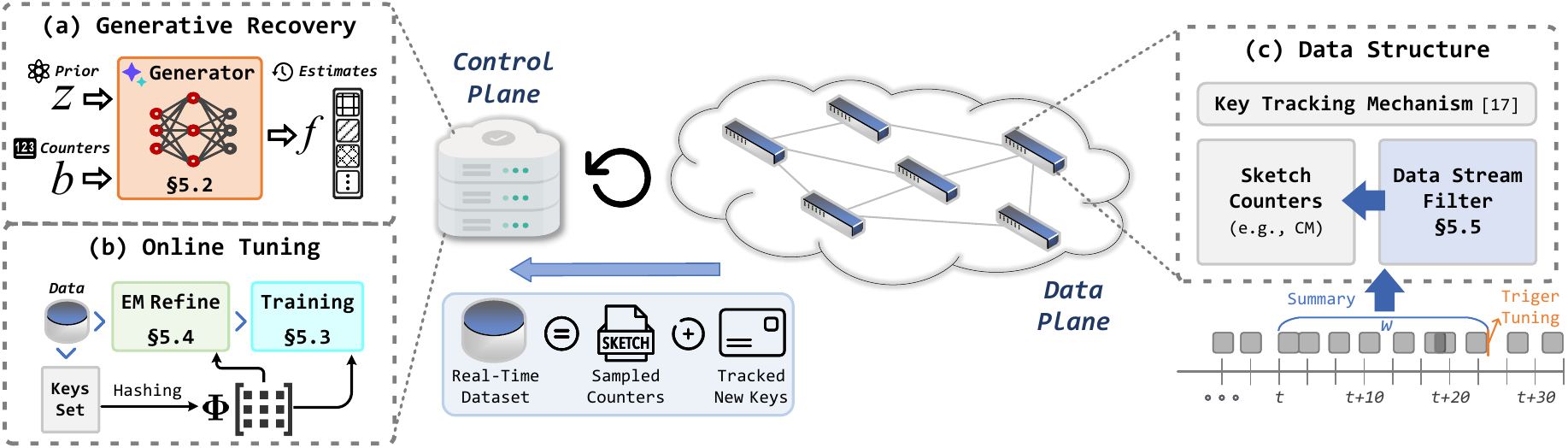}
    \caption{\textbf{System overview of \texttt{FLORE}.} (a) illustrates the control-plane recovery procedure, where flow-based generators are trained. (b) depicts the process of online tuning and scaling. The data-plane structure details are explained in (c).}
    \label{fig:system}
\end{figure}

\subsection{System Overview}
\label{subsec:sys_view}

As illustrated in Figure~\ref{fig:system}, we consider a distributed system design similar to Network Telemetry~\cite{li2016flowradar,yang2023sketchint}, where encoded sketching information are captured on local devices\footnote{\texttt{FLORE} can be extended so that multiple devices or sketches share a few base models (\S\ref{subsec:model}), instead of maintaining one model per sketch. As this is beyond the scope of this paper, we leave a deeper exploration to future work.} such as switches, and then decoded at a remote collector such as data center. We refer to the former as data-plane summary while the latter control-plane recovery. They operate in parallel, and here the control plane doesn't suffer from stringent computing resource limitations. Periodically, the data plane transmits the collected data (i.e., tracked keys and sampled counters) to the control plane for tuning GM via sliding windows (SWs)~\cite{liang2024swix}.

\textbf{\texttt{FLORE} vs existing solutions.} From the figure, what clearly sets \texttt{FLORE} apart from learning-based competitors~\cite{hsu2019learning,cao2024learning} is that the learned model resides in the control plane (also see Figure~\ref{fig:difference}), leaving the data-plane summary procedure and time-space efficiency largely unchanged. Although \texttt{CS}-based approaches~\cite{li2016flowradar,sheng2021pr} do adopt the same decoupled system design and require no training, their recovery algorithm incurs substantial overhead (e.g., computing the pseudoinverse matrix of $\boldsymbol{\Phi}$~\cite{li2023cs}), while the resulting sparse estimations often lack realness. In comparison, \texttt{FLORE} provides a generative one-step solution, uses a scalable model architecture, and handles online tuning without manually annotated data. 

\textbf{Scope of \texttt{FLORE}.} Similar to \texttt{CS}-based sketch solutions~\cite{li2016flowradar,huang2017sketchvisor,sheng2021pr,huang2021toward}, \texttt{FLORE} relies on the control plane to query the measurement results. The data plane is only responsible for updating the data structure and tracking the operator $\mathbf{\Phi}$. Thus, any applications that require computation locally or nonlinear aggregation are not targeted use cases of \texttt{FLORE}.

\subsection{Model Architecture} 
\label{subsec:model}
As with classical FGMs, \texttt{FLORE} leverages INNs to learn the target distribution. The INN is constructed as a sequence of bijective transformations ${T_i}$ with tractable Jacobian determinants, whose composition defines the overall mapping: 
\[
    f = z_K = T_K  \circ T_{K-1}  \circ \cdots  \circ T_1({z_0}), \quad z_0 = [b, z].
\]

\begin{figure}
    \centering
    \includegraphics[width=0.84\linewidth]{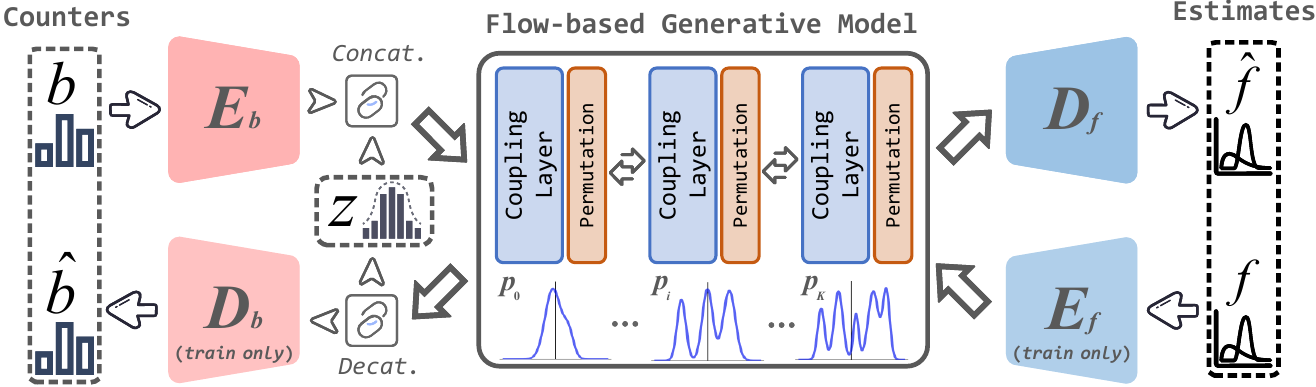}
    \caption{\textbf{The architecture of \texttt{FLORE}.} \texttt{FLORE} is modeled as an invertible mapping in the latent space.}
    \label{fig:flow_arch}
\end{figure}

\begin{figure}
    \centering
    \subfigure[Illustration of our scalable design\label{fig:scalable}]{
        \centering
        \includegraphics[width=0.39\linewidth]{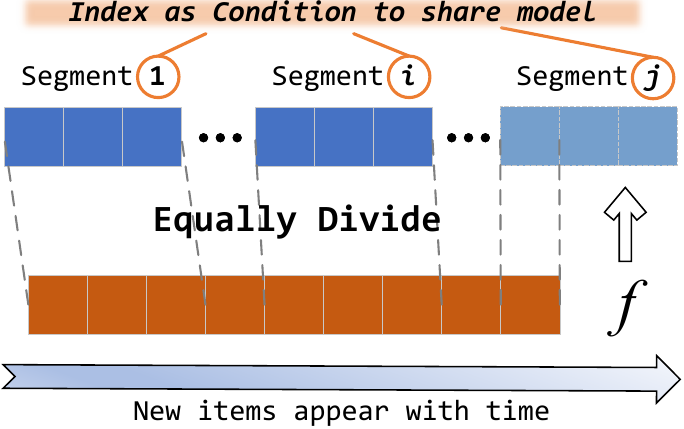}
    }
    \quad
    \subfigure[Model size v.s. Data stream size\label{fig:size}]{
        \centering
        \includegraphics[width=0.39\linewidth]{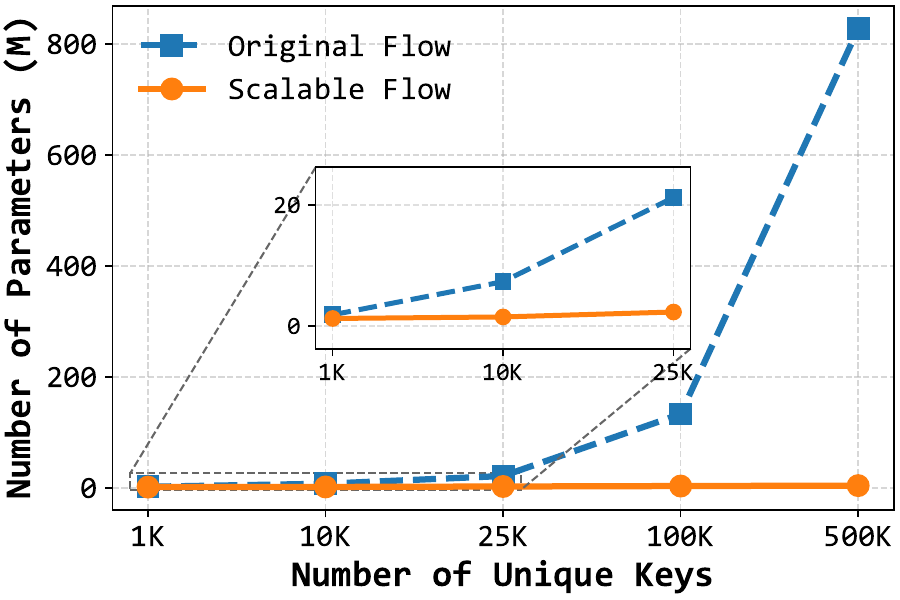}
    }
    \caption{Scalable flow-based generative model for online learning in one-pass data streams.}
\end{figure}

Similar to DMs~\cite{rombach2022high}, such a chain preserves dimensionality and thus precludes intrinsic dimension reduction within INNs. Consequently, modeling large-scale data streams (with $10^4$s unique items) makes them prone to spending excessive amounts of network parameters. So in contrast to purely INN-based approaches, our framework introduces two autoencoders to project $b$ and $f$ into lower-dimensional representations (latent space) respectively, which scales more gracefully to higher-dimensional data. 
The overall architecture is depicted in Figure~\ref{fig:flow_arch}, where the GM learns the semantic and conceptual structure of the data. The coupling operations in \texttt{FLORE} are built on the foundations of prior works~\cite{dinh2014nice,dinh2016density,kingma2018glow,ardizzone2019analyzing}. The fundamental unit of this INN is a reversible block consisting of two complementary affine coupling layers. Concretely, each block partitions its input into two components, $[u_1, u_2]$, and applies alternating affine transformations with strictly upper or lower triangular Jacobians:
\[
        v_1 = u_1 \odot \exp\left(s_1(u_2)\right) + t_1(u_2), \quad
        v_2 = u_2 \odot \exp\left(s_2(v_1)\right) + t_2(v_1).
\]
Here, $\odot$ denotes element-wise multiplication. The resulting outputs $[v_1, v_2]$ are concatenated and passed to the subsequent block. Crucially, the internal functions $s_j$ and $t_j$ can be represented by arbitrary neural networks and do not require invertibility themselves. In our implementation, these are realized by a sequence of fully connected layers with \textsc{ReLU} activations. Given the outputs $[v_1, v_2]$, the forward pass is analytically invertible:
\[
    u_2 = \left( v_2 - t_2(v_1) \right) \oslash \exp\left(s_2(v_1)\right), \quad u_1 = \left( v_1 - t_1(u_2) \right) \oslash \exp\left(s_1(u_2)\right).
\]
Here, $\oslash$ denotes element-wise division. To enhance model capacity and ensure global dependency, we interleave permutation layers~\cite{kingma2018glow} between reversible blocks. These layers shuffle the elements of the input in a fixed, pseudo-random manner, forcing the splits $[u_1, u_2]$ to vary across layers and facilitating interaction among all variables.

\textbf{Scalability under data stream uncertainty.} As discussed in \S\ref{subsec:formu}, \texttt{FLORE} may support dynamic capacity expansion to accommodate streaming data recovery. Direct retraining on the continually expanded key space would incur prohibitive computational overhead, while high-dimensional data streams (often reaching millions of dimensions) would also lead to a dramatic increase in parameter count. 

To address these challenges, we partition the GT vector $f$ into equal-length segments and incrementally increase the number of segments over time to adapt to those new elements (see Figure~\ref{fig:scalable}). To maintain parameter efficiency, these segments are able to share the same network conditioned on their respective indices. Such a design ensures that when the brand-new keys arrive, only the parameters associated with the most recent index require fine-tuning, thereby minimizing the impact on existing segments. As a result, the model's parameter complexity is reduced from $\mathcal{O}(N)$ to $\mathcal{O}(\text{\textsc{segment length}})$. We empirically demonstrate this benefit in Figure~\ref{fig:size} by plotting how the number of \texttt{FLORE}'s parameters scales with the size of the data stream (i.e., data dimensions). As can be seen, the scalable design substantially reduces the parameter count, achieving a reduction of over 200 times at 500K unique keys. In our context, this approach is implemented using a Conditional INN (cINN)~\cite{ardizzone2019guided}. Since the subnetworks $s_j$ and $t_j$ are not explicitly inverted, they can be conditioned on condition $c$ via concatenation without compromising the overall invertibility of the block (i.e., replacing $s_1(u_2)$ with $s_1(u_2, c)$). Specifically, the index $c$ is represented as a one-hot vector and processed by a feedforward conditioning network to extract intermediate representations, which are trained jointly with the cINN. Figure~\ref{fig:architecture} illustrates the final architecture; note that two autoencoders are omitted from the diagram for clarity.

\begin{figure}
    \centering
    \includegraphics[width=0.84\linewidth]{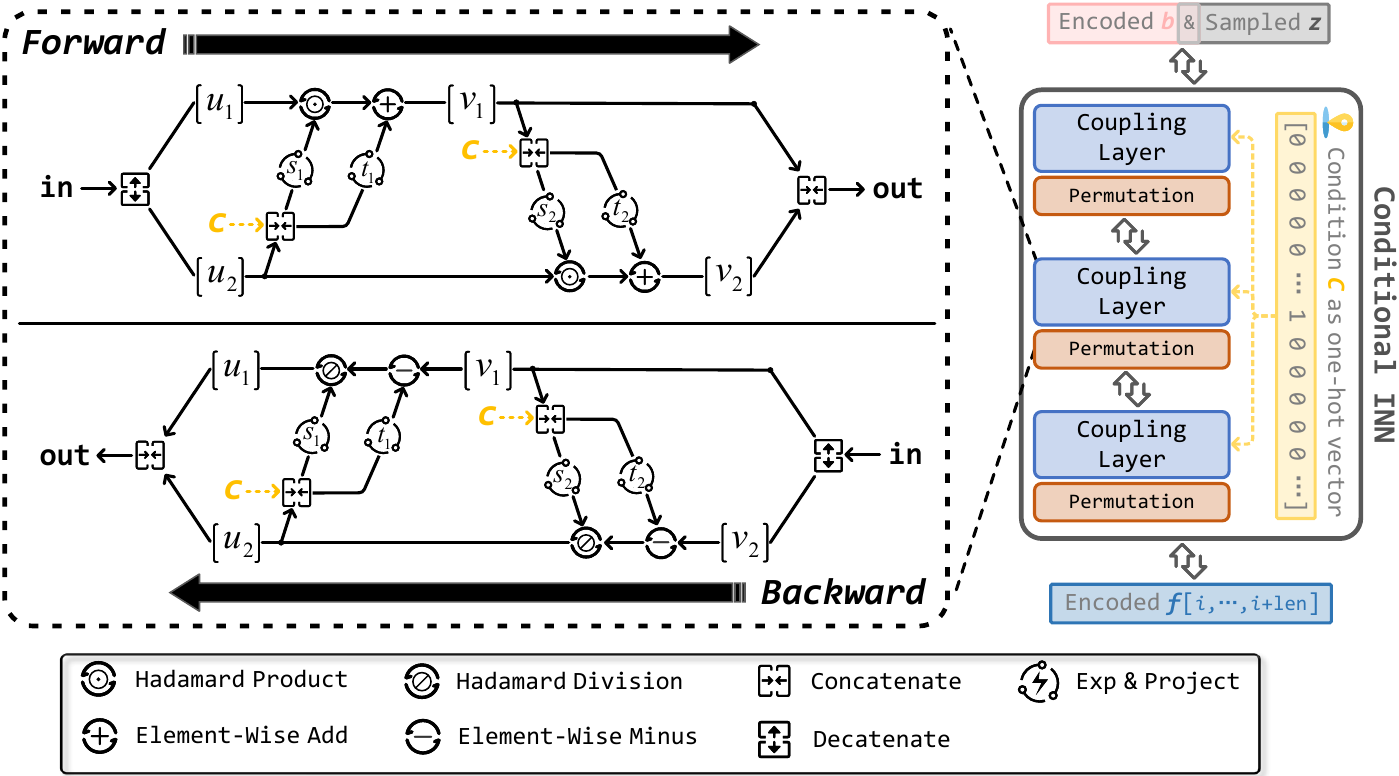}
    \caption{The underlying cINN architecture used in \texttt{FLORE}.}
    \label{fig:architecture}
\end{figure}    

\subsection{Training Objective} 

Below, we shift focus to the training scheme of \texttt{FLORE} that enforces both the measurement consistency and the fidelity condition. In our design, the overall objective consists of three types of training loss terms as follows. 

\textbf{\rbox{L1}} For aligning with the observation, there are two loss terms: \textit{consistency} loss $\mathcal{L}_{\mathbf{con}}$ and \textit{reconstruction} loss $\mathcal{L}_{\mathbf{rec}}$. They penalize the deviations between $b$ and $f$ respectively. Let $f_G := G([b,z])$, then we can write 
\[
    \mathcal{L}_{\mathbf{con}} = \mathbb{E} \left[ \left( \boldsymbol{\Phi}f_G - b \right)^2 \right], \quad \mathcal{L}_{\mathbf{rec}} = \mathbb{E} \left[ \left( f_G - f \right)^2 \right].
\]

\textbf{\rbox{L2}} For modeling the distribution, we also have two loss terms: \textit{invertibility} loss $\mathcal{L}_{\mathbf{inv}}$ and \textit{orthogonal} loss $\mathcal{L}_{\mathbf{ort}}$. The former forces \texttt{FLORE} to remain its invertibility by
\[
    \mathcal{L}_{\mathbf{inv}} = \mathbb{E} \left[ \left( f -G\left(G^{-1}(f) \right) \right)^2 \right].
\]
While the latter embraces the theoretical analysis presented in \S\ref{subsec:fail}, which aims to make the distributions of $z$ and $b$ as orthogonal as possible, i.e., $q(b, z) = p_B(b)p_Z(z)$. 
Using an appropriate $\mathcal{D}$ that measures the similarity of two distributions, e.g., Kullback-Leibler Divergence (KLD)~\cite{kullback1951kullback} or Maximum Mean Discrepancy (MMD)~\cite{gretton2012kernel}, we can define it as 
\[
    \mathcal{L}_{\mathbf{ort}} = \mathcal{D} \left(  q(b,z) \parallel p_B(b)p_Z(z) \right).
\]

\begin{table}[t]
\centering
\caption{Ablation of the Jacobian Loss Term}
\label{tab:jacobian}
\begin{tabular}{l|c|c|c}
\toprule
\textbf{Method (256KB)} & \textbf{ARE $\downarrow$} & \textbf{WMRE $\downarrow$} & \textbf{Training Iters} \\
\midrule
\texttt{FLORE} w/ Jac. & $1.20 \pm 0.78$ & $\mathbf{0.41 \pm 0.18}$ & $200\mathrm{s}\sim400\mathrm{s}$ \\
\texttt{FLORE} w/o Jac.       & $\mathbf{1.03 \pm 0.38}$ & $0.45 \pm 0.10$ & $\mathbf{50\mathrm{s}\sim100\mathrm{s}}$ \\
\bottomrule
\end{tabular}
\end{table}

\begin{tcolorbox}[
    colback=blue!3,
    colframe=blue!50,
    boxrule=0.5pt,
]
{\large$\diamond$}\ \textbf{Note:}  During the training of \texttt{FLORE}, we don't incorporate the traditional FGM's negative log-likelihood loss term, i.e., $\log p(z_0) - \log\sum_{i=1}^K \left|\det \frac{dT_i}{dz_{i-1}}\right|$, into the objectives~\cite{xiao2020invertible,liang2021hierarchical}. Empirically, we observe that it tends to slow down model convergence (see Table~\ref{tab:jacobian} and CAIDA-2018 is the default dataset hereafter); in theory, Theorem~\ref{thm:convergence} shows that the aforementioned losses alone suffice to yield an unbiased estimation of the target (conditional) distribution. 
\end{tcolorbox}

\begin{lightgraybox}
\begin{theorem}
    \textit{If the designed loss terms, i.e., $\mathcal{L}_{\mathbf{con}}$, $\mathcal{L}_{\mathbf{inv}}$ and $\mathcal{L}_{\mathbf{ort}}$ vanish, the invertible mapping $G: \left[b, z\right] \leftrightarrow f$ applied to the $p(b, z)$ recovers the true posterior $p(f \mid b)$. }
    \label{thm:convergence}
\end{theorem}
\end{lightgraybox}

\begin{proofbox}
\begin{proof}
    Suppose we take a posterior conditioned on a fixed $b$, i.e., $p(f \mid b)$, and transform it using the forward pass of our perfectly converged INN.
    First, if the orthogonal loss $\mathcal{L}_{\mathbf{ort}}$ converges to zero, we have $ q(b,z) = p_B(b)p_Z(z)$, according to Definition~\ref{def:kld} (for KL divergence) and Lemma~\ref{lem:mmd} (for MMD). Then because $\mathcal{L}_{\mathbf{con}} = \mathbb{E} \left[ \left( \boldsymbol{\Phi}f - b \right)^2 \right] \rightarrow 0$ and $\mathcal{L}_{\mathbf{inv}} = \mathbb{E} \left[ \left( f -G\left(G^{-1}(f) \right) \right)^2 \right] \rightarrow 0$, we know that the INN's output distribution of $b'$ (marginalized over $z$) must be \( p_B'(b') = \delta(b' - b) \) where $\delta (\cdot)$ is the Dirac delta function. Also, because of the independence between $z$ and $b'$ in the output, the output distribution of $z$ is still pre-defined $p_Z(z)$. So the joint distribution of outputs is $q(b',z) = \delta(b' - b) p_Z(z)$.

    When we run forward the INN, and repeatedly input $b$ while sampling $z \sim p_Z(z)$, this is the same as sampling $\pi := [b', z]$ from the joint distribution $q(b',z)$ above. Next, we show that the INN $G$ will transform these samples back to $p(f \mid b)$ or $p_F(f)$ for convenience.
    
    Note that by the change of variables theorem, we can write the backward transformation as 
    \( q(\pi) = p_F\left(f=G \left( \pi \right) \right) \left| \operatorname{det} \frac{dG}{d\pi} \right| \) where $\operatorname{det} \frac{dG}{d\pi}$ is the Jacobian determinant of the bijective function. Similarly, for the forward direction, we have $p'_F(f) = q\left(\pi =G^{-1}\left( f \right)\right) \left| \operatorname{det} \frac{dG^{-1}}{df} \right|$. Putting it together, we have
    \[
        p'_F(f) = p_F\left(f=G \left(G^{-1}\left( f \right) \right) \right) \left| \operatorname{det} \left[ \frac{dG}{d\pi} \frac{dG^{-1}}{df} \right] \right| = p_F(f) \left| \operatorname{det} \left[ \mathbf{Id} \right] \right| = p_F(f) = p(f \mid b),
    \]
    which concludes our proof.
\end{proof}
\end{proofbox}

\textbf{\rbox{L3}} To instill \textit{sparsity} into \texttt{FLORE}'s recovery (i.e., the idea of \texttt{CS}), we append a regularization term $\mathcal{L}_{\mathbf{sp}} = \left\|f_G\right\|_1$.

Taken together, we train \texttt{FLORE} by minimizing the loss
\begin{equation}
   \mathcal{L}_{\mathbf{total}} = \mathcal{L}_{\mathbf{con}} + \alpha_1 \mathcal{L}_{\mathbf{rec}} + \alpha_2 \mathcal{L}_{\mathbf{inv}} + \alpha_3 \mathcal{L}_{\mathbf{ort}} + \alpha_4 \mathcal{L}_{\mathbf{sp}}
   \label{eqn:loss}
\end{equation}
where $\alpha_{(\cdot)}$ are trade-off parameters controlling the relative contributions of individual loss terms.
After training, the learned generator $G$ can be directly deployed to recover the streaming ground truth $f$ from sampled inputs $[b, z]$, for both training data and previously unseen observations.

\begin{algorithm}[H]
\caption{(Pytorch) EM refinement pseudocode for \texttt{FLORE}}
\label{alg:em}
\begin{lstlisting}[language=Python]
# f: Count-Min initial solution  
# b, Phi: sketch observations and matrix 
# T: number of EM steps
                        
for t in range(T):
    # E-step: compute normalization factors
    C1 = Phi/(Phi@f)
    C2 = f/Sum(Phi, axis=0)  
    
    # M-step: multiplicative update
    f = C2 * WeightedSum(b, C1)                             
\end{lstlisting}
\end{algorithm}

\subsection{GT-free Learning}

It's often extremely hard to construct datasets that capture all aspects of very large databases in high-speed environments or testbeds. Hence, a particular challenge against this backdrop is to train \texttt{FLORE} without ground truth (only measurement $b$ is available). Fortunately, taking into account the recent history of those counters of a linear sketch data structure, we are able to automatically infer hidden GT patterns. The underlying original sketch approximation rules, e.g., Count-Min and Count, provide a natural inductive bias for \texttt{FLORE}'s learning~\cite{cheng2023alsketch}. Nonetheless, naively substituting the true $f$ in the aforementioned objectives by querying the sketch would destabilize the training procedure, because these estimations are quite coarse in general as discussed in \S\ref{subsec:lim}. 

\begin{figure}
    \centering
    \subfigure[AAE on the CAIDA dataset]{
        \centering
        \includegraphics[width=0.475\linewidth]{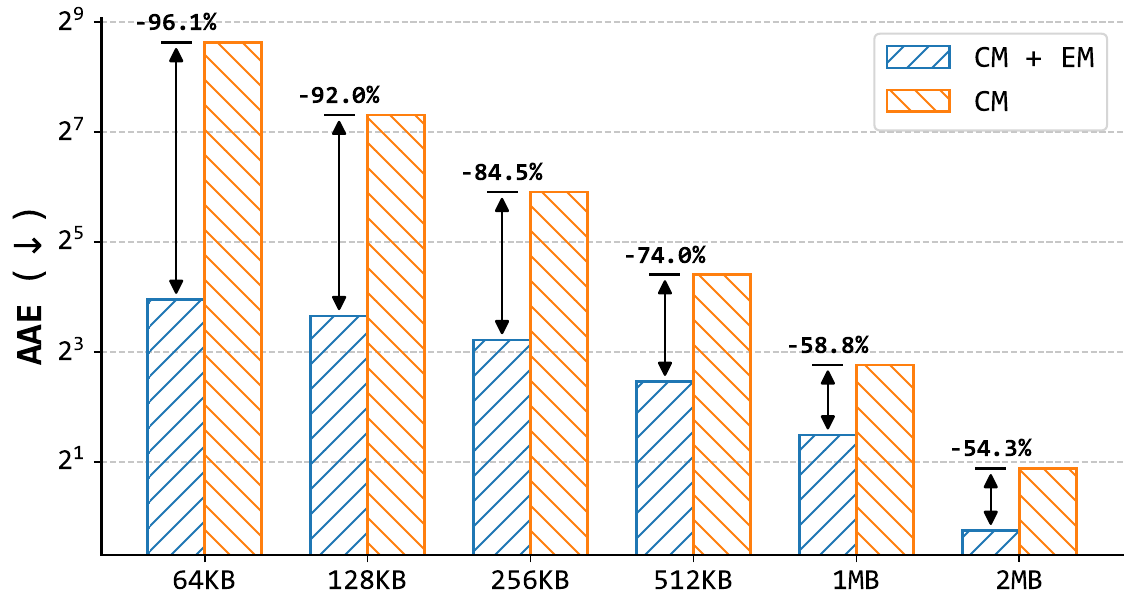}
    }
    \quad
    \subfigure[ARE on the CAIDA dataset]{
        \centering
        \includegraphics[width=0.475\linewidth]{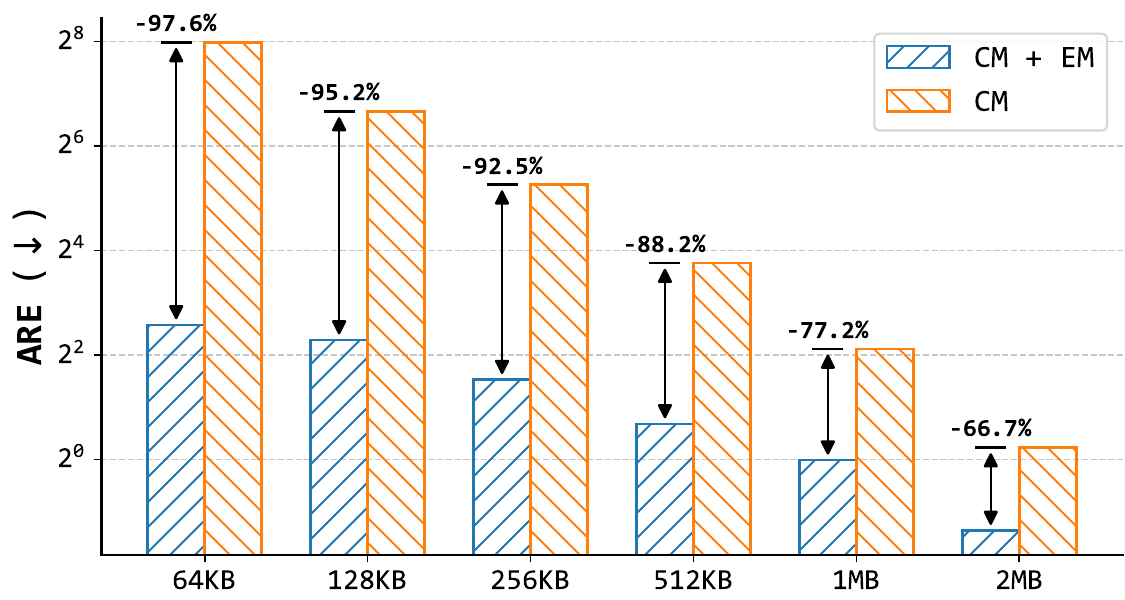}
    }\\
    \subfigure[AAE on the MAWI dataset]{
        \centering
        \includegraphics[width=0.475\linewidth]{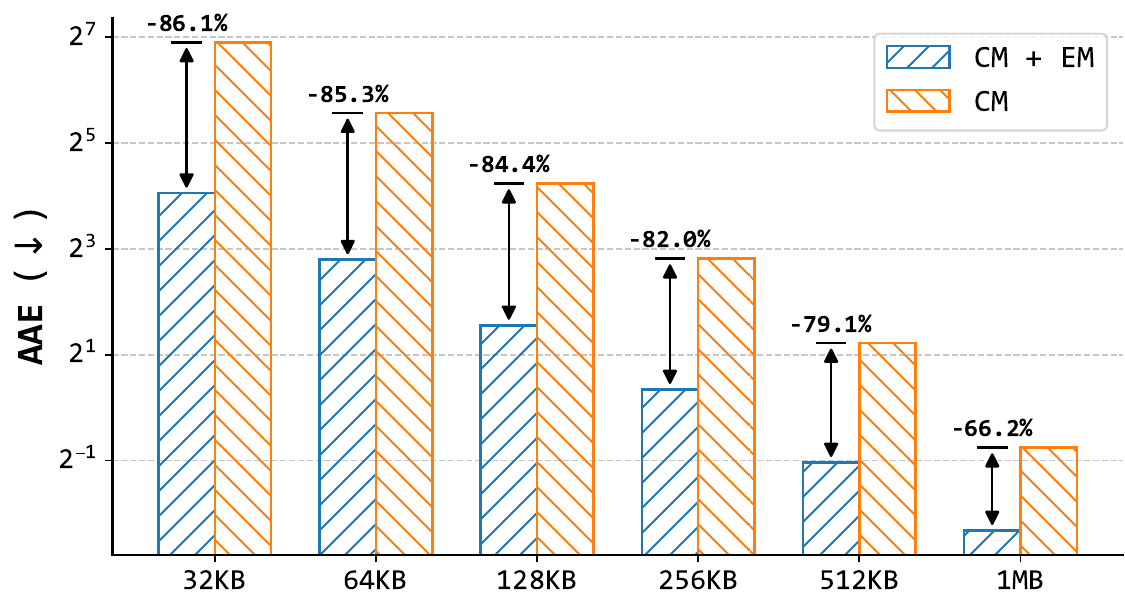}
    }
    \quad
    \subfigure[ARE on the MAWI dataset]{
        \centering
        \includegraphics[width=0.475\linewidth]{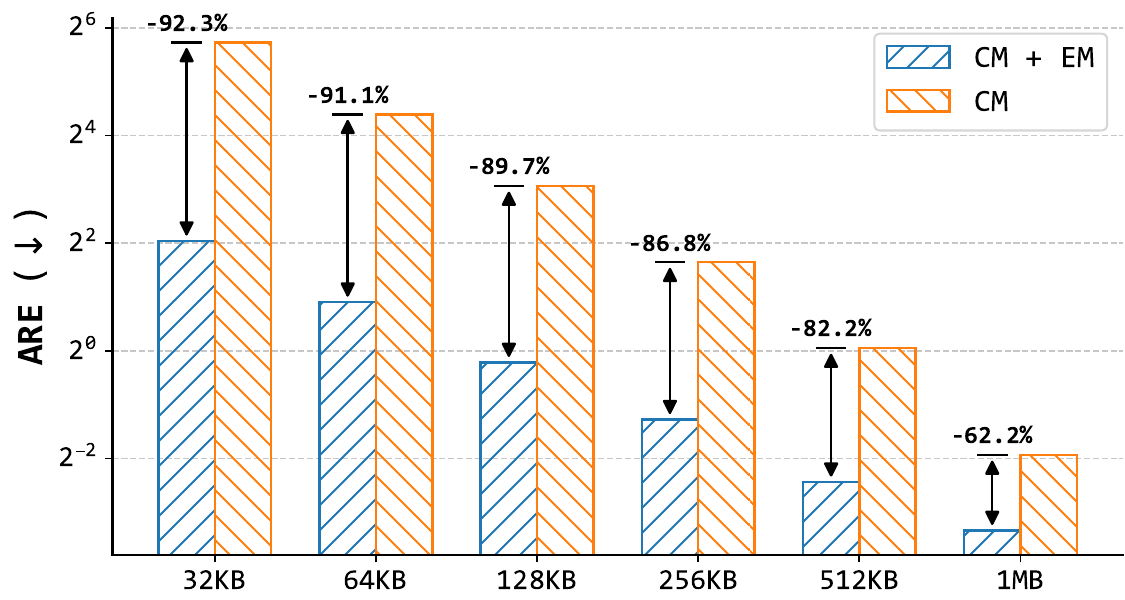}
    }
    \caption{Ablation comparing Count-Min Sketch and Count-Min Sketch + EM refinement (with ten steps).}
    \label{fig:ab_em}
\end{figure}

To address this challenge, we connect the posterior approximation problem with maximum likelihood estimation (MLE). Specifically, we can apply an EM algorithm to refine estimates obtained from the underlying sketch, and this results in the iterations described in Theorem~\ref{thm:em} as follows.

\begin{lightgraybox}
\begin{theorem}
    \textit{Minimizing the Kullback–Leibler Divergence $\mathcal{D}_{\text{KL}}\left(p\left(f \mid b) \parallel p(f^{(\cdot)}\right)\right)$ yields the following EM update step: $f_i^{(n)} = f_i^{(n-1)} \sum_j \frac{\boldsymbol{\Phi}^\top_{i,j}}{(\boldsymbol{\Phi}f^{(n-1)})_j} b_j$, initialized with $f_i^{(0)}>0$.} 
    \label{thm:em}
\end{theorem}
\end{lightgraybox}

\begin{proofbox}
\begin{proof}
    Without loss of generality, we start by a rescaled version of the linear sketching problem:
    \[
        g_j = \sum_i f_i h_{i,j}, \quad \text{s.t.} \ \ f_i \geq 0, \ \sum f_i = 1, \ g_j = \frac{b_j}{\sum_j b_j}, \ \text{and} \ h_{i, j} = \frac{\boldsymbol{\Phi}^\top_{i,j}}{\sum_i \boldsymbol{\Phi}^\top_{i,j}} = \frac{\boldsymbol{\Phi}^\top_{i,j}}{k_j}, 
    \]
    which transforms all measures, i.e., $f$ and $b$, into probability measures. Then, if the obtained solution is given, the solution can be transformed as follows: $\hat f = \frac{k_j}{\sum_j b_j} f = \frac{k}{\sum_j b_j} f$. 
    
    Our formulation can also be regarded as a statistical model. Suppose we observe $ N $ samples $ Y_1, \ldots, Y_N $, which are independently drawn from a distribution $ g $. If the distribution $ g(y) $ is given by $ g(y) = \int h(x, y) f(x) \, dx$ where $h(x, y)$ is the corresponding kernel function of $h_{(\cdot)}$, this is equivalent to the following two-step sampling process: first draw $ X $ from $ f(x) $, and then draw $ Y $ from the conditional distribution $ h(x, \cdot) $ given $ X $. In other words, the observed data $ Y $ arise from a latent variable model:
    $$
        X \to Y, \quad \text{where } X \sim f \text{ and } Y \mid X = i \sim h_{i,j}.
    $$
    As a result, the joint and marginal probabilities are given by
    $$
        P(X = i, Y = j) = f_i h_{i,j}, \quad P(Y = j) = g_j = \sum_i f_i h_{i,j}.
    $$
    This structure constitutes a standard latent variable model, which is formally very similar to a Gaussian Mixture Model (GMM), except that in this case, the transition probabilities $ h_{i,j} $ are fixed and known.
    If we only observe the data $ Y_1, \ldots, Y_N $, the likelihood function is given by
    $
        L_N(f) = \prod_j \left( \sum_i f_i h_{i,j} \right)^{n_j},
    $
    where $ n_j $ denotes the number of times $ Y $ takes the value $ j $, and $ \sum_j n_j = N $.
    Therefore, the log-likelihood is
    $
        \log L_N(f) = \sum_j n_j \log \left( \sum_i f_i h_{i,j} \right).
    $
    If we express the empirical distribution as
    $
        g_j = \hat g_j = \frac{n_j}{N},
    $
    then we obtain
    \begin{equation}
        \frac{1}{N} \log L_N(f) = \sum_j g_j \log \left( \sum_i f_i h_{i,j} \right)
        \label{eqn:logll}
    \end{equation}

    Next, we show its connection with KL divergence. First, we can interpret the sketching problem $g_j = \sum_i f_i h_{i,j}$ as solving the following optimization problem: $\min_{f \geq 0,\ \sum f_i = 1} \mathcal{D}_{\text{KL}}\left( g \parallel H^\top f \right)$ where $ H^\top f $ represents the distribution induced by $ f $ through the transition matrix $ H $, i.e., $ (H^\top f)_j = \sum_i f_i h_{i,j} $. In other words, we seek a distribution $ f $ such that the resulting mixture distribution $ H^\top f $ is as close as possible to the observed empirical distribution $ g $, measured in terms of KL divergence. By Definition~\ref{def:kld}, this minimization problem is equivalent to:
    \(
        \max_f \sum_j g_j \log \left( \sum_i f_i h_{i,j} \right),
    \)
    which is precisely Eqn.~(\ref{eqn:logll}).

    Now we return to the latent variable interpretation introduced earlier. Let us clarify the assumptions and notation: \textcircled{\scriptsize 1} Complete data: $ (X, Y) $, with joint distribution $ P(X = i, Y = j) = f_i h_{i,j} $; \textcircled{\scriptsize 2} Observed data: only $ Y $ is observed, $P(Y = j) = \sum_i f_i h_{i,j}$. This is precisely the definition of an incomplete data problem.
    
    Consequently, the log-likelihood based on the complete data is 
    \[ \log L_C(f) = \sum_{i,j} n_{i,j} \log(f_i h_{i,j}) = \sum_{i,j} n_{i,j} \log f_i + \text{constant} \] 
    where $ n_{i,j} $ is the number of times $ (X = i, Y = j) $ occurs. However, in practice, we can only compute the log-likelihood based on the observed data $\log L_N(f) = \sum_j n_j \log\left( \sum_i f_i h_{i,j} \right)$. Therefore, this problem is well-suited for iterative solution using the EM algorithm, which alternates between estimating the missing data (Expectation step, E-step) and maximizing the likelihood (Maximization step, M-step) under the current estimate of the parameters.

    Formally, in each iteration of the EM algorithm, we can proceed as follows:
    
    \textbf{\textbullet\ E-step:} Based on the current estimate $ f^{(n-1)} $, we compute the hidden responsibility probabilities, i.e., the posterior probability that a given observation $ Y = j $ originated from latent component $ X = i $. This is given by Bayes' rule:
    \begin{equation}
        r_{i,j}^{(n-1)} := P^{(n-1)}(X = i \mid Y = j) = \frac{P^{(n-1)}(X = i, Y = j)}{P^{(n-1)}(Y = j)} = \frac{f_i^{(n-1)} h_{i,j}}{\sum_k f_k^{(n-1)} h_{k,j}}
        \label{eqn:prob}
    \end{equation}
    Then, we have the estimate $\mathbb{E}\left[n_{i,j}\right] = n_j P^{(n-1)}(X = i \mid Y = j) = n_j r_{i,j}^{(n-1)}$.

    \textbf{\textbullet\ M-step:} We update the parameter $ f_i $ by maximizing the expected complete-data log-likelihood with respect to $ f $. We first construct the Q-function $Q(f\mid f^{(n-1)}):=\mathbb{E} \left[ \log L_C(f) \right] = \sum_{i,j} n_j r_{i,j}^{(n-1)} \log f_i + \text{constant}$. So, the M-step is equivalent to
    \[
        \max_f \sum_{i,j} n_j r_{i,j}^{(n-1)} \log f_i \quad \text{subject to} \quad f_i \geq 0 \ \text{and} \ \sum_i f_i = 1
    \]
    We can utilize Lagrange Multipliers to solve this problem. Define $L = \sum_{i,j} n_j r_{i,j}^{(n-1)} \log f_i + \lambda \left( \sum_i f_i - 1 \right)$, and take the partial derivative with respect to $f_i$: \(\frac{\partial L}{\partial f_i} = \frac{\sum_j n_j r_{i,j}^{(n-1)}}{f_i} + \lambda\). Set it to zero, we have $f_i = -\frac{\sum_j n_j r_{i,j}^{(n-1)}}{\lambda}$. Applying the constraint $\sum_i f_i = 1$, it yields \(-\lambda = \sum_i \sum_j n_j r_{i,j}^{(n-1)}\). Since $\sum_i r_{i,j}^{(n-1)} = 1$, we have $-\lambda = \sum_j n_j = N$. Thus, the updated estimate is obtained as
    
    \begin{equation}
        f_i^{(n)} = \frac{1}{N} \sum_j n_j r_{ij}^{(n-1)} = \sum_j g_j \, P^{(n-1)}(X = i \mid Y = j)
        \label{eqn:rule}
    \end{equation}

    Substituting the expression for the conditional probability, i.e., Eqn.~(\ref{eqn:prob}), into the update rule, i.e., Eqn.~(\ref{eqn:rule}), we have
    \[
        f_i^{(n)} = f_i^{(n-1)} \sum_j \left( \frac{h_{i,j}}{\sum_k f_k^{(n-1)} h_{kj}} \right) g_j  \quad \Longleftrightarrow \quad \hat f_i^{(n)} = \hat f_i^{(n-1)} \sum_j \frac{\boldsymbol{\Phi}^\top_{i,j}}{(\boldsymbol{\Phi} \hat f^{(n-1)})_j} b_j.
    \]
    This recurrence relation provides an iterative scheme for updating the distribution $ f $, and it corresponds exactly to the formula presented in Theorem~\ref{thm:em}, which concludes the proof.
\end{proof}
\end{proofbox}

If the given linear sketching problem does not have a non-negative solution, then the EM algorithm can be proven to yield the closest approximation of $f$ in terms of KLD~\cite{vardi1993image}. Empirically, Figure~\ref{fig:ab_em} compares the simple CM and CM equipped with 10-step EM refinement (as indicated by ``CM + EM'') when testing in the CAIDA dataset (1st row) and MAWI dataset (2nd row). It can be seen that the enhancement brought by EM is truly impressive. To be more specific, on the CAIDA dataset, it reduces CM’s AAE by approximately 76\% ARE by about 86\%. Similarly, on the MAWI dataset, it lowers CM’s AAE by roughly 80\% and ARE by around 84\%. These results are consistent with Theorem~\ref{thm:em}.
From an interpretative perspective, we can view this EM algorithm as ``denoising'' the Count-Min solutions using global information, i.e., the linear system $b = \boldsymbol{\Phi} f$. To realize it, we have the following formula:
\begin{equation}
    {f_{i}^{\textbf{new}}} \leftarrow \frac{{{f_i}^{\textbf{old}}}}{{\sum\nolimits_{j = 1}^m {{\boldsymbol{\Phi}^\top_{i,j}}} }}\sum\nolimits_{j = 1}^m {\frac{{{\boldsymbol{\Phi}^\top_{i,j}} \times {b_j}}}{{\sum\nolimits_{k = 1}^N {{\boldsymbol{\Phi}_{i,k}} \times {f_k}} }}}
\end{equation}
While the update is defined on a per-key basis, it admits GPU parallelism. Algorithm~\ref{alg:em} illustrates the EM algorithm's pseudocode, in which all batch-wise operations are implemented via tensor primitives.
Finally, Algorithm~\ref{alg:training} depicts the complete training procedure of \texttt{FLORE} in a Pytorch-like style. It uses only a dataset of sketch counters. In the forward pass, we first compute f0 as an estimate of the actual f, which is a pseudo-label EM solution. We subsequently ensure data consistency by reconstructing f1 from compressed f. According to the invertible property of INNs, we pass f through the network for one forward-backward pass to obtain f2.

\begin{algorithm}[H]
\caption{(Pytorch) Training pseudocode for \texttt{FLORE}}
\label{alg:training}
\begin{lstlisting}[language=Python]
# Phi: linear sketching matrix (Count-Min) 
# G.forw, G.back: forward and reverse process of GM G
# a: list of trade-off parameters (alphas)

for b in loader:      # load a minibatch from the SW
    # approximate GT via EM (can be pre-computed)
    f = EM_OPT_CM(b, Phi)      # CM as the start point
    z = Gaussian(0,I)          # sample the prior z
    f0 = G.forw([b,z])         # reconstruct f from b
    f1 = G.forw([Phi@f,z])     # reconstruct f from Phi@f
    f2 = G.forw(G.back(f))     # reconstruct f from itself
    bz = G.back(f0)            # get the joint distribution
    
    # training loss, Eqn. (\ref{eqn:loss})
    loss = L2Loss(Phi@f0, b)              # consistency loss term, i.e., $\mathcal{L}_{\mathbf{con}}$
         + a[0]*L2Loss(f1, f)             # reconstruction loss term, i.e., $\mathcal{L}_{\mathbf{rec}}$
         + a[1]*L2Loss(f2, f)             # invertibility loss term, i.e., $\mathcal{L}_{\mathbf{inv}}$
         + a[2]*MMD_or_KLD([b,z], bz)     # orthogonality loss term, i.e., $\mathcal{L}_{\mathbf{ort}}$
         + a[3]*L1Loss(f0)                # sparsity loss term, i.e., $\mathcal{L}_{\mathbf{sp}}$

    loss.backward()       # compute the gradient
    update(G.params)      # update G network
\end{lstlisting}
\end{algorithm}

\begin{table}[t]
\centering
\caption{Time-accuracy Trade-off of the EM Algorithm}
\label{tab:em}
\begin{tabular}{l|cccccc}
\toprule
\textbf{\# of EM Steps} & 1 & 3 & 5 & 10 & 20 & 50 \\
\midrule
\textbf{Total Time (s) $\downarrow$} & 0.54 & 1.72 & 3.22 & 6.85 & 13.91 & 31.05 \\
\midrule
\textbf{Normal. ARE $\downarrow$} & 9.71 & 2.04 & 1.55 & 1.41 & 1.40 & 1.39 \\
\bottomrule
\end{tabular}
\end{table}

\textbf{Why is EM alone \emph{not} a good solver?} 
One may naturally ask whether the EM algorithm can be directly applied to linear sketch recovery. 
Although parallel computation can accelerate its iterative process, it comes at the cost of memory-expensive per-key-wise optimization (e.g., $10$s GB GPU memory usage for $10^4$s keys). Furthermore, Table~\ref{tab:em} reports the impact of different numbers of EM update steps from two perspectives: normalized ARE (divided by that of \texttt{FLORE}) and execution time. We can see that it consistently underperforms our method which has \emph{ms}-level latency (\S\ref{main_sec:exps}), and the benefit drops sharply after 5-10 steps.
A core reason is that each EM-based recovery is solved independently and thereby fails to exploit collective gains (i.e., the underlying distribution), rendering EM not even an ``overkill''.
Thus, in our implementation, where tuning is not subject to strict query response time constraints, the EM is performed 3–5 steps as a \emph{one-time} preprocessing operation.

\subsection{Data Structure}

\begin{figure}
    \centering
    \includegraphics[width=0.9\linewidth]{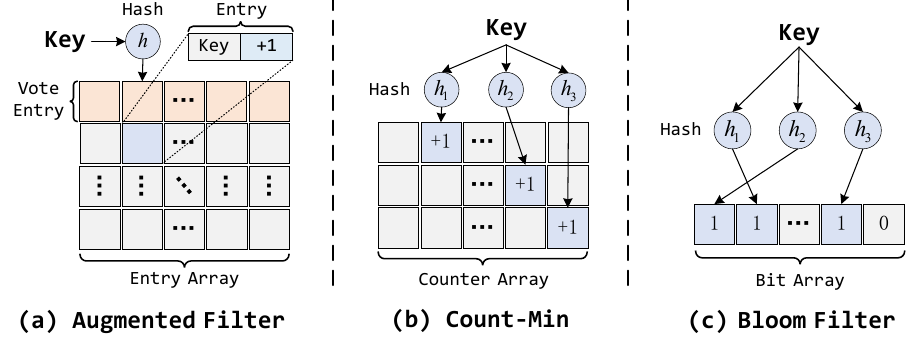}
    \caption{Data structures used in \texttt{FLORE}: a Bloom Filter for key tracking, a Count-Min Sketch for store infrequent items, and an augmented filter for saving frequent items as key-value pairs.}
    \label{fig:structure}
\end{figure}

In this part, we provide data-plane details of \texttt{FLORE}. As illustrated in Figure~\ref{fig:structure}, the local data structure of \texttt{FLORE} consists of three components: (1) an augmented filter for storing frequent elements, (2) a Count–Min Sketch for recording the remaining elements, and (3) a Bloom Filter for tracking newly observed keys. We first describe the Bloom Filter–based key tracking mechanism and the stream filtering algorithm for automatically separating hot items from the data stream into the augmented filter.

\subsubsection{Key Tracking Mechanism}
\label{subsec:track}

As noted before, constructing $\boldsymbol{\Phi}$\footnote{We do not explicitly store the 0–1 sparse matrix $\boldsymbol{\Phi}$ in memory. Instead, we generate its entries on demand during adaptation.} requires tracking all distinct inserted keys, either during training or/and at query time (retroactively). While we do not concern ourselves with this problem and assume that the key set $\mathcal{I}$ is known, we briefly introduce a standard key tracking mechanism commonly adopted in prior studies~\cite{eppstein2011s,xiong2014kbf,li2016flowradar,sheng2021pr,huang2021toward} to collect $\mathcal{I}$ from the data stream. 

The mechanism consists of two parts: key identification and key recording. Specifically, key identification part maintains a lightweight data structure, i.e., a Bloom Filter, to identify keys in the summary phase (or data plane). As depicted in Figure~\ref{fig:structure}(c), the Bloom filter is implemented as a one-dimensional array with $m_{bf}$ entries and is associated with $k_{bf}$ independent hash functions. Each entry stores a single bit indicating whether the mapped input corresponds to a key that has not been observed before. For each incoming item, we examine the bit statuses of the $k_{bf}$ mapped entries prior to updating them.
Once a new key is identified in the summary phase, it is forwarded to the recovery phase (or control plane). In the recovery phase, the key recording component stores these keys for subsequent analysis. This offloading strategy reduces the resource consumption of the key-tracking mechanism, as each unseen key is identified and transmitted at most once. 

\begin{figure}[ht]
    \centering
    \begin{minipage}[t]{0.49\textwidth} 
        \begin{algorithm}[H]             
        \caption{\texttt{FLORE} Data-plane Summary}
        \label{alg:update}
        \begin{algorithmic}[1]           
            \ENSURE Encode $(\text{key}, v)$ into \texttt{FLORE} data structures
            \STATE $i \leftarrow \text{\texttt{Hash}}(k)$
            \STATE \textsc{Array} $\leftarrow$ \textsc{AugmentedFilter}[$i$]
            \STATE Lookup $k$ in \textsc{Array}[$1 \sim \text{end}$]
            \STATE ~
            \IF{\underline{Case 1}: item found at position $j$} 
            \STATE \textsc{Array}[$j$].value + $v$ $\rightarrow$ \textsc{Array}[$j$].value
            \ELSIF{\underline{Case 2}: position $j$ is empty}
            \STATE $k$ $\rightarrow$ \textsc{Array}[$j$].key ~and~ $v$ $\rightarrow$ \textsc{Array}[$j$].value
            \ELSE
            \STATE \textsc{Array}[$0$].value + $v$ $\rightarrow$ \textsc{Array}[$0$].value
            \STATE \textsc{vote}$^-$ $\leftarrow$ \textsc{Array}[$0$].value
            \STATE $m$ $\leftarrow$ find minimum entry in \textsc{Array}[$1 \sim \text{end}$] 
            \STATE \textsc{vote}$^+$ $\leftarrow$ \textsc{Array}[$m$].value
            \IF{\underline{Case 3}: \textsc{vote}$^-$ / \textsc{vote}$^+$ $>$ \textsc{Threshold}}
            \STATE load tuple $(k_{ex},v_{ex})$ from position $m$
            \STATE clear entry at position $m$
            \STATE add tuple $(k,v)$ to position $m$ (like Case 2)
            \STATE \textcolor{gray}{\# identify and record $k_{ex}$ with \textsc{BloomFilter}}
            \STATE insert tuple $(k_{ex},v_{ex})$ into \textsc{Count-Min}
            \ELSE
            \STATE \textcolor{gray}{\# identify and record $k$ with \textsc{BloomFilter}}
            \STATE insert tuple $(k,v)$ into \textsc{Count-Min}
            \ENDIF
            \ENDIF
        \end{algorithmic}
        \end{algorithm}
    \end{minipage}%
    \hfill
    \begin{minipage}[t]{0.49\textwidth} 
        \begin{algorithm}[H]
        \caption{\texttt{FLORE} Control-plane Recovery}
        \label{alg:recovery}
        \begin{algorithmic}[1]
            \ENSURE Decode all counters using the GM of \texttt{FLORE}
            \STATE $z \leftarrow$ sample priors from \textsc{Gaussian} distribution
            \STATE $b \leftarrow$ flatten counters in \textsc{Count-Min}
            \STATE \textsc{SketchVec} $\leftarrow$ \textsc{GenerativeModel}$(b,z)$
            \STATE \textsc{FilterVec} $\leftarrow$ \texttt{Query}$(\text{\textsc{AugmentedFilter}}, \text{keys})$
            \STATE \textsc{DecodeVec} $\leftarrow$ \textsc{SketchVec} + \textsc{FilterVec}
            \STATE return \textsc{DecodeVec} and keys
        \end{algorithmic}
        \end{algorithm}

        \begin{center}
            \includegraphics[width=0.87\linewidth]{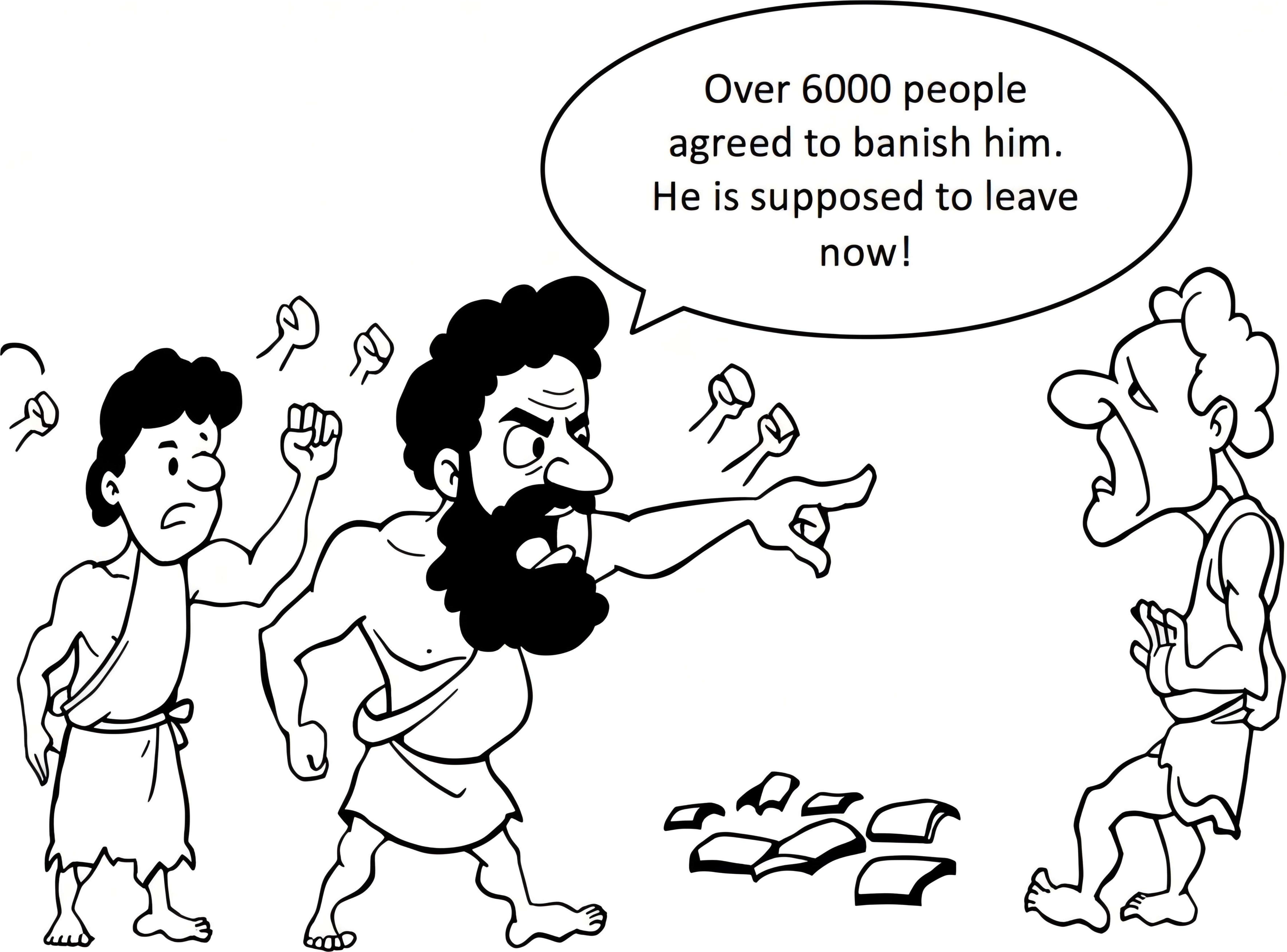}
            \captionof{figure}{Ostracism: the exclusion from the collective; figure is polished and excerpted from page~18 of \href{https://yangtonghome.github.io/uploads/\%E5\%8C\%97\%E4\%BA\%AC\%E5\%A4\%A7\%E5\%AD\%A6-\%E6\%9D\%A8\%E4\%BB\%9D-\%E9\%AB\%98\%E6\%80\%A7\%E8\%83\%BD\%E7\%BD\%91\%E7\%BB\%9C\%E6\%B5\%8B\%E9\%87\%8F.pptx}{a publicly available slide}.}
        \end{center}
    \end{minipage}
\end{figure}

\subsubsection{Stream Filtering Mechanism}
\label{subsec:filter}

In the light of previous studies~\cite{roy2016augmented,huang2021toward}, we also separate frequent elements as explicit key-value pairs while compressing infrequent elements in a sketch (see \S\ref{subsec:lim}). On one hand, it is well acknowledged that such separation can be realized with limited overheads, and improve summarization efficiency; on the other hand, its operations do not affect our analytical results thus far. To achieve this, in the data plane of \texttt{FLORE}, we employ a \textit{oracle-free} stream filtering technique based on a \textit{separate-and-guard-hot} strategy used by Elastic Sketch~\cite{yang2018elastic}. It tracks hot keys and records their information with dedicated entries. We first attempt to insert keys into the dedicated part and then evict infrequent keys to the shared part, i.e., a Count-Min. As illustrated in Figure~\ref{fig:structure}(a), the augmented filter is composed of some arrays, and each array is comprised of a number of entries. In spirit to Ostracism (Greek: ostrakismos, where any citizen could be voted to be evicted from Athens for ten years), we aim to use two types of entries to store hottest items with their frequencies: the negative vote lives in the first entry, while the dedicated counters live in the other entries that consists of two fields—fingerprint (key) and counter (frequency). For each incoming tuple $(k, v)$, we compute a hash function to map it to a specific array. Within the selected array, we identify the entry with the minimum counter value and treat it as the positive vote. The update procedure then follows one of three cases:
\begin{itemize}
\item \textbf{Case 1.} A dedicated entry with a matching fingerprint is found; its counter is incremented by $v$.
\item \textbf{Case 2.} An empty dedicated entry is found and initialized with $(k, v)$.
\item \textbf{Case 3.} Otherwise, an opposing vote is cast to determine whether eviction should occur. Concretely, we increment the negative vote by $v$ and compute the decision variable
\(
\lambda = \frac{\text{negative vote}}{\text{positive vote}}.
\)
If $\lambda$ exceeds a predefined threshold (e.g., $8$), the tuple $(k, v)$ replaces the item in the minimum entry, and the evicted item is forwarded to the Count-Min. Otherwise, $(k, v)$ is directly sent to the Count-Min.
\end{itemize}
The pseudo-code of the data-plane summary procedure is provided in Algorithm~\ref{alg:update}. To query the frequency of a key, we perform a single hash operation to locate the corresponding array. If a stored fingerprint matches the queried key, the associated counter value is returned; otherwise, the query returns zero. 

\subsection{Other Details}
\label{subsec:other_details}

Below, we summarize several auxiliary yet important details that complement the main methodology, covering dataset construction, normalization, and fine-tuning.

\paragraph{Dataset Construction} 
First of all, \texttt{FLORE} eliminates the need for real samples ($f$) during training, facilitating online dataset construction, which requires only sketch summaries (${b}$). Then, counters are periodically sampled at fixed intervals (defined by either the number of updates or elapsed time), which we term sampling epochs. To make the model capture the most recent data stream characteristics, we maintain a sliding window of these sampled states by following~\cite{yuan2025learning}. This window operates on a first-in-first-out (FIFO) style, discarding stale states once its capacity is reached. In our experimental setup, we configure the sampling interval to 5,000 updates and the sliding window length to 200$\sim$500.

\paragraph{Data Normalization} 
As discussed in the above dataset construction process, we do not have fixed global statistics such as the overall minimum/maximum values or mean/std, due to the evolving nature of the data stream. Therefore, instead of classical min-max normalization or mean normalization, we employ instance normalization~\cite{du2021putting,yuan2025learning} to mitigate the influence of the sliding window. It is noteworthy that, in \texttt{FLORE}'s CM with $k$ rows and $m/k$ counters per row, the arrival of a streaming item triggers an update to exactly one counter per row, indexed by hash functions. Then, by leveraging the non-negativity of CM, we can define a scaling factor, $\text{scale}$, as the minimum of the row-wise maximum values: $\text{scale} := \min_{1 \le j \le k} \left( \max_{1 \le i \le m/k} {b} \left[ \frac{m(j-1)}{k} + i \right] \right)$. This important factor can be computed in $\mathcal{O}(m)$ linear time (often very fast), and every element in $f$ is smaller than this constant. The normalized counters are then input into the neural network as ${b}' = {b} / \text{scale}$. At inference time, the model $f_{G}$ predicts the GT vector in the normalized space, which is subsequently mapped back to the original space via: $\hat{f} = \text{scale} \cdot f_{G}({b}')$. 

\paragraph{Online Tuning \texttt{FLORE}}
Since \texttt{FLORE}’s generative model is trained dynamically on sampled data, its operational model can be periodically updated by the cINN fine-tuned in the background on more recent and up-to-date data, to gracefully adapt to \emph{planned} changes in the database system (resetting sketch parameters, planned addition/removal of keys) and to natural drifts in data stream distributions. 
Our evaluation results indicate that \texttt{FLORE} produces high-quality recoveries even epochs after being deployed, and even if the distribution changes during this time (e.g., due to time or/and attacks). This provides ample time for training substitute invertible functions, and one can tentatively handle previously unencountered queries using a simple algorithm such as Count-Min, since \texttt{FLORE} does not disrupt the original data-plane sketching operation.

\subsection{End-to-End View}
\label{subsec:structure}

Putting everything together, we now present an end-to-end view of \texttt{FLORE}. As described in \S\ref{subsec:sys_view} and Fig.~\ref{fig:system}, we have a data plane component to process streams with sketching and a control plane component to estimate the statistics from the sketch. We further split the control plane into a recovery stage and a tuning stage. Below, we describe the three stages.

\textbf{Data-plane stage (Algorithm~\ref{alg:update}).} Once tuple batches arrive, \texttt{FLORE} first attempts to insert the item into the Augmented Filter. If it hits, the other operations are skipped. Otherwise, \texttt{FLORE} inserts it into the CM and uses the Bloom Filter to determine whether the key is unseen, thereby deciding whether it should be forwarded to the control plane.

\textbf{Control-plane recovery stage (Algorithm~\ref{alg:recovery}).} When a query is issued, or recovery is needed for downstream analysis (e.g., anomaly detection), \texttt{FLORE} requests the data plane to transmit its data structures to the control plane. The control plane then decodes the heavy part and light part independently: the former is decoded losslessly (with unrecorded ones treated as zero), while the latter, together with a sampled prior $z$, is fed into the invertible solver for reconstruction. The final estimate is obtained by summing the two results.

\textbf{Control-plane tuning stage (Algorithm~\ref{alg:training}).} \texttt{FLORE} only delivers the sampled CM counters to this stage periodically, which minimizes manual cost. Using either a fixed number of updates or elapsed time as a data-plane period, the control plane maintains an SW dataset to fine-tune the GM. \texttt{FLORE} pre-processes these new counters via the EM, and the tuning offers quick convergence time as in our technical report~\cite{tech_report}.

\section{Evaluation}
\label{main_sec:exps}

In this section, we seek to understand: (1) With only sketch counters, can \texttt{FLORE} satisfy the requirements of downstream tasks? (2) Can \texttt{FLORE} achieve real-time stream processing? (3) How does those changing stream conditions affect \texttt{FLORE}'s performance? (4) How does each design in \texttt{FLORE} contribute to its performance? Note that our approach has two variants: \texttt{FLORE} and \texttt{FLORE} (perfect). The former applies when $\mathcal{I}$ is unknown and 
therefore requires allocating an additional portion of memory to deploy Bloom Filters. This leads to the final question: (5) do they exhibit a similar performance in the conditional generation?

\textbf{Summary of Main Results:} 

\textbullet\ For both real-world and synthetic data sketching, \texttt{FLORE}, using only counters, achieves up to 98.8\% and 93.1\% reduction in terms of accuracy and distribution/entropy estimation, respectively. And for heavy-hitters detection, \texttt{FLORE} can almost always achieve an F1 score close to 1.0 with only 64KB-128KB memory usage.

\textbullet\ \texttt{FLORE} relaxes the requirements for data collection and computational cost for ML models. On a local device, \texttt{FLORE} can reach a throughput of 70+ summary Mops, while on a centralized server, it can recover all the information within 10s milliseconds, proving its real-time synthesis capacity.

\textbullet\ We observe that, regardless of temporal variation, spatial variation, or natural drift, the performance degradation of \texttt{FLORE} can almost always be kept within 10\%, demonstrating strong robustness to environmental changes.

\textbullet\ We confirm the effectiveness of the data-plane data structures and control-plane generative models in \texttt{FLORE}.

\textbullet\ Across all experiments, \texttt{FLORE} (perfect) and \texttt{FLORE} only exhibit some differences under tight memory budgets.

\subsection{Experimental Setup}

We conducted all experiments on a server equipped with an 8-core 3.80GHz CPU, 32GB DDR5 and a 16GB NVIDIA GPU. The server runs on Ubuntu 20.04, and all all experiments are repeated 5-10 times. By default, we implemented all programs in Python 3.10 and use PyTorch 2.9 for tensor operations and ML models. Below, we first detail the datasets used in our experiments, followed by the introduction of baseline methods and hyperparameters. Finally, we formulate the evaluation metrics. 

\subsubsection{Datasets}

\begin{table}[t]
    \centering
    \caption{Statistics of Real-world and Synthetic Data Streams Used in the Evaluation}
    \label{tab:stream_statistics}
    \renewcommand{\arraystretch}{1.15}
    \setlength{\tabcolsep}{6pt}
    \resizebox{1.\textwidth}{!}{\begin{tabular}{c|ccc||c|ccc}
        \toprule
        \textbf{Real-World} 
        & \textbf{\# of Keys} 
        & \textbf{\# of Items} 
        & \textbf{Skewness}
        & \textbf{Synthetic} 
        & \textbf{\# of Keys} 
        & \textbf{\# of Items} 
        & \textbf{Skewness} \\
        \midrule
        \textit{{CAIDA}}  & \num{157269} & \num{2000000} & 291.56 & \textit{Zipf-icml} & $\sim$ \num{30000} & \num{1000000} & 24.27 \\
        \midrule
        \textit{Webdocs}    & \num{125623} & \num{2000000} & 25.42 & \textit{Zipf} & $\sim$ \num{28000} & \num{1000000} & 125.93 \\
        \midrule
        \textit{{MAWI}}       & \num{41471} & \num{2000000} & 92.67 & \textit{Pareto} & $\sim$ \num{28000} & \num{1000000} & 10.24 \\
        \midrule
        \textit{Kosarak}    & \num{30495} & \num{2000000} & 89.56 & \textit{Exponential} & $\sim$ \num{26000} & \num{1000000} & 3.81 \\
        \midrule
        \textit{{Retail}}       & \num{16470} & \num{908576} & 71.27 & \textit{Log-Normal} & $\sim$ \num{30000} & \num{1000000} & 5.64 \\
        \bottomrule
    \end{tabular}}
\end{table}

We evaluate the sketching quality of our approach on \textit{ten} datasets, in which the data stream key-space size ranges from $10^3$s to $10^4$s, comprising \textit{five} real-world data streams and \textit{five} synthetic data streams. The real-world datasets include: CAIDA-2018 backbone IP packet traces (\textit{CAIDA}~\cite{caida}), Web document access traces (\textit{Webdocs}~\cite{lucchese2004webdocs}), traffic traces from the WIDE project (\textit{MAWI}~\cite{cho2000traffic}), click-stream data (\textit{Kosarak}~\cite{kosarak}), and retail market basket data (\textit{Retail}~\cite{brijs1999using}). For the CAIDA and MAWI datasets, we construct 13-byte flow keys consisting of five fields: source and destination IP addresses, source and destination ports, and transport protocol. For the other datasets, item IDs are converted into 4-byte-long keys. In addition, we synthesize data streams following representative heavy-tailed distributions (e.g., Zipf, Pareto, exponential, and log-normal) and report their results in our technical report~\cite{tech_report}. TABLE~\ref{tab:stream_statistics} shows the important details of these datasets.

In addition, we generate five synthetic datasets following representative heavy-tailed distributions: \textit{Zipf-icml}, \textit{Zipf}, \textit{Pareto}, \textit{Exponential}, and \textit{Log-Normal}. Formally, let \(N\) denote the number of distinct items, and let \(f_i\) be the frequency associated with item \(i\), we can write 
\begin{itemize}
    \item \textbf{Zipfian Distribution}. Frequencies follow Zipf's law
    \(
        f_i \propto \frac{1}{i^{\alpha}}, \ i = 1, 2, \ldots, N,
    \)
    where \(\alpha > 0\) controls the degree of skew. Note that Zipf-icml is generated by following the implementation of~\cite{du2021putting}, i.e., $f_i \propto \left\lfloor \frac{N}{i^{\alpha}} \right\rfloor$, while Zipf strictly follows the aforementioned formulation with a zipfian skewness of $1.4$. 
    \item \textbf{Pareto Distribution}. Frequencies are sampled from a Pareto distribution with shape parameter \(\alpha > 0\), whose CDF is given by
    \(
        P(F > x) = \left(\frac{x_m}{x}\right)^{\alpha}, \ x \ge x_m,
    \)
    where \(x_m\) is the minimum value. We then draw samples \(f_i \sim \mathrm{Pareto}(\alpha)\), scale them by a constant factor, and discretize the results to obtain integer frequencies.
    \item \textbf{Log-Normal Distribution}. Frequencies follow a log-normal distribution, i.e.,
    \(
        \log f_i \sim \mathcal{N}(\mu, \sigma^2),
    \)
    where \(\mu\) and \(\sigma\) denote the mean and standard deviation of the underlying normal distribution. The sampled frequencies are also scaled and rounded to integers.
    \item \textbf{Exponential Distribution}. Frequencies are generated from an exponential distribution with rate parameter \(\lambda > 0\), whose density function is
    \(
        p(x) = \lambda e^{-\lambda x}, \ x \ge 0.
    \) As with the other distributions, samples are scaled and discretized to form integer-valued frequencies.
\end{itemize}
For most synthetic workloads, the resulting frequency vectors are randomly permuted to eliminate any ordering bias, and each distinct item is encoded as a fixed-length 4-byte key that appears in the stream according to its assigned frequency. Table~\ref{tab:stream_statistics} shows the important details of the datasets (i.e., the number of keys, the number of items, and the skewness of streaming data distribution). 

\subsubsection{Baselines}
\label{subsec:bl}

We compare our proposal to {\textit{ten}} SOTA streaming solutions, covering four mainstream sketch techniques as follows. \textcircled{\scriptsize 1} Two classical sketches: \textit{CM}~\cite{cormode2005improved} and \textit{CS}~\cite{charikar2002finding}. \textcircled{\scriptsize 2} Three classical optimizations: \textit{AG}~\cite{roy2016augmented}, \textit{ES}~\cite{yang2018elastic} and Conservative Update Sketch (\textit{CU}~\cite{estan2003new}). \textcircled{\scriptsize 3} Two \texttt{CS}-based sketches: \textit{PR}~\cite{sheng2021pr} and \textit{NZE}~\cite{huang2021toward}. \textcircled{\scriptsize 4} Three learning-augmented sketches: Learned CM (\textit{LCM})~\cite{hsu2019learning}, Learned CS (\textit{LCS})~\cite{hsu2019learning}, and Learned Sketch without predictions (\textit{LS})~\cite{aamand2023improved}. We list these algorithms as follows:
\begin{enumerate}
    \item \textbf{Count-Min Sketch} (\textit{CM})~\cite{cormode2005improved}: CM maintains $k$ pairwise independent hash functions $h_\ell : U \rightarrow [E]$ and a counter array $C \in \mathbb{R}^{k \times E}$, where $E = m / k$. Upon processing the stream, each counter is updated as $C[\ell, h_\ell(j)] = \sum f_j$. For any item $i \in U$, its frequency estimate is $\tilde{f}_i = \min_{\ell \le k} \{C[\ell, h_\ell(i)]\}$. In our experiments, we fix $k=4$; all subsequent sketches follow this configuration unless otherwise specified.
    \item \textbf{Count Sketch} (\textit{CS})~\cite{charikar2002finding}: Similar to CM, CS employs $k$ hash functions $h_\ell : U \rightarrow [E]$ and a counter array $C \in \mathbb{R}^{k \times E}$. Additionally, it utilizes $k$ independent sign hash functions $g_\ell : U \rightarrow \{-1, 1\}$. The counters are updated as $C[\ell, h_\ell(j)] = \sum f_j \cdot g_\ell(j)$, and the estimate for item $i$ is computed as $\tilde{f}_i = \operatorname{median}_{\ell \le k} \{ g_\ell(i) \cdot C[\ell, h_\ell(i)] \}$.
    \item \textbf{Conservative Update Sketch} (\textit{CU})~\cite{estan2003new}: As a variant of CM, CU employs a conservative update policy to maintain an upper bound on the ground truth ($\tilde{f}_i \ge f_i$). For each incoming tuple $(i, v)$, a counter $C[\ell, h_\ell(i)]$ is updated only if it achieves the minimum value among all $k$ rows. Formally, $C[\ell, h_\ell(i)] \leftarrow C[\ell, h_\ell(i)] + v$ for all $\ell \in \arg\min_{j \le k} \{C[j, h_j(i)]\}$.
    \item \textbf{Augmented Sketch} (\textit{AG})~\cite{roy2016augmented}: AG utilizes the skew of the underlying stream data to improve the frequency estimation accuracy for the most frequent items by filtering them out earlier. It exchanges data items between the filter and the sketch (i.e., a CM) such that hot items are kept in the filter. We configure the filter with $\sim 100$ key-value pairs; this relatively small size improves accuracy while preventing lookup overhead from degrading data stream processing throughput.
    \item \textbf{Elastic Sketch} (\textit{ES})~\cite{yang2018elastic}: Using the same data structure with us (except for the Bloom Filter), ES tracks hot keys and records their information with dedicated entries. It first attempts to insert keys to the dedicated part and then evict infrequent keys to the shared part, i.e., a Count-Min. In the experiments, we allocate 50\% of the memory budget to the shared part and the remaining 50\% to the dedicated part.
    \item \textbf{PR-sketch} (\textit{PR})~\cite{sheng2021pr}: PR also improves estimation accuracy by constructing linear equations that relate counter values to per-key aggregations. Solving the resulting least-squares problem requires an iterative solver. While the original implementation adopts a conjugate gradient method~\cite{EigenLeastSquaresConjugateGradient}, we instead use LSMR~\cite{fong2011lsmr} from \texttt{scipy.sparse.linalg}. We set the number of Bloom Filter hash functions to 7. Memory is partitioned equally between the sketch and the Bloom Filter, subject to an upper bound on the latter: to ensure a false positive rate below 1\%, we limit its size to $9.6N$ bits (e.g., 120 KB for $N=100K$) as suggested by~\cite{cormode2011synopses}.
    \item \textbf{NZE-sketch} (\textit{NZE)~\cite{huang2021toward}}: NZE leverages compressive sensing by constructing a fractional sensing matrix and employing a hash table for elephant element filtering. Specifically, it maintains three data structures: key-value pairs for heavy hitters, a fractional-valued sketch for small elements, and a Bloom filter. We implement the SeqSketch variant due to its superior memory efficiency, using OMP~\cite{pati1993orthogonal} for signal recovery. The configuration mirrors PR, but incorporates a hash table; we increase the hash table allocation to 75\% in high-memory scenarios to reduce computational overhead during reconstruction.
    \item \textbf{Learned CM Sketch} (\textit{LCM})~\cite{hsu2019learning}: The idea of LCM is similar to AG, which is to store hot and cold elements separately. However, while AG relies on a data-exchange algorithm for classification, LCM employs a learned model. We implement an RNN-based classifier\footnote{The original implementation in the prior work~\cite{hsu2019learning} was based on TensorFlow for IP address identification. In our implementation, we re-implement this component using PyTorch. For data streams where such correlations are absent, we directly consider an idealized setting by removing the Top-K hot keys from the stream.} with a 10—20KB memory footprint for label prediction. The remaining memory budget is split equally (1:1) between the heavy-hitter (unique buckets) and light-hitter (CM) components. 
    \item \textbf{Learned Count Sketch} (\textit{LCS})~\cite{hsu2019learning}: LCS follows the exact configuration of LCM but substitutes the underlying Count Min Sketch with a Count Sketch.
    \item \textbf{Learned Sketch without Predictions} (\textit{LS})~\cite{aamand2023improved}: LS extends LCS by introducing two key query-time optimizations. (i) Efficiency: To mitigate inference overhead, LS performs model predictions probabilistically rather than for every item; we set this sampling probability to 15\% in our experiments. (ii) Accuracy: To improve precision for small elements, LS truncates low-value estimates to zero. The threshold is defined as $\frac{C \cdot kN}{m}$, where $C$ is a constant (set to 10 in our experiments). Then, any query result obtained by its light part below this threshold is set to zero.
\end{enumerate}

\subsubsection{Hyperparameters} 

\textbf{Data-plane structural parameters.} The memory allocation of \texttt{FLORE} follows the NZE-sketch instance. Specifically, the total memory budget is initially partitioned: 50\% is allocated to the Bloom Filter, while the remaining 50\% is split equally between the Count–Min Sketch and the augmented filter. Both the Bloom Filter and the CM sketch are subject to capacity constraints: a maximum false-positive rate threshold for the Bloom Filter and a hard memory cap of 256KB for the CM sketch. If either limit is reached, any residual memory budget is reallocated to the augmented filter. For CUDA (GPU) implementation of \texttt{FLORE}, we vary the batch size from 100K to 1M and recompile the program for each configuration to evaluate its impact on GPU insertion throughput.

\textbf{Control-plane model hyperparameters.} We evaluate segment lengths from the set $\{512, 1024, 2048\}$. The hidden dimension of \texttt{FLORE}’s $f$-autoencoder is selected from $\{64, 128, 256, 512\}$. For the hidden representation of $b$, we set its dimension to 2/3 of that of $f$, reserving the remaining 1/3 for concatenation with the sampled Gaussian noise vector $z$. The number of cINN layers is chosen from $\{2, 3, 4, 5\}$. The final layer employs a \textsc{Tanh} activation function, consistent with the encoder’s projection of inputs into the $[-1, 1]$ range.

\textbf{Training parameters.} The loss trade-off coefficients $\alpha_1$ through $\alpha_4$ are set to (or selected from) $0.5$, $0.001$, $0.01$, and $\{0.01, 0.05, 0.1, 0.25\}$, respectively. Notably, $\alpha_2$ is assigned a near-zero value, as we observed that the INN demonstrates near-perfect invertibility even without strong regularization. The initial training phase consists of 20$\sim$30 epochs, while subsequent fine-tuning to incorporate new elements needs 3$\sim$5 epochs. By default, the MMD metric is used to measure the loss term $\mathcal{L}_{\mathbf{ort}}$, and the Adam optimizer~\cite{kingma2014adam} is employed throughout all experiments.

\subsubsection{Metrics}
\label{subsec:metric}

\begin{figure}
    \subfigure[Comparison of frequency estimation absolute (1st row) and relative (2nd row) errors.\label{fig:acc_compare}]{
        \centering
        \includegraphics[width=1.\linewidth]{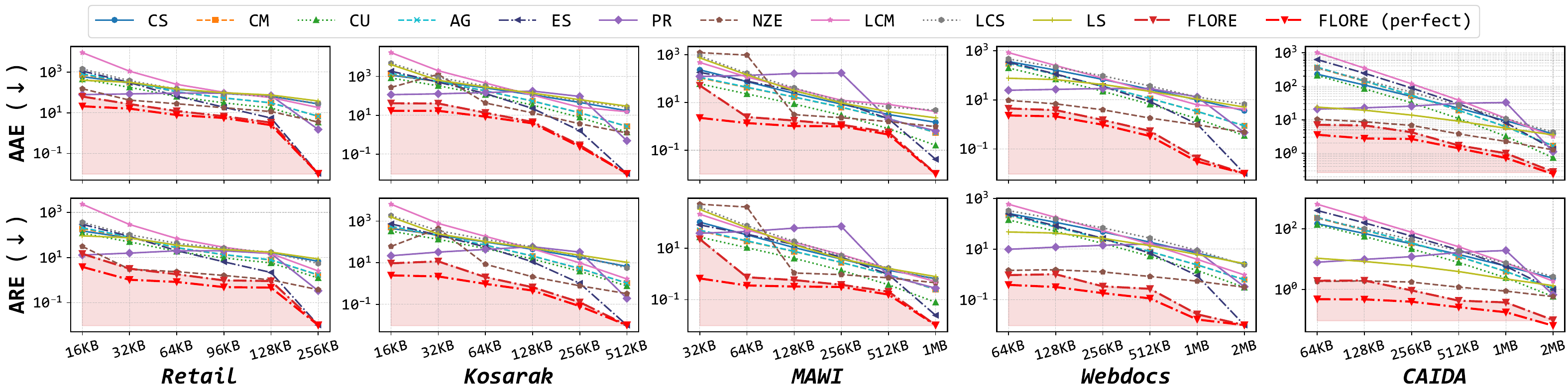}
    }\\
    \subfigure[Comparison of heavy hitters (1st row) and Top-k (2nd row) detection accuracy.\label{fig:hhd_compare}]{
        \centering
        \includegraphics[width=1.\linewidth]{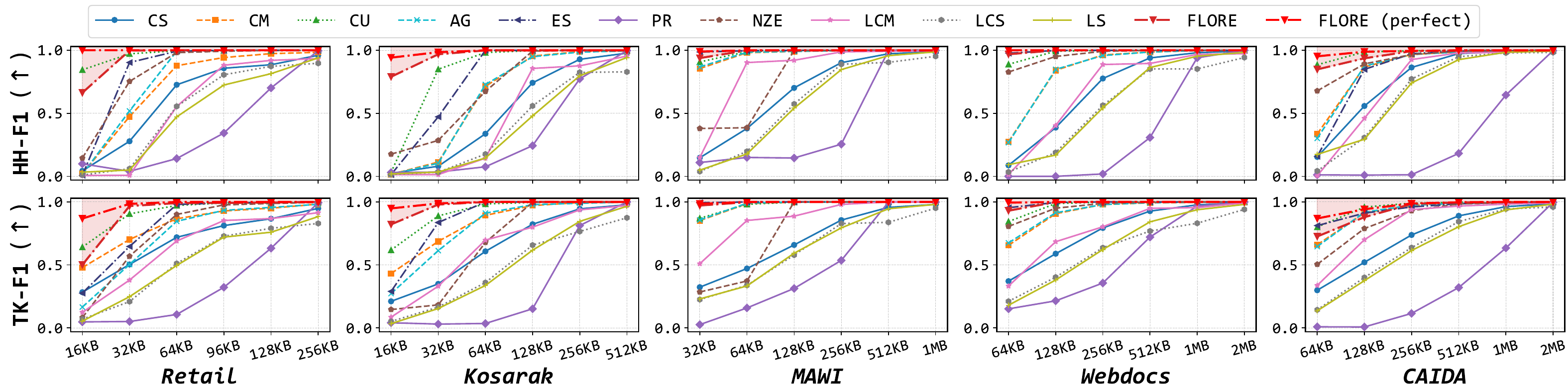}
    }
    \\
    \subfigure[Comparison of distributional (1st row) and entropy (2nd row) distances to the GT.\label{fig:dist_compare}]{
        \centering
        \includegraphics[width=1.\linewidth]{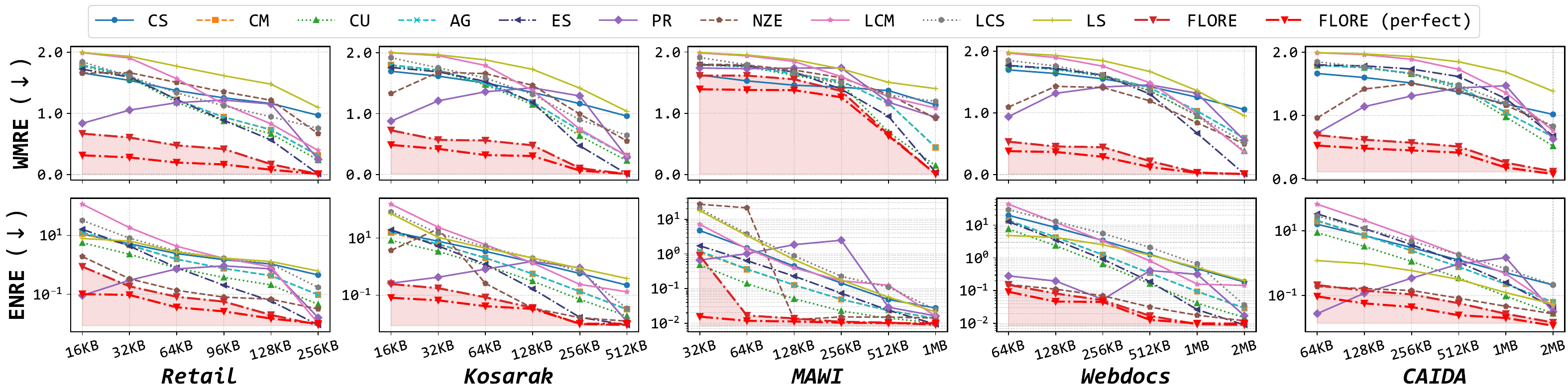}
    }
    \caption{Sketching performance on five real-world datasets.}
\end{figure}

We evaluate the performance of sketching algorithms using the following metrics. Again, all experiments are repeated multiple times, and the average results are reported.
\begin{itemize}
    \item \textbf{Average Absolute Error (AAE)}: $\frac{1}{N} \sum^N_i\left|f_i - \hat f_i\right|$, which measures the average magnitude of estimation errors.
    \item \textbf{Average Relative Error (ARE)}: $\frac{1}{N} \sum^N_i \frac{\left|f_i - \hat f_i\right|}{f_i}$, used to evaluate the error normalized by the true element frequency. The ARE metric is highly sensitive to mouse elements, because their small true sizes reside in the denominator, even marginal overestimations yield extreme relative errors that inflate the reported mean.
    \item \textbf{Weighted Mean Relative Error (WMRE)}: $\frac{\sum_{j=1}^z |F_j - \hat{F}_j|}{\sum_{j=1}^z \frac{F_j + \hat{F}_j}{2}}$, where $z$ is the maximum element frequency, and $F_j$ and $\hat{F}_j$ are the true and estimated counts of elements with frequency $j$. This metric assesses how well the sketch preserves the overall frequency distribution.
    \item \textbf{Entropy Absolute Error (AE)}: $|H - \hat{H}|$, defined as the absolute difference between the true entropy $H$ and the estimated entropy $\hat{H}$, where $H = -\sum_{i=1}^N \frac{f_i}{S} \log \frac{f_i}{S}$ and $S = \sum f_i$.
    \item \textbf{F1 Score}: $\frac{2 \cdot PR \cdot RR}{PR + RR}$, the harmonic mean of Precision ($PR$) and Recall ($RR$). Here, Recall is the ratio of correctly identified heavy hitters to the total number of true heavy hitters, and Precision is the ratio of correctly identified heavy hitters to the total number of reported instances. 
    \item \textbf{Throughput}: Million operations (insertions/queries) per second (Mops).
\end{itemize}
\textbf{\textit{Note}}: when the key space is unknown, traditional approaches typically rely on optimization-based methods to approximate the cardinality or distribution of data streams from counter values, such as the basic MRAC algorithm~\cite{kumar2004data}. In this paper, we assume that the key space is known or can be efficiently tracked. Accordingly, for a fair comparison, all traditional algorithms are allowed to directly leverage the known keys when detecting heavy hitters or estimating distribution (i.e., computing WMRE) and entropy.

\subsection{Experiments on Real Data}


\textbf{Per-element frequency recovery.}
\label{subsec:pfe}
We first evaluate the frequency estimation for both frequent and infrequent elements, with results shown in Fig.~\ref{fig:acc_compare} (logarithmic scale). Generally, accuracy improves as memory increases. The baseline CM and its optimizations (AG, CU) provide stable improvements, whereas CS typically lags behind CM except under some small memory usage. When there is enough memory to identify and store all frequent elements, ES can perform very well but below certain memory threshold, its performance is average. Learning-augmented sketches (LCM, LCS) reveal a clear trade-off: they struggle in resource-constrained settings due to the high memory overhead of auxiliary structures (buckets/classifiers). While LS mitigates this via the truncation trick, its benefits diminish as memory expands. In contrast, \texttt{CS}-based methods (PR, NZE) achieve remarkable precision under high compression, which beat CM by two orders of magnitude, though their performance is non-monotonic due to the under-determined recovery process. We also observe that \texttt{FLORE} outperforms almost all methods by several orders of magnitude. At a 256KB budget, it reduces AAE and ARE significantly, consistently surpassing CM by over $400\times$, and usually dominates PR and NZE even without GT data for its learning. Meanwhile, \texttt{FLORE} (perfect) marginally outperforms the standard \texttt{FLORE}, as expected.


\textbf{Heavy element identification.}
\label{subsec:hhd}
Next, we report the results for heavy-hitter detection and finding global top-k hot elements, where the predefined frequency threshold is 1,000, and the top $\log N$ frequent elements are selected, respectively. Fig.~\ref{fig:hhd_compare} compares the F1 Score of all solutions in heavy hitter (HH-F1) and top-k (TK-F1) detection under different memory sizes. 
As in the figure, HH-F1 and TK-F1 exhibit a totally same characteristics. 
Since only frequent elements are counted here, we observe the following phenomenon. First, CM and CU perform exceptionally well due to their inherent non-negativity. While AG employs an augmented filter to accurately track a small subset of hot items (e.g., $\leq100$ positions), its overall improvement remains limited. ES is free from this problem; however, it often fails under a low memory budget. Second, CS and its derivatives (LCS, LS) lack this non-negative property, resulting in inferior accuracy. Third, the performance of PR and NZE is strikingly reversed, with PR consistently yielding the lowest F1 Score. This is because these \texttt{CS}-based methods prioritize global estimation fidelity over the precise recovery of individual heavy hitters. In contrast, \texttt{FLORE} and \texttt{FLORE} (perfect) consistently outperform all these sketch techniques, maintaining a much higher F1 Score across the board. Specifically, on nearly all datasets, \texttt{FLORE} achieves near-perfect detection accuracy (90\%$\sim$100\%) with 64 KB of memory. This advantage becomes even more pronounced under fewer counters, where \texttt{FLORE} delivers up to a $4.5\times$ improvement in average.


\textbf{Estimation of distribution \& entropy.}
\label{subsec:ede}
We further evaluate the capacity of each sketch to model the underlying stream distribution, representing the fidelity of the recovery. As the 1st row of Fig.~\ref{fig:dist_compare} shows, an impressive performance disparity can be observed: despite the absence of GT data during training, \texttt{FLORE} achieves a substantially lower discrepancy to the real distribution than all baselines. Taking the CAIDA dataset as an example, \texttt{FLORE} with a mere 64KB memory budget attains a level of fidelity that competing sketches cannot achieve even with 2MB, corresponding to a 32$\times$ reduction in memory footprint for comparable accuracy. This advantage arises from \texttt{FLORE}’s generative sketching paradigm, which explicitly captures and exploits the statistical structure of the data stream. Furthermore, our results reveal that per-element estimation accuracy is not necessarily indicative of global distribution fidelity. For instance, while \texttt{CS}-based methods may perform well in point estimation, they struggle to faithfully recover the overall distribution.
A similar trend is evident in the 2nd row of Fig.~\ref{fig:dist_compare}, which illustrates the advantages of our approach in entropy estimation. Interestingly, PR’s estimates are highly accurate under extreme memory scarcity but degrade sharply before stabilizing once the system transitions from under-determined to well-determined. In contrast, \texttt{FLORE} maintains much stable performance, with the ENRE consistently remaining near or below 0.1 on average. 
This represents an improvement of up to two orders of magnitude compared to baselines. 
These findings demonstrate that \texttt{FLORE}’s recovery is both precise and statistically consistent with GT, marking a clear departure from prior works.

\subsection{Experiments on Synthetic Data}

\begin{figure}
    \setlength{\belowcaptionskip}{-0.3cm}
    \subfigure[Comparison of frequency estimation absolute (1st row) and relative (2nd row) errors.\label{fig:acc_compare_syn}]{
        \centering
        \includegraphics[width=1.\linewidth]{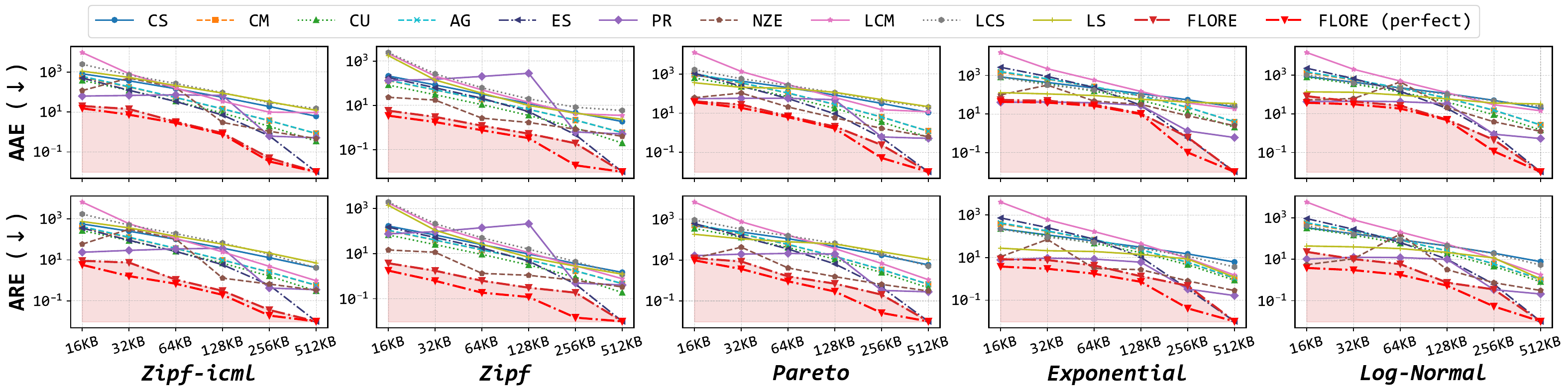}
    }\\
    \subfigure[Comparison of heavy hitters (1st row) and Top-k (2nd row) detection accuracy.\label{fig:hhd_compare_syn}]{
        \centering
        \includegraphics[width=1.\linewidth]{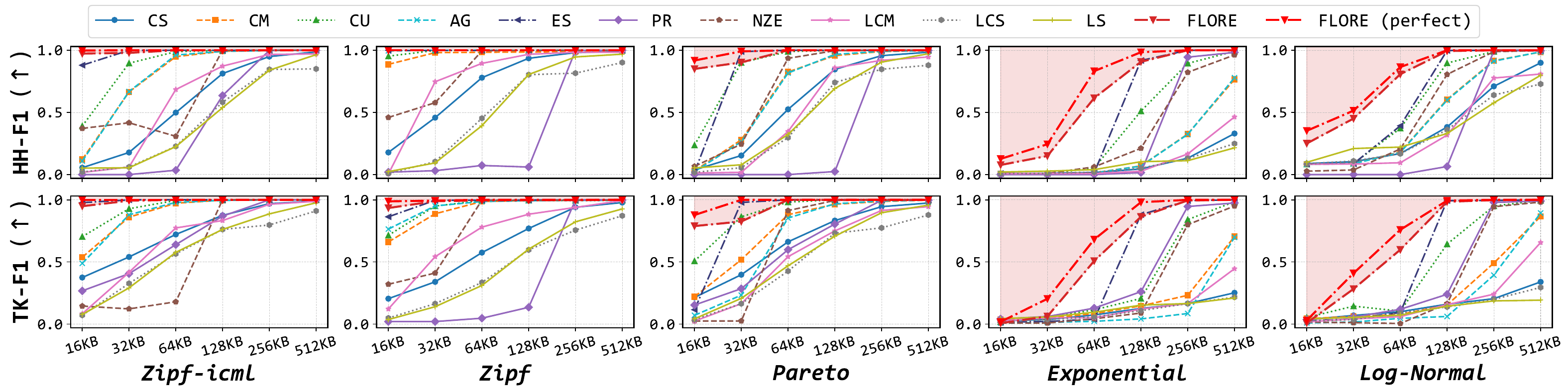}
    }
    \\
    \subfigure[Comparison of distributional (1st row) and entropy (2nd row) distances to the GT.\label{fig:dist_compare_syn}]{
        \centering
        \includegraphics[width=1.\linewidth]{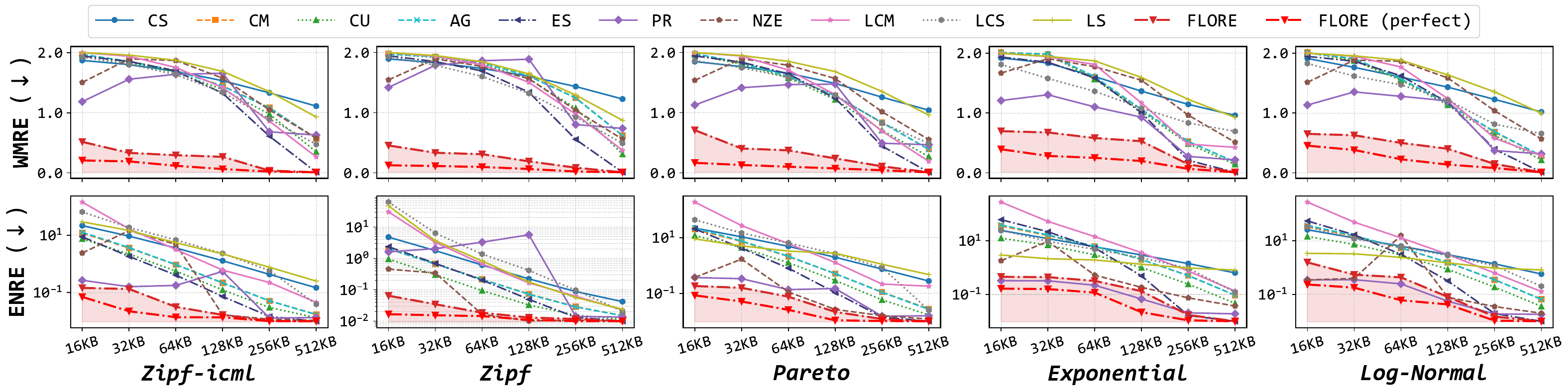}
    }
    \caption{Sketching performance on five synthetic datasets.}
    \label{fig:syn}
\end{figure}

To evaluate the adaptability of our method to data streams with different distributions, Figure~\ref{fig:syn} details its performance across five heavy-tailed distributions, including Zipf, Pareto, Exponential, and Log-Normal. The overall trend closely mirrors that observed on real-world datasets in the main text (with \texttt{FLORE} consistently achieving optimal or near-optimal results), yet the degree of distribution skew impacts performance.

First, regarding per-element accuracy (Figure~\ref{fig:acc_compare_syn}): Under tight memory budgets, all non-\texttt{CS}-based approaches perform poorly. For instance, CM and CS typically incur errors on the order of $10^3$ (both relative and absolute) at 16KB. Learning-augmented sketches fare even worse, as they require extra space to store a small subset of hot items and their classifiers. Interestingly, LS performs well under certain skew levels by directly truncating counters (e.g., highly skewed CAIDA, and moderately skewed Exponential and Log-Normal here), though this advantage quickly diminishes as memory increases. Among \texttt{CS}-based schemes, PR exhibits severe estimation instability, only converging to reasonable error levels when ample memory is available (as the system becomes determined and admits a unique solution). NZE steadily improves accuracy with more memory, making it the most accurate sketch algorithm (trading computational complexity for performance), but it is still outperformed by \texttt{FLORE}. Specifically, \texttt{FLORE} reduces error by up to $700\times$ compared to other methods and achieves near-zero error at memory sizes of 256KB and above.

Second, regarding heavy-hitter accuracy (Figure~\ref{fig:hhd_compare_syn}): The behavior under the Zipf distributions closely mirrors that of the real-world distributions presented in the main text, where \texttt{FLORE} always achieves near 100\% accuracy, requiring only 16KB. However, ``surprises'' emerge under the Exponential and Log-Normal distributions. Almost all methods fail to detect heavy hitters effectively under a low skewness, as the keys become less distinguishable. Notably, CM performs worse than usual, even falling behind PR (which typically ranks last), with its peak F1 score barely reaching around 80\%. Other approaches, including \texttt{CS}, optimizations, and learning-augmented sketches, suffer similarly. LS is particularly noteworthy: its counter-truncation strategy yields unusually low accuracy on these latter two distributions. In contrast, \texttt{FLORE} not only consistently sits at the top of the curves but also demonstrates remarkable recovery: although its performance at 16KB is subpar, it rebounds almost linearly thereafter, pulling ahead of all competitors by a wide margin. Specifically, \texttt{FLORE} achieves roughly $2\times$ to $4\times$ higher F1 scores.

Third, regarding distribution alignment (Figure~\ref{fig:dist_compare_syn}): Thanks to its generative nature, though it is fully GT-free, \texttt{FLORE} continues to dominate decisively on these metrics. Specifically, \texttt{FLORE}’s WMRE rarely exceeds 0.5, while most baseline methods require as much as 512KB of memory to achieve comparable accuracy. Even more strikingly, under tight memory budgets, \texttt{FLORE}’s error is less than 1/3 that of baselines. Moreover, even high-accuracy \texttt{CS}-based methods like PR and NZE suffer from inherent limitations in distribution modeling. Notably, NZE’s enforced sparsity during reconstruction inevitably compromises fidelity to the true underlying distribution. Regarding entropy estimation, although low-skew distributions narrow the performance gap between \texttt{FLORE} and other methods, \texttt{FLORE} still significantly outperforms all non-\texttt{CS}-based approaches. For example, on the Zipf dataset, \texttt{FLORE}’s absolute entropy estimation error is, on average, $105.14\times$ and $712.98\times$ smaller than that of CM and \texttt{CS}, respectively. Another observation is that PR’s performance improves noticeably in the bottom two rows, sometimes even surpassing \texttt{FLORE}, but its gains remain very limited when ample memory is available. We attribute this stagnation to the solver’s inability to find a sparser solution, causing it to plateau.


In summary, the results across diverse synthetic distributions show that \texttt{FLORE} generalizes well to different patterns and consistently achieves superior accuracy under tight memory budgets.

\subsection{Results on \texttt{FLORE}'s Efficiency}

Owing to the data-control-plane separation design, \texttt{FLORE} is expected to achieve both high runtime efficiency and fast training convergence. We validate this property by evaluating its stream processing throughput, recovery speed, and convergence behavior.

\begin{figure}
    \subfigure[CPU summary processing speed in data streams]{
        \centering
        \includegraphics[width=0.39\linewidth]{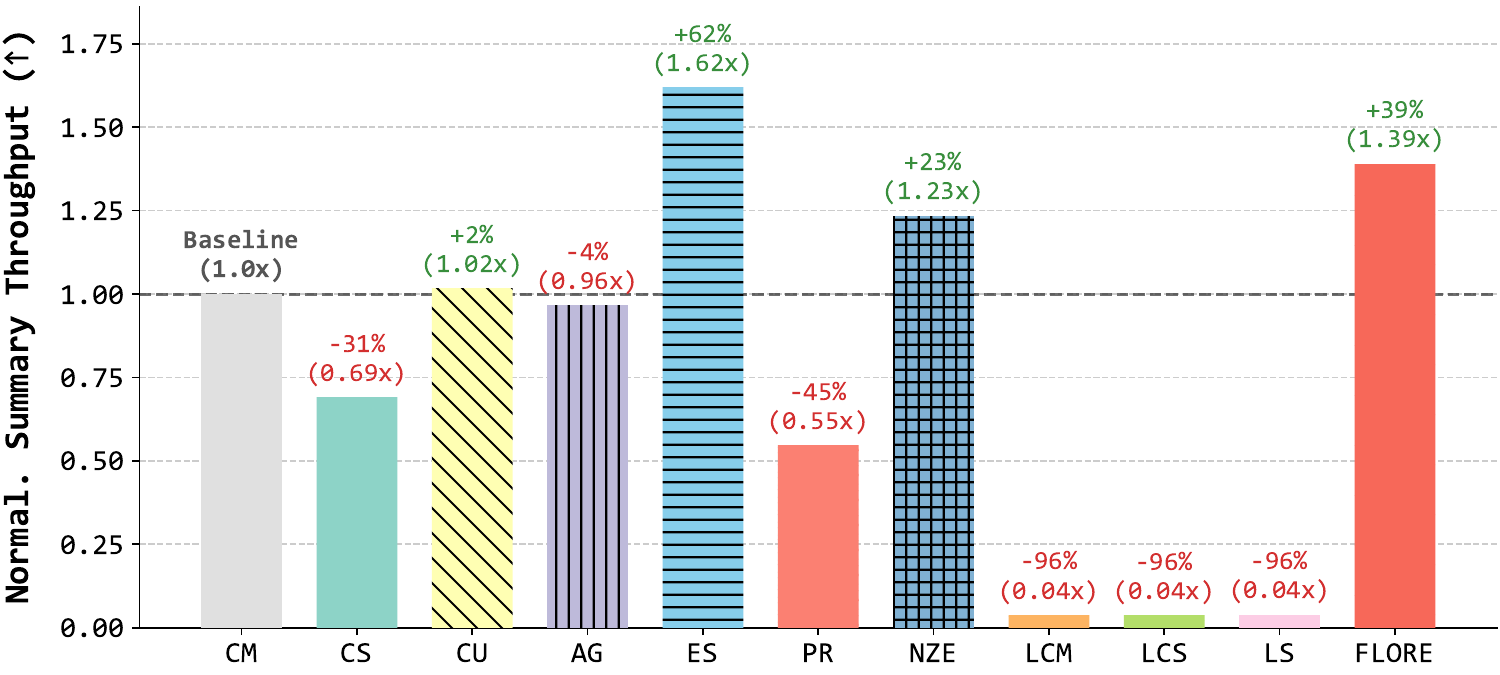}
    }
    \hfill
    \subfigure[GPU recovery (for all keys) processing speed]{
        \centering
        \includegraphics[width=0.39\linewidth]{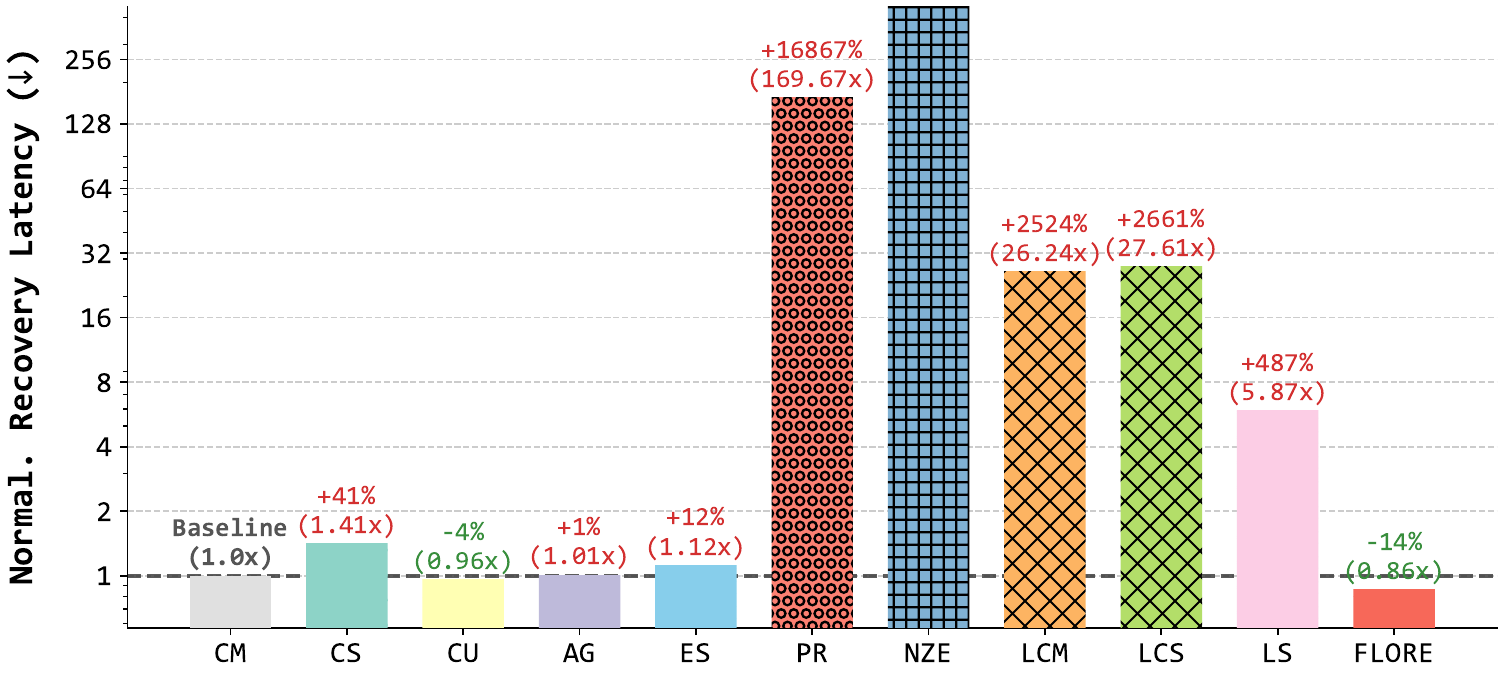}
    }
    \hfill
    \subfigure[Speed v.s. Platform\label{subfig:platform}]{
        \centering
        \includegraphics[width=0.175\linewidth]{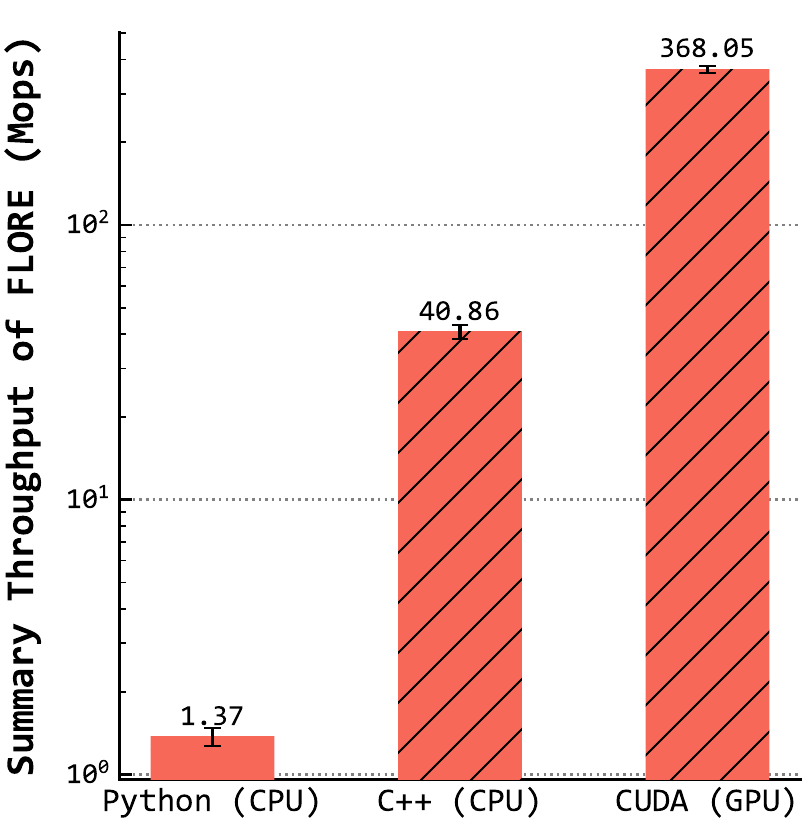}
    }
    \caption{Comparison of stream processing speed on CAIDA with 1MB memory (\textcolor{teal}{Green} means better than CM).}
    \label{fig:speed_compare}
\end{figure}


\subsubsection{Stream Processing Speed}

Figure~\ref{fig:speed_compare} presents the scaled stream processing speed relative to the CM baseline. We also implement and evaluate \texttt{FLORE} on various platforms, such as GPU using CUDA Toolkit~\cite{cuda}. Note that, due to the data-control-plane separation design, sketch summarization is performed on the CPU, while recovery is executed in parallel on the GPU through tensor operations. The results indicate that learning-augmented sketches suffer from low efficiency due to the required classifier forward pass. Although LS limits the usage of Oracle (to 15\% of queries here), its performance is still suboptimal. \texttt{CS} decoding algorithms need over $10^2$ times longer recovery time to get a reliable result, which is unacceptably high. Conversely, \texttt{FLORE} incurs no overhead on the data-plane stream processing as it is used solely for recovery, and filtered stream with $\mathcal{O}(1)$ complexity, which outperforms the baseline in both summary and recovery phases, with gains becoming even more pronounced across different platforms (see Figure~\ref{subfig:platform} in the right). Specifically, the speed of \texttt{FLORE} on Cuda and C++ is 29.8, 268.6 times higher than in Python. We can also see that the summary throughput of \texttt{FLORE} is less than ES because the Bloom Filter will add extra operations when item replacement occurs.

\subsubsection{Convergence Time}

\begin{figure}
    \centering
    \subfigure[Curve of $\mathcal{L}_{\mathbf{total}}$]{
        \centering
        \includegraphics[width=0.31\linewidth]{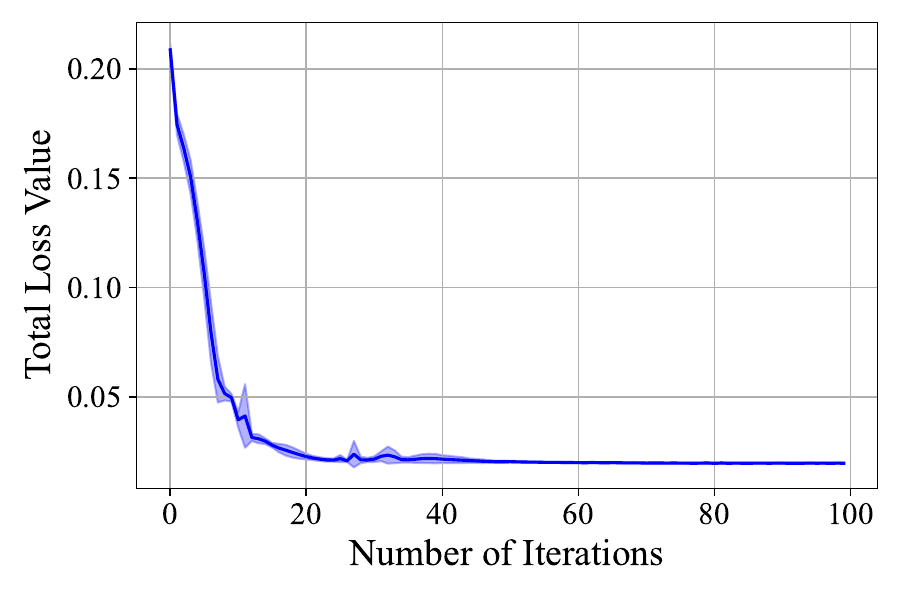}
    }
    \hfill
    \subfigure[Curve of $\mathcal{L}_{\mathbf{ort}}$]{
        \centering
        \includegraphics[width=0.31\linewidth]{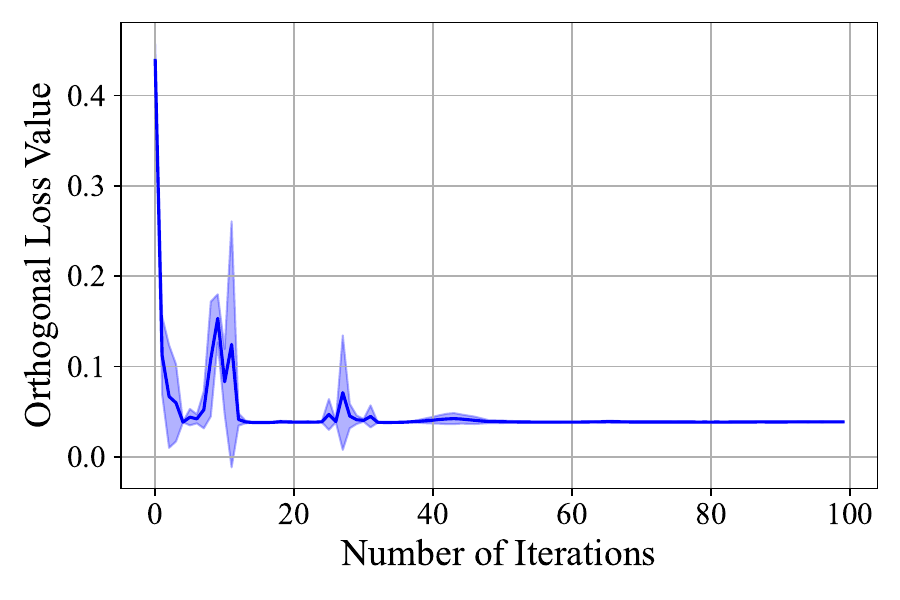}
    }
    \hfill
    \subfigure[Curve of $\mathcal{L}_{\mathbf{inv}}$]{
        \centering
        \includegraphics[width=0.31\linewidth]{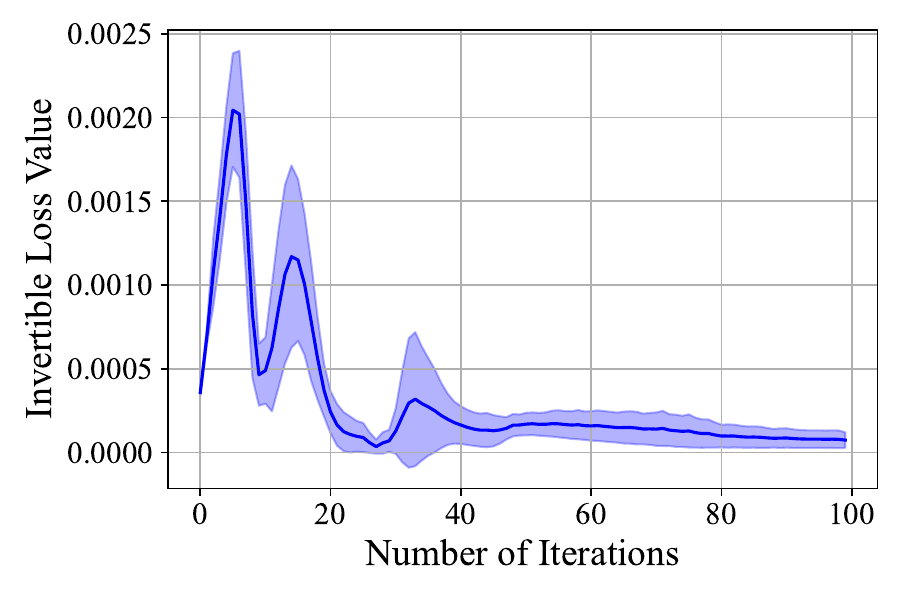}
    }\\
    \subfigure[Curve of $\mathcal{L}_{\mathbf{con}}$]{
        \centering
        \includegraphics[width=0.31\linewidth]{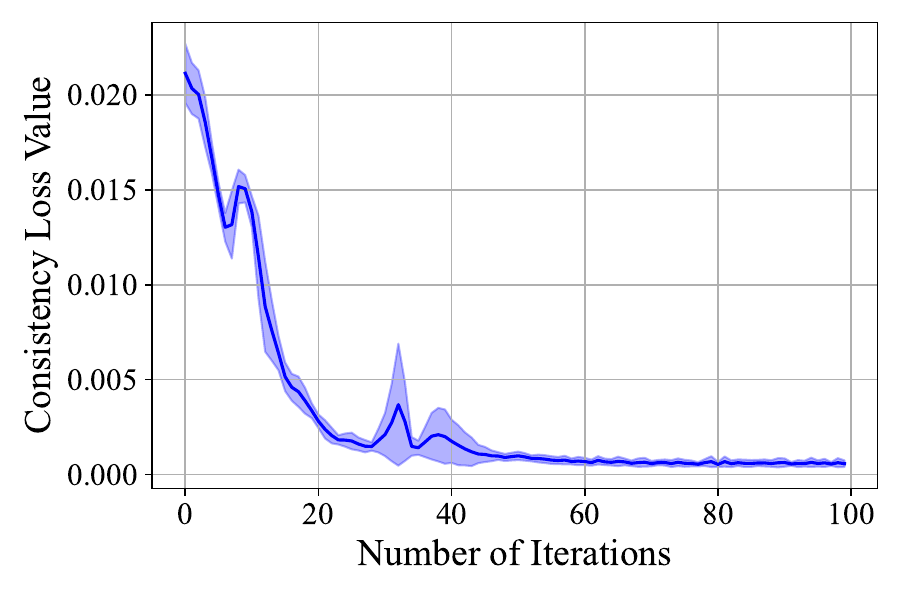}
    }
    \hfill
    \subfigure[Curve of $\mathcal{L}_{\mathbf{rec}}$]{
        \centering
        \includegraphics[width=0.31\linewidth]{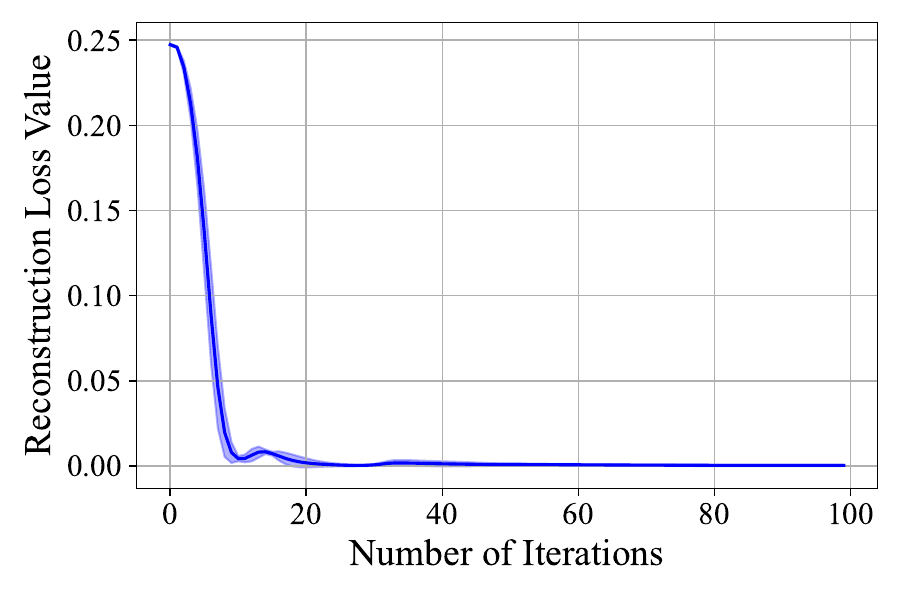}
    }
    \hfill
    \subfigure[Curve of $\mathcal{L}_{\mathbf{sp}}$]{
        \centering
        \includegraphics[width=0.31\linewidth]{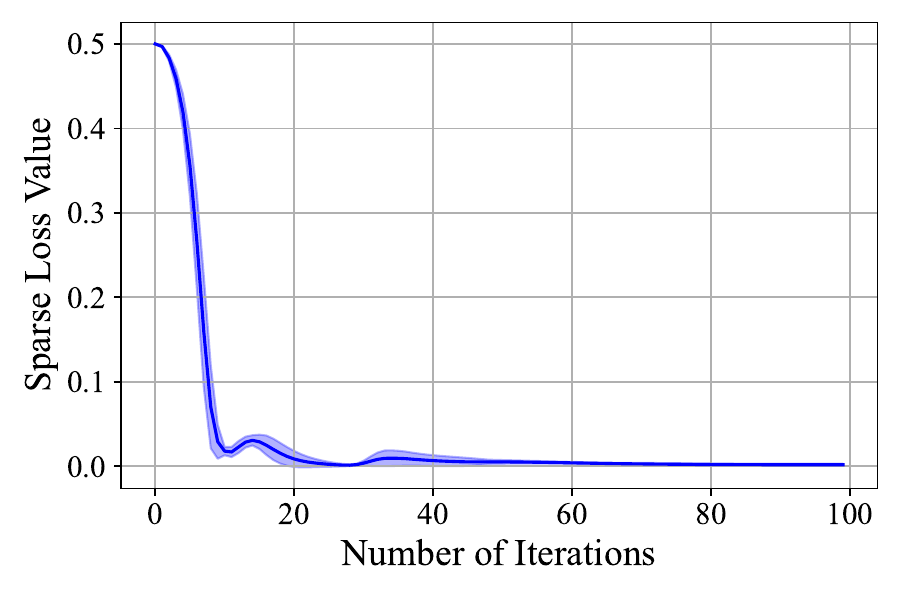}
    }
    \caption{Training loss curves of \texttt{FLORE} on the CAIDA dataset.}
    \label{fig:converge}
\end{figure}

We indeed discuss at the start of \S\ref{main_sec:gensys} the benefits that a “perfect generator” could bring to the linear sketching problem. However, for Theorem~\ref{thm:convergence} in our framework design, such an idealized setting (losses $= 0$) is not required; in practice, approximate satisfaction ($\approx 0$) is sufficient to approximate the true posterior $p(f \mid b)$. Compared to prior works (e.g., see~\cite{hsu2019learning,aamand2023improved} that needs a perfect oracle), our assumption is much more weaker. Empirically, we conduct experiments to train \texttt{FLORE} on the CAIDA dataset using a 128KB memory. The curve of all losses can be found in Figure~\ref{fig:converge}, and we observe that all losses are reduced to nearly-zero values during training. In particular, $\mathcal{L}_{\mathbf{ort}}$ (i.e., the MMD between $[z, b]$ and $G(z, b)$) is driven to consistently small and stable levels (i.e., a good $z$–$b$ independence), while $\mathcal{L}_{\mathbf{inv}}$ is always $< 10^{-2}$, implying that the underlying model is strictly invertible. More importantly, each iteration typically completes within tens of milliseconds, meaning the model converges within tens of seconds and reaches a reasonable loss range in under ten seconds. Therefore, our method fully supports online iterative updates and is well-suited for real-world database or network monitoring system deployment.

\subsection{Robustness to Stream Changes}
\label{subsec:rb}

In what follows, we evaluate the robustness of \texttt{FLORE} under temporal, natural, and spatial variations in data streams. In these experiments, we randomly sample 10\% to 20\% of the key set to simulate changes in stream behavior in each sampling epoch, following the prior work~\cite{yuan2025learning}. All experiments are conducted on the CAIDA dataset. 

\begin{table}[htbp]
    \centering
    \resizebox{\textwidth}{!}{
        \begin{tabular}{@{}ccc@{}} 
            
            \begin{tabular}[t]{llcccc}
                \toprule
                \multirow{2}{*}{Memory} & & \multicolumn{4}{c}{Fluctuation Factor} \\
                & & 0.2 & 0.5 & 1.0 & 2.0 \\
                \midrule
                \multirow{2}{*}{128KB}  & average   & 0.0\% & -1.0\%  & -1.8\%  & 2.7\%  \\
                                         & 90th Pct. & 0.4\%  & -0.2\%  & -2.1\%  & 5.9\% \\ \addlinespace
                \multirow{2}{*}{256KB} & average   & -0.0\%  & -1.1\% & -1.8\% & 1.9\% \\
                                         & 90th Pct. & 0.2\%  & -0.3\%  & -1.4\% & 2.8\%  \\ \addlinespace
                \multirow{2}{*}{512KB} & average   & 0.2\%  & 0.0\% & -1.0\% & 1.0\% \\
                                         & 90th Pct. & 0.3\%  & -0.0\%  & 1.1\% & 1.3\%  \\ \addlinespace
                \multirow{2}{*}{1MB}  & average   & 0.0\%  & -1.9\%  & 0.2\%  & -0.3\% \\
                                         & 90th Pct. & 1.1\%  & -0.8\%  & 0.2\%  & -0.5\%  \\
                \bottomrule
            \end{tabular} & 
            
            \begin{tabular}[t]{llccc}
                \toprule
                \multirow{2}{*}{Memory} & & \multicolumn{3}{c}{Training stream time windows} \\
                & & 0\%-25\% & 25\%-50\% & 50\%-75\% \\
                \midrule
                \multirow{2}{*}{128KB}  & average   & 5.4\% & 3.6\% & 4.4\% \\
                                         & 90th Pct. & 9.6\% & 7.8\% & 8.7\% \\ \addlinespace
                \multirow{2}{*}{256KB} & average   & 2.6\%  & 3.2\%  & 2.1\%  \\
                                         & 90th Pct. & 4.0\%  & 5.7\%  & 4.6\%  \\ \addlinespace
                \multirow{2}{*}{512KB} & average   & 0.2\%  & 1.8\%  & 0.5\%  \\
                                         & 90th Pct. & 0.9\%  & 0.2\%  & 1.0\%  \\ \addlinespace
                \multirow{2}{*}{1MB}  & average   & 0.5\%  & 1.6\%  & -1.5\%  \\
                                         & 90th Pct. & 0.0\%  & 1.0\%  & 1.6\%  \\
                \bottomrule
            \end{tabular} & 
            
            \begin{tabular}[t]{llcccc}
                \toprule
                \multirow{2}{*}{Memory} & & \multicolumn{4}{c}{Proportion Factor} \\
                & & 0.9 & 0.5 & 0.25 & 0.1 \\
                \midrule
                \multirow{2}{*}{128KB}  & average   & 0.3\% & 0.7\%  & 1.8\%  & 2.3\% \\
                                         & 90th Pct. & 1.7\%  & 1.0\%  & 3.1\%  & 6.5\% \\ \addlinespace
                \multirow{2}{*}{256KB} & average   & 0.2\%  & 0.6\% & 1.5\% & 2.6\% \\
                                         & 90th Pct. & 0.4\%  & 0.7\%  & -0.4\% & 2.4\%  \\ \addlinespace
                \multirow{2}{*}{512KB} & average   & 0.0\%  & -0.6\% & 1.7\% & 1.3\% \\
                                         & 90th Pct. & 0.2\%  & 0.4\%  & 1.4\% & 0.6\%  \\ \addlinespace
                \multirow{2}{*}{1MB}  & average   & 0.5\%  & 0.7\%  & -1.0\% & -1.6\% \\
                                         & 90th Pct. & 0.0\%  & -0.6\%  & -0.4\% & 1.4\% \\
                \bottomrule
            \end{tabular} \\
            
            \begin{minipage}[t]{0.5\textwidth}
                \caption{ARE Performance decline with increased CAIDA network traffic fluctuation. Negative values indicate no degradation.}
                \label{tab:tc}
            \end{minipage} & 
            \begin{minipage}[t]{0.5\textwidth}
                \caption{ARE Performance decline with natural drift in CAIDA network traffic. Negative values indicate no degradation.}
                \label{tab:tn}
            \end{minipage} & 
            \begin{minipage}[t]{0.5\textwidth}
                \caption{ARE Performance decline with spatial shift in CAIDA network traffic. Negative values indicate no degradation.}
                \label{tab:ts}
            \end{minipage}
        \end{tabular}
    }
\end{table}

\textbf{Temporal changes.} To increase temporal variability, we proceed as follows. For each selected key, we generate a temporal change in one epoch by sampling from a Gaussian distribution $\mathcal{N}(\mu, \sigma^2)$ and scaling it by a factor, where $\mu = 0$ and $\sigma$ is the standard deviation across different sampling epochs of the key. The scaling factor $\in \{0.2, 0.5, 1.0, 2.0\}$ controls the amplitude of the fluctuation. We evaluate the impact of sudden changes by measuring the ARE degradation of \texttt{FLORE} relative to its performance under stable conditions. The results are summarized in Table~\ref{tab:tc}. When the scaling factor is relatively small, \texttt{FLORE} exhibits no noticeable performance degradation. Even when the scaling factor $=2$, the performance degradation is still within 10\%.

\textbf{Natural changes.} To test the impact of natural shifts, we conduct training separately with 0\%-25\%, 25\%-50\%, and 50\%-75\% of the sampled counters, followed by testing on the last 25\%. Note that evaluation is performed only on the selected keys that appear in all four parts. This protocol enables us to quantify performance degradation relative to the setting in which 75\% of the data is used for training. The results w.r.t. ARE are summarized in Table~\ref{tab:tn}. We observe that \texttt{FLORE} maintains stable performance on the selected keys even long after the completion of training (exceeding two times the total duration of the training data).

\textbf{Spatial changes.} To simulate changes in the spatial distribution of frequencies, we redistribute frequencies across different keys in one epoch. Specifically, we reassign the top 10\% of hot keys to the selected keys such that they account for 90\%, 50\%, 25\%, and 10\% of their original total volume. We refer to the resulting scaling parameter as the proportion factor $\in \{0.9, 0.5, 0.25, 0.1\}$, which controls the magnitude of the spatial change. As presented in Table~\ref{tab:ts}, the performance variations are not significant across all spatial distributions. Compared to the results in Table~\ref{tab:tc}, although there is indeed a slightly greater decline in \texttt{FLORE}’s performance, the impact is still limited, and \texttt{FLORE} consistently maintains good accuracy across different proportion factors.

Overall, \texttt{FLORE} is robust to changing conditions in the real world. We attribute this partly to the GM’s ability to capture the underlying distribution inherent to the sketching problem, while compressive sensing enables the model to effectively detect anomalous changes~\cite{mardani2015estimating}.

\subsection{Interpreting \texttt{FLORE}’s Effectiveness}

Finally, we perform several ablation studies to assess the impact of \texttt{FLORE}’s key features on its overall performance.

\subsubsection{Effect of Generative Model} 
\label{subsec:decoding}

\begin{figure}
    \subfigure[Recovery accuracy comparison using the AAE \label{fig:cs_acc}]{
        \centering
        \includegraphics[width=1.\linewidth]{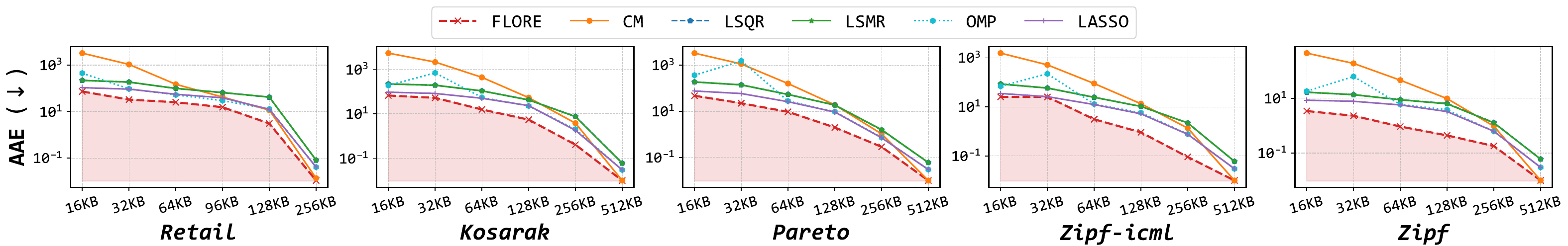}
    }\\
    \subfigure[Recovery realness comparison using the WMRE \label{fig:cs_dis}]{
        \centering
        \includegraphics[width=1.\linewidth]{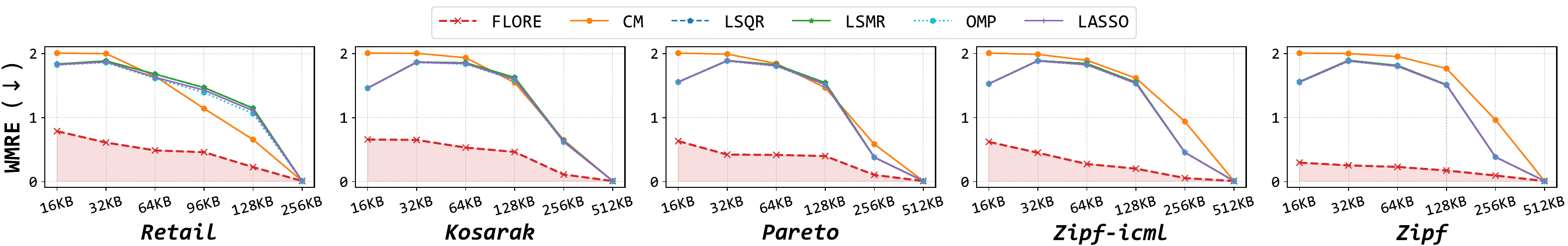}
    }
    \caption{Comparing \texttt{FLORE} with the Count-Min and SOTA compressive sensing algorithms by replacing the FGM.}
    \label{fig:cs}
\end{figure}

\begin{figure}
    \centering
    \includegraphics[width=1.\linewidth]{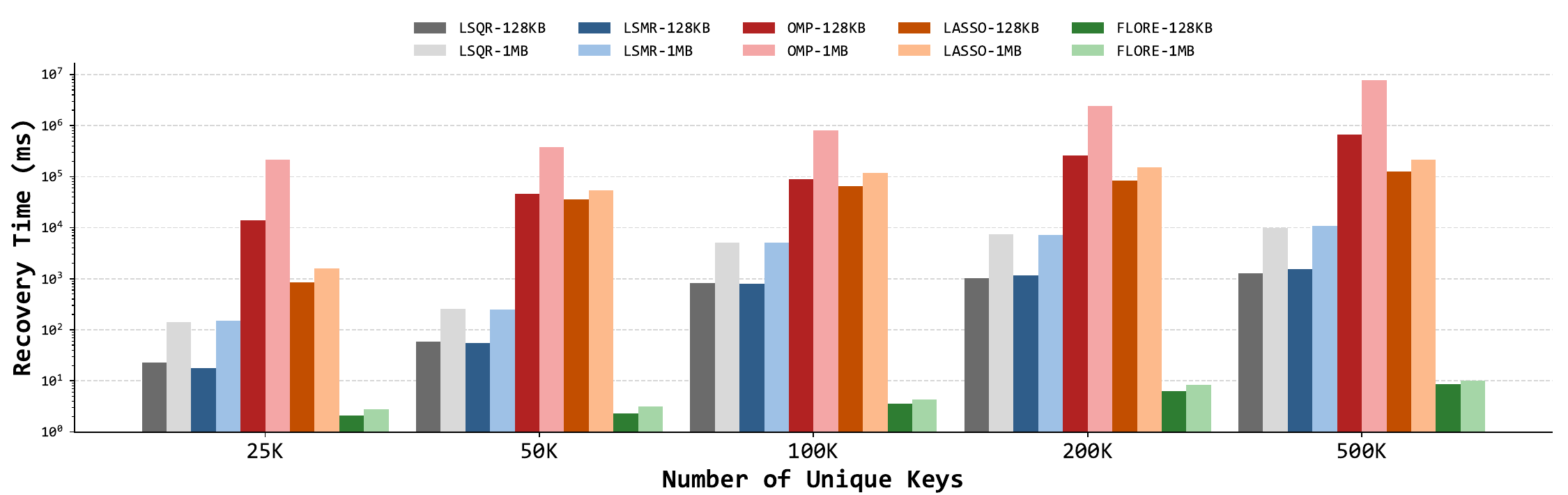}
    \caption{Comparison of computation times for recovery procedure.}
    \label{fig:cs_time}
\end{figure} 

To further validate our generative method, we replaced the INN-based generator with four SOTA \texttt{CS} algorithms: LSQR~\cite{paige1982lsqr}, OMP~\cite{pati1993orthogonal}, LASSO~\cite{tibshirani1996regression}, and LSMR~\cite{fong2011lsmr}. Since \texttt{Flore} is built on top of Count-Min Sketch, a Count-Min query algorithm (i.e., CM) is included here as a traditional benchmark. The main results are shown in Figure~\ref{fig:cs}. We conduct experiments on five relatively small datasets, since these \texttt{CS} algorithms are often very resource-consuming (see results discussed later). Overall, our method consistently outperforms the \texttt{CS} baselines in terms of accuracy. More notably, the improvement in distribution fidelity is substantially larger, achieving on average approximately a 65\% gain over the best baseline.

Also, we compare the computation time for \texttt{FLORE} with other LP-based \texttt{CS} algorithms in Figure~\ref{fig:cs_time}. As the stream size increases, we observe that \texttt{FLORE} demonstrates scalable performance precisely as intended, requiring less than 15 ms of computation time while achieving better recovery. When the number of unique keys is 500K, for instance, \texttt{FLORE} achieves $10^2\sim 10^6\times$ speedups relative to the baselines, accentuating the importance of employing an FGM in the control plane.

\begin{figure}
    \subfigure{
        \centering
        \includegraphics[width=1.\linewidth]{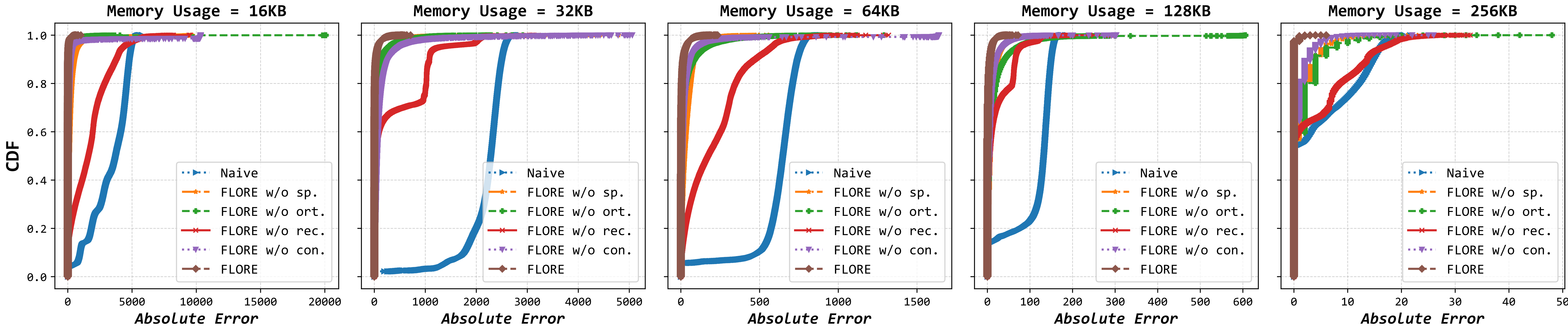}
    }\\
    \subfigure{
        \centering
        \includegraphics[width=1.\linewidth]{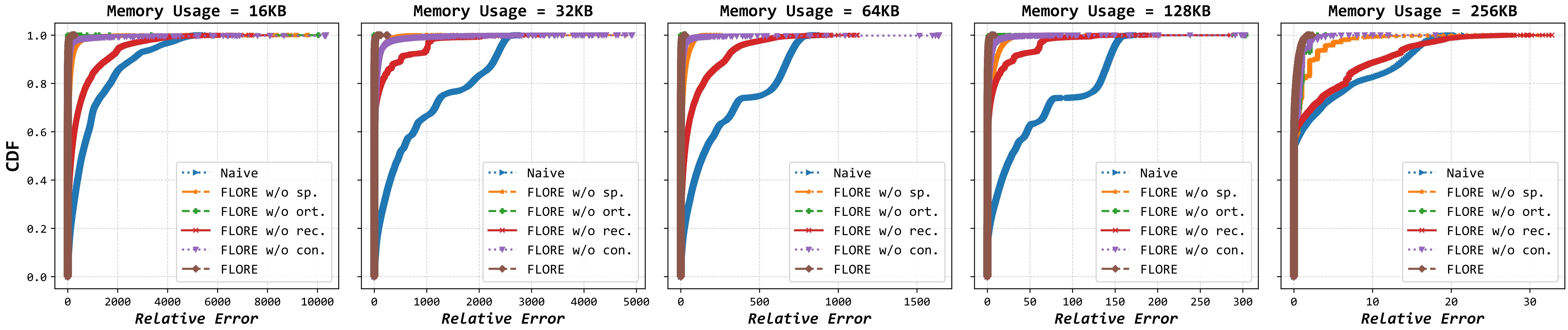}
    }
    \caption{CDFs for the AAE and ARE of \texttt{FLORE} and its variants on the real-world Kosarak dataset.}
    \label{fig:ab_k}
\end{figure}

\begin{figure}
    \subfigure{
        \centering
        \includegraphics[width=1.\linewidth]{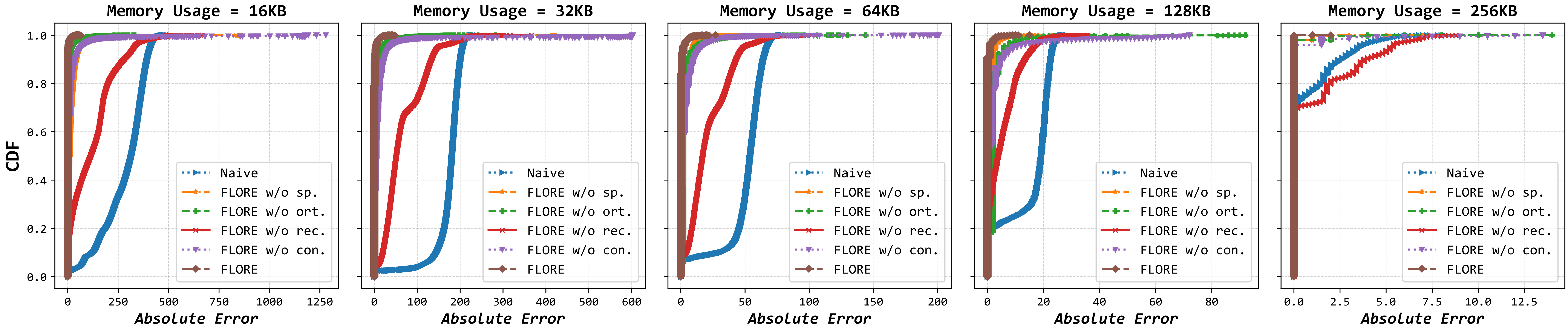}
    }\\
    \subfigure{
        \centering
        \includegraphics[width=1.\linewidth]{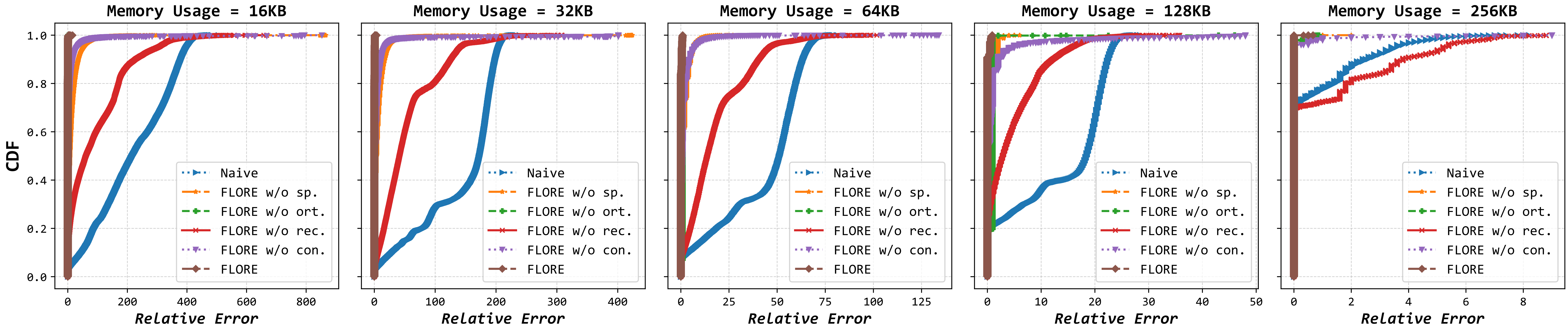}
    }
    \caption{CDFs for the AAE and ARE of \texttt{FLORE} and its variants on the synthetic Zipf dataset.}
    \label{fig:ab_z}
\end{figure}

\begin{figure}
    \centering
    \includegraphics[width=1.\linewidth]{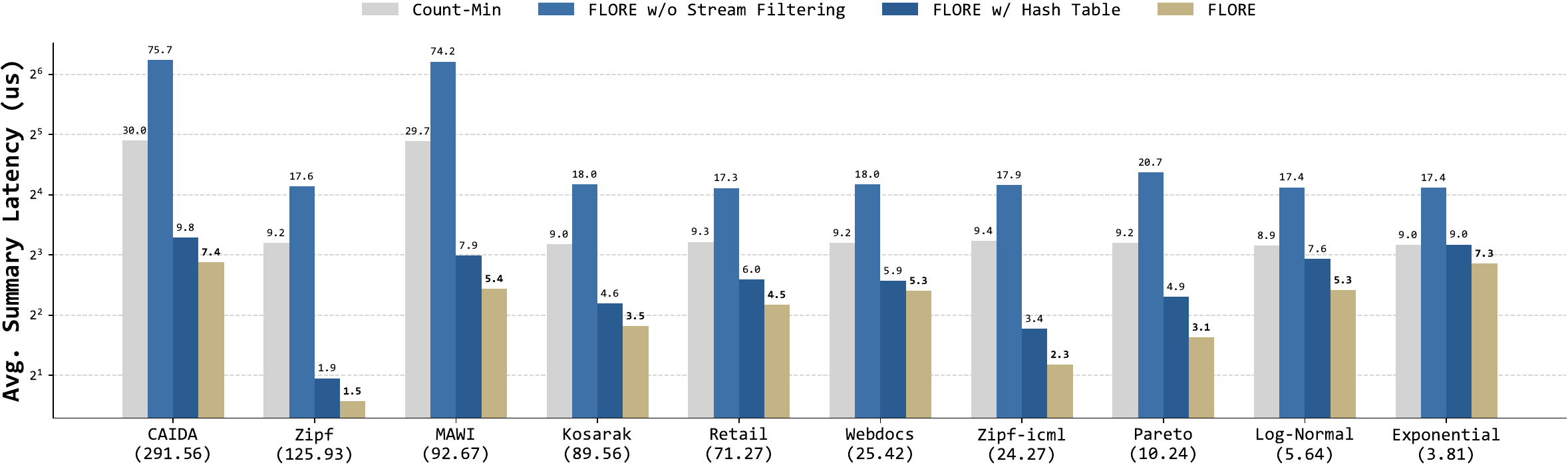}
    \caption{Ablation comparing \texttt{FLORE} and \texttt{FLORE} without the stream filtering mechanism or separation design, where $(x)$ at the bottom indicates that the data stream has skewness $x$.}
    \label{fig:af_time}
\end{figure} 

\subsubsection{Effect of Training Objectives}

We next evaluate the impact of training losses used by \textbf{\texttt{FLORE}}. In accordance with our design choices, we consider the following five \texttt{FLORE} variants for comparison. (i) \textbf{Naive}: This variant removes all generative losses (i.e., $\mathcal{L}_{\mathbf{inv}}$, $\mathcal{L}_{\mathbf{ort}}$, and $\mathcal{L}_{\mathbf{rec}}$) in the FGM and degenerates into a simple one-way mapping function. (ii) \textbf{\texttt{FLORE} w/o sp.}: This variant removes the sparsity constraint (i.e., $\mathcal{L}_{\mathbf{sp}}$) introduced by compressed sensing theory. (iii) \textbf{\texttt{FLORE} w/o ort.}: This variant eliminates the orthogonality constraint (i.e., $\mathcal{L}_{\mathbf{ort}}$) between $z$ and $b$. (iv) \textbf{\texttt{FLORE} w/o rec.}: This variant discards the reconstruction guidance (i.e., $\mathcal{L}_{\mathbf{rec}}$) provided by CM + EM and relies solely on counter-based learning. (v) \textbf{\texttt{FLORE} w/o con.}: This variant removes the explicit linear equation constraints (i.e., $\mathcal{L}_{\mathbf{con}}$), leaving the model to learn without structural supervision. As shown in Figure~\ref{fig:ab_k} and Figure~\ref{fig:ab_z}, we plot the cumulative distribution function (CDF) that shows absolute/relative errors achieved by \texttt{FLORE} and its variants on Kosarak and Zipf, respectively. 

The simple mapping function \textbf{Naive} exhibits the most severe performance degradation, underscoring the critical role of the generative paradigm in recovery. The variant \textbf{\texttt{FLORE} w/o rec.} performs better than \textbf{Naive} but still suffers from a substantial loss in accuracy. This behavior arises because a generative model requires an effective ``information channel'' to anchor its learning to the underlying GT vector, a role that is fulfilled by our EM algorithm. While not as critically important as the two aforementioned loss terms, the other three variants demonstrate their impact through fine-grained adjustments. In particular, these variants tend to produce a small number of outliers with exceptionally large errors. For example, on the Kosarak dataset with a 16KB budget, \textbf{\texttt{FLORE} w/o ort.}, \textbf{\texttt{FLORE} w/o con.}, and \textbf{\texttt{FLORE} w/o sp.} each exhibit data points with absolute errors of 20,000, 10,000, and 8,000, respectively. Interestingly, although the consistency enforced by the linear equations is intuitively expected to play a central role, directly removing it does not cause the sketch measurement system to collapse. This is because the solution obtained by our EM algorithm already approximately satisfies this constraint, thereby implicitly guiding the model to learn the desired relationship. In contrast, \texttt{FLORE} exhibits consistently stable performance across all scenarios: more than 90\% of its restoration errors lie within the lowest 10\% of error magnitudes, and such extreme outliers are almost entirely eliminated. These observations corroborate the effectiveness of each training loss component in \texttt{FLORE}.

\subsubsection{Effect of Stream Filtering}

Because the key-value pairs stored in the augmented filter are guaranteed to be accurate, we further conduct an analysis on the processing speed improvement brought by \texttt{FLORE}'s stream filtering mechanism, as indicated by ``\texttt{FLORE} w/o Stream Filtering''. To enable a finer-grained dissection of effective components, we also introduce two additional baselines: Count-Min and \texttt{FLORE} w/ Hash Table. The former is a standard Count-Min sketch that completely removes the idea of separation, representing a naive data stream summarization approach without key tracking mechanism. The latter replaces \texttt{FLORE}’s Ostracism-based filtering algorithm with a simpler hash-table strategy (following the design of NZE-sketch) to investigate the impact of the proposed algorithmic design. We report the average time (or latency) required to summarize a single data item across all ten datasets in Figure~\ref{fig:af_time}. From the figure, we have three important findings. First, \texttt{FLORE} w/o Stream Filtering is even slower than Count-Min because it requires additional time to consult the Bloom Filter to determine whether a key belongs to the existing set. Second, even simply adding a hash table (i.e., \texttt{FLORE} w/ Hash Table) dramatically reduces processing latency, by more than 5× on average, and the performance gain becomes even greater as the skewness of the data stream increases. Third, our Ostracism-based mechanism further accelerates processing speed. Although \texttt{FLORE} involves a lookup operation, this operation runs in constant time and yields an additional 15\% improvement over the hash-table baseline.

\section{Related Works}
\label{app:works}

\paragraph{Classical sketching algorithms} Sketches are classical streaming algorithms and compact data structures that support approximate queries over large-scale data. Because streaming data comes at them continuously, and in such volume, they try to record the essence of what they’ve seen while strategically forgetting the rest. For more than 30 years, computer scientists have worked to build a better sketching algorithm. Representative examples include Count-Min Sketch~\cite{cormode2005improved} and Count Sketch~\cite{charikar2002finding}, which enable fast updates and queries with provable accuracy bounds. 
These sketches typically maintain an array of \(m\) counters organized into \(k\) rows, together with \(k\) pairwise-independent hash functions that map each item to one counter per row. For instance, in Count-Min Sketch, each update increments the \(k\) counters indexed by the hashed positions of the item. To estimate an item’s frequency, the sketch hashes the queried item and returns the minimum value among the corresponding counters across all rows.
When increments are positive, a simple modification known as Conservative Update~\cite{estan2003new} leads to better accuracy in frequency estimation. It avoids unnecessary increments by first comparing stored counts for an item and then only incrementing the smallest ones~\cite{mazziane2024count}. Moreover, Augmented sketch~\cite{roy2016augmented} uses an extra elephant-flow recording structure to further improve accuracy. Due to the skewed nature of data streams, it can potentially provide even higher throughput. Based on this idea, Pyramid sketch~\cite{yang2017pyramid}, Elastic sketch~\cite{yang2018elastic}, HeavyGuardian~\cite{yang2018heavyguardian}, HeavyKeeper~\cite{yang2019heavykeeper}, and Dynamic Hierarchical Sketch~\cite{zhao2021dhs} all store the elephant and mice flow information separately in the same space. As these sketches use the hash value as the index, they can support \(\mathcal{O}(1)\) insertion and query, but they have suboptimal estimation accuracy of per-element frequency, which in turn greatly reduces the accuracy of other statistics. Additional sketches have been proposed for specialized tasks, such as burst/DDoS victim detection~\cite{zhong2021burstsketch,yu2013software}, graph summarization~\cite{tang2016graph}, reporting top-$k$ items~\cite{metwally2005efficient,li2020wavingsketch}, and super-spreader query~\cite{tang2020spreadsketch}. However, they are typically optimized for a narrow set of queries and rely on carefully engineered heuristics. By comparison, our proposal enjoys a tighter theoretical error bound while remaining applicable to a broad class of sketching tasks.

\vspace{-2mm}
\paragraph{Learning-based sketching algorithms} In recent years, researchers have attempted to use machine learning (ML) methods to improve sketch performance. This type of sketching algorithm was first explored by~\cite{yang2018empowering} and~\cite{kraska2018case}.
The former tries to learn task-specific representations, while the latter trains a classifier to distinguish frequent data items; the overall idea and architecture are consistent with Augmented Sketch. Specifically, the learning-augmented sketch uses neural networks (NN), e.g., an RNN~\cite{hsu2019learning}, to learn and infer whether an item is large and uses an additional lookup table to record large items. This general paradigm of leveraging ML to improve hashing-based sketches has inspired a series of follow-up works~\cite{zhang2020learned,jiang2020learning,du2021putting,yan2022talentsketch,cheng2023alsketch}. However, because a forward pass through an ML model is required prior to each hash operation, these methods often incur substantial computational overhead, resulting in significantly reduced stream processing throughput. This issue becomes more pronounced as model complexity increases; for example, a recent study~\cite{li2025llm} employs a large language model (LLM) as a learned ``oracle'', which is typically impractical for high-speed data streams. 
Although \cite{da2022bayesian} and \cite{aamand2023improved} partially mitigate this efficiency bottleneck by reducing or even eliminating the need for per-item model inference, a second fundamental challenge remains: the reliance on continuously collected ground-truth data to train the neural models. To address this limitation, Meta-sketch~\cite{cao2024learning} is the first learning-based sketch that avoids ground-truth supervision by performing meta-learning over synthetically generated Zipfian datasets. This GT-free learning paradigm has subsequently been adopted and refined in Mayfly~\cite{feng2023mayfly}, Lego-sketch~\cite{feng2025lego}, and RatioSketch~\cite{wang2026ratiosketch}. Nevertheless, beyond the open question of whether Zipfian assumptions consistently hold in real-world workloads, the neural architectures employed by these methods still exhibit much lower processing throughput compared to conventional, non-learning-based sketches. 
In contrast to the aforementioned sketches, our method is grounded in the theory of compressed sensing. It is both more practical and more efficient, as it requires no modification to the classical summary algorithm and relies only on sampled counters for training.

\vspace{-2mm}
\paragraph{Decoding-based measurement} The idea of decoding or recovery from coarse-grained, low-cost measurements obtained via linear systems can be traced back to Network Tomography~\cite{vardi1996network}. By inferring the full end-to-end traffic matrix from more measurable link loads, this technique sparked a surge of research interest in network measurement~\cite{zhang2003fast,zhang2003information,zhang2009spatio,bowden2011network,chen2014robust,mardani2015estimating,10659230,yuan2025diffusion}. However, traffic matrices are defined at the wide area network (WAN) level and are far coarser than five-tuple packet measurements via data sketching, making them unsuitable for fine-grained applications and thus beyond the scope of this work.
Turning to sketching algorithms, \cite{lee2005improving} were the first to apply the Least Linear Square method to reconstruct flow values from Count-Min Sketch. While Counter-Braids~\cite{lu2008counter} highlighted the thematic connection between sketching and compressive sensing, it did not actually employ compressive sensing. In contrast, FlowRadar~\cite{li2016flowradar} and SketchVisor~\cite{huang2017sketchvisor} explicitly demonstrated that sketches can be formulated as a special case of compressive sensing. Specifically, they typically use a Bloom filter to record unique keys, and at the end of each measurement epoch, leverage this information to recover the complete measurement results using compressed sensing algorithms. This encoding–decoding paradigm culminated in NZE-sketch~\cite{huang2021toward} and PR-sketch~\cite{sheng2021pr}, subsequently inspiring a line of research on sketching algorithms for high-precision measurement~\cite{fu2020clustering,liu2021xy,li2023cs,he2023histsketch}. However, the improved accuracy of these approaches comes at the cost of query efficiency: both NZE-sketch and PR-sketch solve regularized optimization problems involving potentially millions of variables, which introduces significant complexity even with state-of-the-art numerical solvers. Our work is inspired by UCL-sketch~\cite{yuan2025learning}, which addresses this challenge using a learned solver. Similar to Meta-sketch, UCL-sketch relies on prior assumptions about the underlying data stream distribution. Different from them, this work leverages generative modeling to automatically infer and learn the distribution, thereby circumventing these potential limitations. Although some recent works~\cite{sun2024netgsr,wang2025resolving} have introduced deep generative models, such as GANs and autoregressive models, into network monitoring tasks, they neither perform systematic theoretical or experimental analysis nor avoid reliance on large pre-collected datasets for offline training.

\section{Conclusion}

This work revisits linear sketching techniques and highlights their limitations. We find that under resource (memory) constraints, existing techniques are fundamentally insufficient for near-perfect recovery. To bridge this gap, we introduce a novel generative sketching paradigm that reframes the problem as posterior generation. These analyses guide us to design \texttt{FLORE}, a GT-free ML framework, accordingly. The superior performance validates its effectiveness. 
We expect that more generative streaming algorithms can be designed in the future. 

At the end of this work, we discuss the practical applicability of \texttt{FLORE} across a range of sketch-based tasks and outline its current limitations along with promising directions for future work.

\paragraph{Applicable Scenarios}

First, we have demonstrated the applicability of \texttt{FLORE} to four tasks, including \textit{frequency estimation}, \textit{heavy hitter detection}, \textit{distribution estimation}, and \textit{entropy estimation}. Since our analysis covers almost all linear aggregation scenarios, which are the most fundamental statistics, the proposed approach can also perform well on all the following use cases. (i) \textit{Heavy changer detection}~\cite{tang2019mv}. The task aims to identify the elements whose frequency changes between two adjacent epochs are larger than the predefined threshold, with applications to imply the existence of a DDoS
attack~\cite{douligeris2004ddos}. (ii) \textit{Blackhole detection}~\cite{mai2011debugging}. In this task, the blackhole behavior is modeled by a normal entity A and a faulty entity B. Specifically, when the data stream is transmitted from entity A to entity B, entity B drops the items associated with certain specific keys due to an interface failure. The goal of the task is to identify the missing (i.e., victim) keys. (iii) \textit{Incorrect routing detection}~\cite{guo2015pingmesh}. Consider three entities A, B, and C, where entity A can send data items to both entity B and entity C. Incorrect routing refers to the scenario where, due to corruption in entity A's routing table, a subset of keys that should have been sent to entity B are mistakenly forwarded to entity C. The goal of the task is to identify the keys affected by this incorrect routing. The above three binary classification tasks can be naturally derived from our heavy-hitter detection task.
(iv) \textit{Norm estimation}~\cite{alon1996space}. The problem of $\ell_p$ norm estimation is to provide a multiplicative approximation to $|f|_p := (\sum_i |f_i|^p)^{1/p}$, which is similar to distribution \& entropy estimation. 
(v) \textit{Graph sketching}~\cite{mcgregor2014graph}. The task falls into the scope of combinatorial problems on graphs. We consider dynamic graph streams where edges can be both inserted and deleted. The input is a sequence 
\(
   S = \langle a_1, a_2, \ldots \rangle  
\), 
where $a_i = (e_i, \Delta_i)$ with  $e_i$  encoding an undirected edge as before, and  $\Delta_i \in \left\{-1, 1\right\}$. The multiplicity of an edge $e$ is defined as  $f_e = \sum_{i: e_i = e} \Delta_i$. For simplicity, we restrict our attention to the case where  $f_e$ is in $\left\{0, 1\right\}$  for all edges. Then, let $f \in \left\{0,1\right\}^{\binom{N}{2}}$ be the vector whose entries equal the current edge multiplicities  $f_e$, and let  $\mathcal{A}(f) \in \mathbb{R}^m$  denote the sketch of this vector. Upon arrival of an update $(e, \Delta)$, we can simply update $\mathcal{A}(f)$ as: $\mathcal{A}(f) \leftarrow \mathcal{A}(f) + \Delta \cdot \mathcal{A}(\mathbf{i}^e)$ where  $\mathbf{i}^e$  is the vector that is ``1'' only in the coordinate corresponding to edge  $e$ , and zero elsewhere. We can see that the formulation of this problem is entirely analogous to our linear sketching problem. Thus, it suffices to store the current sketch and leverage \texttt{FLORE} to compute the solution. 

\paragraph{Limitations and Future Research}

We believe that our investigation of generative optimization for data sketching has but scratched the surface. In what follows, we outline current limitations of our approach, as well as intriguing directions for future research. 
(i) In the \texttt{FLORE}'s data structures, a Bloom Filter is employed to determine whether an element belongs to a known set. As a probabilistic data structure, however, the Bloom Filter may incur false positives, and achieving an extremely low false-positive rate requires memory usage that grows nearly linearly with $N$. To mitigate this overhead, future work may explore learning-enhanced Bloom Filter~\cite{kraska2018case}, biasing key mappings toward higher-indexed regions of the bit array can substantially reduce storage requirements while preserving query accuracy. 
(ii) Note that our method reconstructs the state information of all elements simultaneously, which is not well-suited for point queries. For instance, after the model outputs  $f$, we must locate the position corresponding to a specific key and retrieve its value, an operation that incurs  $\mathcal{O}(N)$  complexity. A good direction for future research is the development of accelerated mapping indices. By decoupling the generative reconstruction from the lookup mechanism, we aim to achieve sub-linear or constant-time retrieval for specific keys, thereby enhancing \texttt{FLORE}'s utility for real-time, per-element monitoring.

\newpage

\bibliographystyle{unsrt}  
\bibliography{main}  






\newpage

\appendix

\part*{Appendix} 

\section{Background}
\label{app:bg}


\subsection{Variational Autoencoder}

The idea behind a Variational Autoencoder (VAE)~\cite{kingma2013auto} is, in fact, less closely related to conventional autoencoders and more deeply rooted in variational Bayesian inference and probabilistic graphical models. In VAEs, the conditional probability distribution $p_{\theta}(\mathbf{x}\mid\mathbf{z})$ defines the generative process, playing a role analogous to the decoder function $f_{\theta}(\mathbf{x}\mid\mathbf{z})$ introduced previously. Accordingly, $p_{\theta}(\mathbf{x}\mid\mathbf{z})$ is often referred to as the probabilistic decoder.
Similarly, the approximation function $q_{\phi}(\mathbf{z}\mid\mathbf{x})$ serves as the probabilistic encoder, analogous to the encoder function $g_{\phi}(\mathbf{z}\mid\mathbf{x})$.

The objective is to make the approximate posterior distribution $q_{\phi}(\mathbf{z}\mid\mathbf{x})$ as close as possible to the true posterior distribution $p_{\theta}(\mathbf{z}\mid\mathbf{x})$. This discrepancy can be quantified using the Kullback–Leibler Divergence (KL divergence). Expanding the KL divergence yields:
\begin{equation}
    \begin{aligned}
        \mathcal{D}_{\text{KL}}\left(q_{\phi}(\mathbf{z}\mid \mathbf{x}) | p_{\theta}(\mathbf{z}\mid\mathbf{x})\right) & = \int q_{\phi}(\mathbf{z}\mid\mathbf{x}) \log \frac{q_{\phi}(\mathbf{z}\mid\mathbf{x})}{p_{\theta}(\mathbf{z}\mid\mathbf{x})} d\mathbf{z} = \int q_{\phi}(\mathbf{z}\mid\mathbf{x}) \log \frac{q_{\phi}(\mathbf{z}\mid\mathbf{x})p_{\theta}(\mathbf{x})}{p_{\theta}(\mathbf{z}, \mathbf{x})} d\mathbf{z} \\
        & = \int q_{\phi}(\mathbf{z}\mid\mathbf{x}) \left( \log p_{\theta}(\mathbf{x}) + \log \frac{q_{\phi}(\mathbf{z}\mid\mathbf{x})}{p_{\theta}(\mathbf{z}, \mathbf{x})} \right) d\mathbf{z} \\
        & = \log p_{\theta}(\mathbf{x}) + \int q_{\phi}(\mathbf{z}\mid\mathbf{x}) \log \frac{q_{\phi}(\mathbf{z}\mid\mathbf{x})}{p_{\theta}(\mathbf{z}, \mathbf{x})} d\mathbf{z} \\
        & = \log p_{\theta}(\mathbf{x}) + \int q_{\phi}(\mathbf{z}\mid\mathbf{x}) \log \frac{q_{\phi}(\mathbf{z}\mid\mathbf{x})}{p_{\theta}(\mathbf{x}\mid\mathbf{z})p_{\theta}(\mathbf{z})} d\mathbf{z} \\
        & = \log p_{\theta}(\mathbf{x}) + \mathbb{E}_{\mathbf{z} \sim q_{\phi}(\mathbf{z}\mid\mathbf{x})} \left[\log \frac{q_{\phi}(\mathbf{z}\mid\mathbf{x})}{p_{\theta}(\mathbf{z})} - \log p_{\theta}(\mathbf{x}\mid\mathbf{z})\right] \\
        & = \log p_{\theta}(\mathbf{x}) + \mathcal{D}_{\text{KL}}\left(q_{\phi}(\mathbf{z}\mid\mathbf{x}) \mid p_{\theta}(\mathbf{z})\right) - \mathbb{E}_{\mathbf{z} \sim q{\phi}(\mathbf{z}\mid\mathbf{x})} \log p_{\theta}(\mathbf{x}\mid\mathbf{z})
    \end{aligned}
\end{equation}
Rearranging the terms gives:
\[ \log p_{\theta}(\mathbf{x}) - \mathcal{D}_{\text{KL}}\left(q_{\phi}(\mathbf{z}\mid\mathbf{x}) \mid p_{\theta}(\mathbf{z})) = \mathbb{E}_{\mathbf{z} \sim q_{\phi}(\mathbf{z}\mid\mathbf{x})} \log p_{\theta}(\mathbf{x}\mid\mathbf{z}) - \mathcal{D}_{\text{KL}}\left(q_{\phi}(\mathbf{z}\mid\mathbf{x}) \mid p_{\theta}(\mathbf{z}\right)\right). \]
The negative of this expression defines the VAE objective:
\[
    \mathcal{L}_{\text{VAE}}(\theta, \phi) = -\log p_{\theta}(\mathbf{x}) + \mathcal{D}_{\text{KL}}\left(q_{\phi}(\mathbf{z}\mid\mathbf{x}) \mid p_{\theta}(\mathbf{z}\mid\mathbf{x})\right) = -\mathbb{E}_{\mathbf{z} \sim q_{\phi}(\mathbf{z}\mid\mathbf{x})} \log p_{\theta}(\mathbf{x}\mid\mathbf{z}) - \mathcal{D}_{\text{KL}}\left(q_{\phi}(\mathbf{z}\mid\mathbf{x}) \mid p_{\theta}(\mathbf{z})\right).
\]

In variational Bayesian inference, this objective is commonly known as the evidence lower bound (ELBO), or variational lower bound. The term “lower bound” arises because KL divergence is always non-negative and thus 
\[
    -\mathcal{L}_{\text{VAE}} = \log p_{\theta}(\mathbf{x}) - \mathcal{D}_{\text{KL}}\left(q_{\phi}(\mathbf{z}\mid\mathbf{x}) \mid p_{\theta}(\mathbf{z}\mid\mathbf{x})\right) \le \log p_{\theta}(\mathbf{x})
\] 
Consequently, minimizing the VAE loss is equivalent to maximizing a lower bound on the log-likelihood of the observed data, thereby improving the model’s ability to generate realistic samples.

\subsection{Generative Adversarial Network}

Generative Adversarial Network (GAN)~\cite{goodfellow2014generative} consists of two models:
(i) A discriminator $D$, which estimates the probability that a given sample originates from the real dataset. Acting as a critic, it is optimized to distinguish real samples from generated ones.
(ii) A generator $G$, which produces synthetic samples from a noise variable $z$ (where $z$ introduces output diversity). The generator is trained to approximate the real data distribution such that its generated samples appear increasingly realistic, ultimately encouraging the discriminator to assign them high probabilities.
During training, these two models engage in an adversarial process: the generator $G$ continuously learns to fool the discriminator, while the discriminator $D$ simultaneously improves its ability to detect generated samples. This game drives both models to progressively enhance their capabilities.

Define $p_z$ as the distribution of the noise input $z$, $p_g$ as the generator-induced distribution over data samples $x$, and $p_r$ as the distribution of real data samples.
On the one hand, the discriminator $D$ is trained to correctly classify real samples by maximizing $\mathbb{E}{x \sim p_r(x)} [\log D(x)]$. At the same time, for a generated sample $G(z)$ with $z \sim p_z(z)$, the discriminator is encouraged to assign a probability close to zero to the fake sample by maximizing $\mathbb{E}{z \sim p_z(z)} [\log(1 - D(G(z)))]$.
On the other hand, the generator $G$ aims to produce realistic samples that can fool the discriminator. Therefore, it is trained to minimize $\mathbb{E}_{z \sim p_z(z)} [\log(1 - D(G(z)))]$, thereby encouraging $D(G(z))$ to be as large as possible.

Combining these two objectives, the discriminator and generator engage in a minimax game, formulated through the following objective function:
\begin{equation}
    \begin{aligned}
        \min_G \max_D \mathcal{L}(D, G) &= \mathbb{E}_{x \sim p_r(x)} [\log D(x)] + \mathbb{E}_{z \sim p_z(z)} [\log(1 - D(G(z)))] \\
        &= \mathbb{E}_{x \sim p_r(x)} [\log D(x)] + \mathbb{E}_{x \sim p_g(x)} [\log(1 - D(x))]
    \end{aligned}
\end{equation} 

Wasserstein Distance can be informally interpreted as the minimum cost required to transform one probability distribution into another, analogous to moving a pile of dirt from one shape to another. The transportation cost is defined as the amount of dirt being moved multiplied by the distance it travels. Unlike divergence measures such as KL divergence or Jensen--Shannon divergence, the Wasserstein metric provides a smoother and more meaningful measure of discrepancy between distributions, which significantly improves the stability of gradient-based optimization during training~\cite{arjovsky2017wasserstein}.

Compared with the original Generative Adversarial Network (GAN), the Wasserstein Generative Adversarial Network (WGAN) introduces two major modifications:
(i) After each gradient update of the critic network, the model weights are clipped to lie within a small fixed range in order to enforce the Lipschitz continuity constraint.
(ii) The original GAN loss is replaced with a new objective derived from the Wasserstein distance, eliminating the logarithmic terms used in the standard adversarial loss.

Under this formulation, the discriminator no longer serves as a binary classifier that directly distinguishes real samples from generated ones. Instead, it acts as a critic that estimates the Wasserstein distance between the real and generated data distributions.

\subsection{Flow-based Generative Model}

The key idea behind flow-based generative models is to transform a simple base distribution into a highly complex target distribution through a sequence of invertible mappings. As samples pass through these transformations, their probability densities are updated iteratively according to the change-of-variables theorem. In this manner, Normalizing Flows (NFs)~\cite{rezende2015variational} preserve exact likelihood computation while greatly enhancing the expressive power of the modeled distribution. Here we define
\[
     (i)~\mathbf{z}_{i-1} \sim p_{i-1}(\mathbf{z}_{i-1}) \quad
     (ii)~\mathbf{z}_i = T_i(\mathbf{z}_{i-1}), \text{ thus } \mathbf{z}_{i-1} = T_i^{-1}(\mathbf{z}_i) \quad
     (iii)~p_i(\mathbf{z}_i) = p_{i-1}(T_i^{-1}(\mathbf{z}_i)) \left| \det \frac{dT_i^{-1}}{d\mathbf{z}_i} \right|
\]

Then, we rewrite the equation in terms of $\mathbf{z}_i$ so that inference can be performed directly in the base distribution space.
\[
    p_i(\mathbf{z}_i) = p_{i-1}(T_i^{-1}(\mathbf{z}_i)) \left| \det \frac{dT_i^{-1}}{d\mathbf{z}_i} \right| = p_{i-1}(\mathbf{z}_{i-1}) \left| \det \left( \frac{dT_i}{d\mathbf{z}_{i-1}} \right)^{-1} \right| = p_{i-1}(\mathbf{z}_{i-1}) \left| \det \frac{dT_i}{d\mathbf{z}_{i-1}} \right|^{-1}
\]

Given this chain of probability density functions, the relationship between every pair of consecutive variables is explicitly defined. Therefore, we can recursively expand the density of the output variable $\mathbf{x}$, tracing it step by step back to the initial latent distribution $\mathbf{z}_0$.
\begin{equation}
    \begin{aligned}
        \mathbf{x} &= \mathbf{z}_K = T_K \circ T_{K-1} \circ \cdots \circ f_1(\mathbf{z}_0) \\
        \log p(\mathbf{x}) &= \log \pi_K(\mathbf{z}_K) = \log \pi_{K-1}(\mathbf{z}_{K-1}) - \log \left| \det \frac{dT_K}{d\mathbf{z}_{K-1}} \right| \\
        &= \log \pi_{K-2}(\mathbf{z}_{K-2}) - \log \left| \det \frac{dT_{K-1}}{d\mathbf{z}_{K-2}} \right| - \log \left| \det \frac{dT_K}{d\mathbf{z}_{K-1}} \right| \\
        &= \cdots = \log \pi_0(\mathbf{z}_0) - \sum_{i=1}^{K} \log \left| \det \frac{dT_i}{d\mathbf{z}_{i-1}} \right|
    \end{aligned}
\end{equation}

The trajectory defined by the sequence of random variables $\mathbf{z}i = T_i(\mathbf{z}{i-1})$ is referred to as the flow, while the entire sequence of intermediate distributions ${\pi_i}$ constitutes a normalizing flow. To make the density transformation computationally feasible, each transformation function $T_i$ must satisfy two key properties: (i) it must be easily invertible, and (ii) the determinant of its Jacobian must be efficiently computable.

With normalizing flows, the exact log-likelihood of an observed sample $\log p(\mathbf{x})$ becomes tractable. Consequently, flow-based generative models can be trained directly through maximum likelihood estimation. Given a training dataset $\mathcal{D}$, the objective function is defined as the negative log-likelihood:
$\mathcal{L}(\mathcal{D}) = - \frac{1}{\|\mathcal{D}\|} \sum_{\mathbf{x} \in \mathcal{D}} \log p(\mathbf{x})$.

\subsection{Diffusion Model}

Diffusion model is a class of deep generative models that parameterize the encoding (forward) and decoding (reverse) processes via a diffusion process. The key idea behind them is to add small amounts of noise gradually to a data point for transitions from data manifold to the noisy manifolds, then train the underlying neural network to transport from the pure Gaussian to the clean area through inverting the diffusion process.

Suppose $p_0\left( \boldsymbol{x_0}\right)$ be the $d$-dimensional data distribution. We consider the general class of score-based diffusion models in a continuous form that can be described with the following 
Ornstein-Uhlenbeck stochastic differential equation (SDE)~\cite{song2020score}
\begin{equation}
    d\boldsymbol{x} = \boldsymbol{f}\left( \boldsymbol{x}, t\right)dt + g\left( t\right)d\boldsymbol{w}
\end{equation}
where $\boldsymbol{w}$ denotes the standard Brownian motion, $\boldsymbol{f}: \mathbb{R}^d \times \mathbb{R} \to \mathbb{R}^d$, $g: \mathbb{R} \to \mathbb{R}$ is the linear drift function and scalar diffusion coefficient, respectively. The SDE results in a series of marginal distributions $\left\{p_t\left(\boldsymbol{x_t}\right)\right\}$ where $t \in \left[0,T\right]$, so that $\boldsymbol{x_T} \sim {\cal N}\left( 0,\sigma_T^2\right)$ with some constants $\sigma_T>0$.

It is known that the density evolution process can be reversed by another SDE with Anderson’s theorem~\cite{anderson1982reverse}
\begin{equation}
    d\boldsymbol{x} = \left[ {\boldsymbol{f}\left( {\boldsymbol{x},t} \right) - {g^2}\left( t \right){\nabla _{\boldsymbol{x_t}}}\log {p_t}\left( \boldsymbol{x_t} \right)} \right]dt + g\left( t \right)d\boldsymbol{\bar w}
    \label{sgm_reverse}
\end{equation}
where $\boldsymbol{\bar w}$ is a Brownian motion running backward in time from $T$ to $0$. Then diffusion models approximate this reverse process by learning the score function ${\nabla _{\boldsymbol{x_t}}}\log {p_t}\left( \boldsymbol{x_t} \right)$ through a score-matching loss:
\begin{equation}
   \mathbb{E}_{t,\boldsymbol{x_0},\boldsymbol{x_t}|\boldsymbol{x_0}} \left[\lambda_t {\left\| {{s_\theta }\left( {\boldsymbol{x_t},t} \right) - {\nabla _{\boldsymbol{x_t}}}\log {p_t}\left( {\boldsymbol{x_t}} \right)} \right\|_2^2}\right]
\end{equation}
where $\lambda_t$ is a positive weight coefficient, and $t$ is uniformly
sampled over $\left[0, T\right]$. With a learned score function ${s_\theta }\left( {\boldsymbol{x_t},t} \right)$, one can then solve the reverse SDE through
\begin{equation}
    d\boldsymbol{x} = \left[ {\boldsymbol{f}\left( {\boldsymbol{x},t} \right) - {g^2}\left( t \right){s_\theta }\left( {\boldsymbol{x_t},t} \right)} \right]dt + g\left( t \right)d\boldsymbol{\bar w}
\end{equation}
by first sampling from the prior Gaussian $p_T\left( \boldsymbol{x_T}\right)$.

In this work, we focus on Denoising Diffusion Probabilistic Models~\cite{ho2020denoising}, which is equivalent to the variance preserving form of the SDE (VP-SDE). To be more specific, we have the forward and reverse SDEs
as the continuous version of the diffusion process in DDPM with the choice of $\boldsymbol{f}\left( {\boldsymbol{x},t} \right) = -\beta\left(t\right) \boldsymbol{x}/2$ and $g\left( t \right) = \sqrt{\beta\left(t\right)}$. The corresponding forward SDE can be formulated as follows:
\begin{equation}
    d\boldsymbol{x} = - \frac{\beta\left(t\right)}{2}\boldsymbol{x}dt + \sqrt {{\beta\left(t\right)}} d\boldsymbol{w}.
\end{equation}
In particular, the forward process of DDPMs gradually corrupts the original data $x_0 \in \mathbb{R}^d$ via a fixed Markov chain $x_0, \dots, x_T$ with each variable in $\mathbb{R}^d$ as follows:
\begin{equation}
   \left\{\begin{aligned}
       &\boldsymbol{x_t}|\boldsymbol{x_{t-1}} \sim {\cal N}\left(\boldsymbol{x_t};\sqrt{{\alpha _t}}{x_{t-1}},(1-{\alpha_t}){\rm I}\right), \\
       &\boldsymbol{x_t}|\boldsymbol{x_0} \sim {\cal N}\left(\boldsymbol{x_t};\sqrt{{\bar\alpha _t}}{x_0},(1-{\bar\alpha_t}){\rm I}\right), \\
       & \boldsymbol{x_t}=\sqrt{{{\bar\alpha}_t}}\boldsymbol{x_0}+\sqrt{1-{{\bar\alpha}_t}}\epsilon
   \end{aligned}\right.
\label{forward}
\end{equation}
with $\epsilon \sim{\cal N}(0,\rm{I})$, $\alpha_t:=1-\beta_t$ where $\beta_t \in \left(0, 1\right)$ is a variance at diffusion step $t$, scaling factor $\bar\alpha_t:=\prod\limits_{s=1}^t \alpha_s$.

Starting from the Gaussian noise $\boldsymbol{x_T}$, we can run the reverse process parametrized by the model ${p_\theta }\left(\boldsymbol{x_{t - 1}}|{x_t}\right): = {\cal N}(\boldsymbol{x_{t - 1}};{\mu}\left(\boldsymbol{x_t},t,\theta\right),\boldsymbol{\Sigma} \left(\boldsymbol{x_t},t,\theta\right) )$ to get $\boldsymbol{x_0}$. The diffusion model is trained to maximize the marginal likelihood of the data ${\mathbb{E}_{\boldsymbol{x_0}}}\left[\log {p_\theta }(\boldsymbol{x_0})\right]$, and we can write the variational lower bound with KLD as follows:
\begin{equation}
        {\boldsymbol{\cal L}_{vlb}} :=  \underbrace{- \log {p_\theta }(\boldsymbol{x_0}|\boldsymbol{x_{1}})}_{{\boldsymbol{\cal L}_0}} + \underbrace{{{\cal D}_{KL}}(p(\boldsymbol{x_T}|\boldsymbol{x_0})\ || \ p(\boldsymbol{x_T}))}_{{\boldsymbol{\cal L}_T}} + \sum\limits_{t = 2}^T {\underbrace{{{\cal D}_{KL}}(p(\boldsymbol{x_{t - 1}}|\boldsymbol{x_t},\boldsymbol{x_0})\ || \ {p_\theta }(\boldsymbol{x_{t - 1}}|\boldsymbol{x_t}))}_{{\boldsymbol{\cal L}_{t - 1}}}}
\end{equation}

The main objective is a sum of independent terms $\boldsymbol{\cal L}_{t - 1}$. There are many different ways to parameterize the posterior mean ${\mu}\left(\boldsymbol{x_t},t,\theta\right)$, and the most obvious option is to predict ${\mu}\left(\boldsymbol{x_t},t,\theta\right)$ directly:
\begin{equation}
    {\boldsymbol{\cal L}_{o}}: = {\mathbb{E}_{t,\boldsymbol{x_0}}}\left[\frac{1}{{2{{\boldsymbol{\Sigma}^2 }\left(\boldsymbol{x_t},t,\theta\right)}}}{\left\| {\hat \mu (\boldsymbol{x_t},\boldsymbol{x_0}) - {\mu}\left(\boldsymbol{x_t},t,\theta\right)} \right\|^2}\right],
\end{equation}
where ${{\boldsymbol{\Sigma} }\left(\boldsymbol{x_t},t,\theta\right)}$ is often set to pre-defined time dependent constants, and $\hat \mu (\boldsymbol{x_t},\boldsymbol{x_0})$ is the mean of the posterior $p(\boldsymbol{x_t}|\boldsymbol{x_{t - 1},x_0})$ which are defined as follows:
\begin{equation}
    \hat \mu (\boldsymbol{x_t},\boldsymbol{x_0})=\frac{{\sqrt {{{\bar \alpha }_{t - 1}}} {\beta _t}}}{{1 - {{\bar \alpha }_t}}}\boldsymbol{x_0} + \frac{{\sqrt {{\alpha _t}} (1 - {{\bar \alpha }_{t - 1}})}}{{1 - {{\bar \alpha }_t}}}\boldsymbol{x_t}.
    \label{posterior_mean}
\end{equation}
Alternatively, the network could also predict $\boldsymbol{x_0}$ or $\boldsymbol{\epsilon}$ using (\ref{forward}) and (\ref{posterior_mean}). In this work, we train our model to reconstruct input $\boldsymbol{x_0}$ itself combined with a reweighted loss function:
\begin{equation}
{\boldsymbol{\cal L}_{simple}}={\mathbb{E}_{t,\epsilon,\boldsymbol{x_0}}}\left[{{{\left\|{\boldsymbol{x_0}-{\boldsymbol{\hat x_0}}(\boldsymbol{x_t},t,\theta)}\right\|}^2}}\right].
\end{equation}

The connection between the score function and the prediction in DDPMs can be formulated approximately as: $s_\theta\left( \boldsymbol{x_t}, t\right) \approx \left({\boldsymbol{x_t} - \sqrt {{{\bar \alpha }_t}} \boldsymbol{{\hat x}_0}\left( {\boldsymbol{x_t},t,\theta } \right)}\right)/\left({{1 - {{\bar \alpha }_t}}}\right)$~\cite{yuan2025diffusion}. Data generation through denoising depends on the score function and can be seen as noise conditional score-based generation. Then the reverse process $p_\theta(\boldsymbol{x_{t-1}}|\boldsymbol{x_t})$ in DDPMs can be written as follows:
{\setlength\abovedisplayskip{0.1cm}
\setlength\belowdisplayskip{0.1cm}
\begin{equation}
        \boldsymbol{x_{t - 1}} = \frac{{\sqrt {{\alpha _t}} (1 - {{\bar \alpha }_{t - 1}})}}{{1 - {{\bar \alpha }_t}}}\boldsymbol{x_t} + \frac{{\sqrt {{{\bar \alpha }_{t - 1}}} {\beta _t}}}{{1 - {{\bar \alpha }_t}}}\boldsymbol{\hat x_0}(\boldsymbol{x_t},t,\theta )  + \frac{{1 - {{\bar \alpha }_{t - 1}}}}{{1 - {{\bar \alpha }_t}}}{\beta _t}\boldsymbol{z_t}.
\end{equation}

\section{Useful Lemmas}

\begin{lightgraybox}
\begin{lemma}
    Let $\mathbf{u}, \mathbf{v} \in \mathbb{C}^N$ be vectors such that $\|\mathbf{u}\|_0 \le s$ and $\|\mathbf{v}\|_0 \le t$. Assume that their supports are disjoint, i.e., \( \operatorname{supp}(\mathbf{u}) \cap \operatorname{supp}(\mathbf{v}) = \varnothing \). Then
    \(
        \bigl| \langle \mathbf{Au}, \mathbf{Av} \rangle \bigr| \;\le\; \delta_{s+t}\, \|\mathbf{u}\|_2 \, \|\mathbf{v}\|_2
    \), where $\delta_{s+t}$ is the isometry constant of matrix $\mathbf{A}$.
    \label{lem:rip_ineqn}
\end{lemma}
\end{lightgraybox}

\begin{proofbox}
\begin{proof}
    Let \(S := \operatorname{supp}(\mathbf{u}) \cup \operatorname{supp}(\mathbf{v})\), so that $|S| \le s+t$. Denote by $\mathbf{u}_S, \mathbf{v}_S \in \mathbb{C}^{|S|}$ the restrictions of $\mathbf{u}$ and $\mathbf{v}$ to the index set $S$, and by $\mathbf{A}_S$ the submatrix of $A$ consisting of the columns indexed by $S$. Then $\mathbf{Au} = \mathbf{A}_S \mathbf{u}_S$ and $\mathbf{Av} = \mathbf{A}_S \mathbf{v}_S$. Since $\operatorname{supp}(\mathbf{u}) \cap \operatorname{supp}(\mathbf{v}) = \varnothing$, we have \(\ \langle \mathbf{u}_S, \mathbf{v}_S \rangle = 0\). Therefore, 
    \[
       \bigl| \langle \mathbf{Au}, \mathbf{Av} \rangle \bigr| = \bigl| \langle \mathbf{A}_S \mathbf{u}_S, \mathbf{A}_S \mathbf{v}_S \rangle - \langle \mathbf{u}_S, \mathbf{v}_S \rangle \bigr| = \bigl| \langle (\mathbf{A}_S^* \mathbf{A}_S - I) \mathbf{u}_S, \mathbf{v}_S \rangle \bigr|. 
    \]
    Applying the Cauchy--Schwarz inequality yields
    \[
        \bigl| \langle (\mathbf{A}_S^* \mathbf{A}_S - \mathbf{Id}) \mathbf{u}_S, \mathbf{v}_S \rangle \bigr| \le \| \mathbf{A}_S^* \mathbf{A}_S - \mathbf{Id} \|_{2 \to 2} \, \|\mathbf{u}_S\|_2 \, \|\mathbf{v}_S\|_2.
    \]
    By the RIP of order $s+t$, i.e., Definition~\ref{def:ric}, we have \( \| \mathbf{A}_S^* \mathbf{A}_S - \mathbf{Id} \|_{2 \to 2} \le \delta_{s+t} \). Finally, noting that $\|\mathbf{u}_S\|_2 = \|\mathbf{u}\|_2$ and $\|\mathbf{v}_S\|_2 = \|\mathbf{v}\|_2$, we obtain \( \bigl| \langle \mathbf{Au}, \mathbf{Av} \rangle \bigr| \le \delta_{s+t} \, \|\mathbf{u}\|_2 \, \|\mathbf{v}\|_2 \), which completes the proof.
\end{proof}
\end{proofbox}

\begin{lightgraybox}
\begin{lemma}
    For any $x \in (0, 1)$, we have that \( \sqrt{x+1} < 1 + \left( \sqrt{2} + 1 \right) x \).
    \label{lem:c_ineqn}
\end{lemma}
\end{lightgraybox}

\begin{proofbox}
\begin{proof}
    Let $g\left( x \right) = \left( \sqrt{2} + 1 \right) x - \sqrt{x+1} + 1$. Then, \( g'\left( x \right) = \sqrt{2} + 1  - \frac{1}{2\sqrt{x+1}} \). Also, $g\left( 0 \right) = 0$, $g\left( 1 \right) = 2$, and $g'\left( x \right) > 0$ for any $x \in (0, 1)$. Therefore, \( \sqrt{x+1} < 1 + \left( \sqrt{2} + 1 \right) x, \forall x \in (0, 1) \). 
\end{proof}
\end{proofbox}

\begin{lightgraybox}
\begin{lemma}
     Let $ \mathbf{x} \in \mathbb{R}^N $ be an $ s $-sparse vector, and $ \mathbf{A} \in \mathbb{R}^{m \times N} $ be a random dense matrix drawn from the Gaussian matrix ensemble. Let $ \varepsilon \in (0,1) $. If 
     \begin{equation}
         \frac{m^2}{m+1} \geq 2s \left( \sqrt{\ln(2.34N/s)} + \sqrt{\frac{\ln(\varepsilon^{-1})}{s}} \right)^2
         \label{ineq:gaum}
     \end{equation}
     then with probability at least $ 1 - \varepsilon $, $ \mathbf{x} $ is the unique minimizer of $ \left\| \mathbf{z} \right\|_1 $ subject to $ \mathbf{A}\mathbf{x} = \mathbf{A}\mathbf{z} $.
     \label{lem:gaum}
\end{lemma}
\end{lightgraybox}

\begin{proofbox}
\begin{proof}
    This lemma is a direct application of Theorem 9.16 from page 258 of~\cite{foucart2022mathematical}.
\end{proof}
\end{proofbox}

\begin{lightgraybox}
\begin{lemma}
    Let $\mathcal{F}$ be a unit ball in a universal reproducing kernel Hilbert space $\mathcal{H}$, defined on the compact metric space $\mathcal{X}$, with associated continuous kernel $k(\cdot,\cdot)$. Then $\text{MMD}[\mathcal{F}, p, q] = 0$ if and only if $p = q$.
    \label{lem:mmd}
\end{lemma}
\end{lightgraybox}

\begin{proofbox}
\begin{proof}
    The original proof can be found in~\cite{gretton2012kernel} (see its Theorem 5).
\end{proof}
\end{proofbox}

\section{Visualizations}


\subsection{Learned Generative Flow}

We would like to clarify that the role of $z$ in our framework is not to encode semantically disentangled features, but rather to serve as an auxiliary latent variable that completes the missing information in $f$ conditioned on the sketch $b$. As such, our objective is aligned with faithful posterior modeling, rather than interpretable or controllable factorization of latent space.   
Also note that, unlike VAEs, flow-based models do not enforce clear semantic disentanglement on $z$. Nevertheless, we still conduct a toy experiment to better characterize the flow between latent $z$ and original $f$, by following \href{https://sites.google.com/view/berkeley-cs294-158-sp20/home#h.p_E-C2dsllTu6x}{this URL}. Specifically, we visualize how the distribution of $z$ from $f$ evolves during training on a synthetic Gaussian mixture dataset, as shown in Figure~\ref{fig:flow}. The figure demonstrates that the flow successfully models the generative prior with a meaningful mapping from GT to $z$.

\begin{figure}
    \centering
    \includegraphics[width=0.95\linewidth]{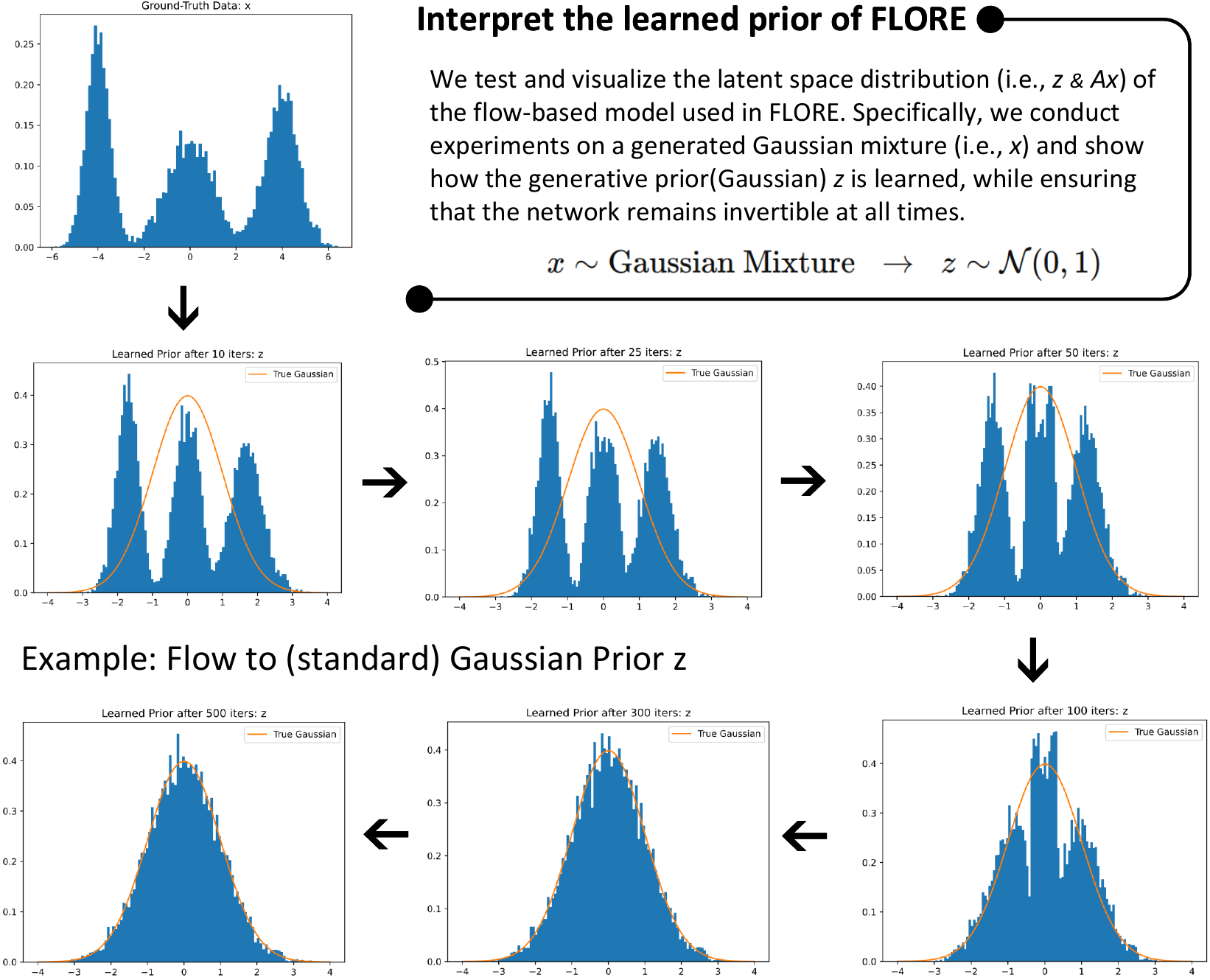}
    \caption{Visualizations of the learned flow-based mapping during training.}
    \label{fig:flow}
\end{figure} 

\subsection{Recovery (distributional) Results}

To gain insights into the behaviors of sketches, we next visualize and compare the recovery results from each sketch viewpoint (i.e., \texttt{FLORE}, CM, CS, NZE, LCM, and LCS), which can also be regarded as a reflection of how well the distribution is perceived by them. The resulting visualizations under different distributions and space budgets are shown from Figure~\ref{fig:syn_dist_32} to Figure~\ref{fig:syn_dist_256}, where key-wise frequencies are aggregated by their nearest 100 elements. Similar to the visualizations in the previous section, here the more the shaded areas overlap and the smaller the area, the better the results. First of all, the output of all sketches becomes more realistic as the memory budget increases. We also observe that all non-\texttt{CS} methods significantly overestimate the true frequencies of most elements because they lack a global perspective. Meanwhile, due to the need to satisfy both consistency and sparsity constraints, the total area of NZE using compressive sensing is close to that of the true distribution; however, it often compensates for high-frequency items by underestimating the majority of values or otherwise. In comparison, \texttt{FLORE} consistently achieves the highest fidelity in both curve shape and overall pattern under identical settings. Although \texttt{FLORE} partially overestimates the true values of some elements at 32KB under the guidance of the CM enhanced by EM refinements, it requires only 128KB to become nearly indistinguishable from the ground-truth pattern, achieving an almost perfect fit. This aligns with our quantitative distribution estimation results (WMRE) and can be attributed to the powerful distribution approximation capability of the generative modeling paradigm.

\begin{figure}
    \centering
    \subfigure[{Zipf-icml Distribution}\label{subfig:zipf_icml_dist_32}]{
        \centering
        \includegraphics[width=1.\linewidth]{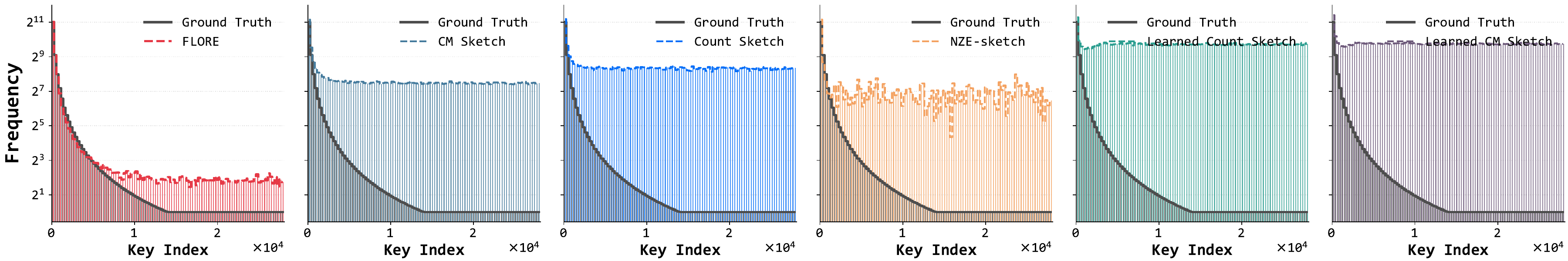}
    }
    \subfigure[{Zipf Distribution}\label{subfig:zipf_dist_32}]{
        \centering
        \includegraphics[width=1.\linewidth]{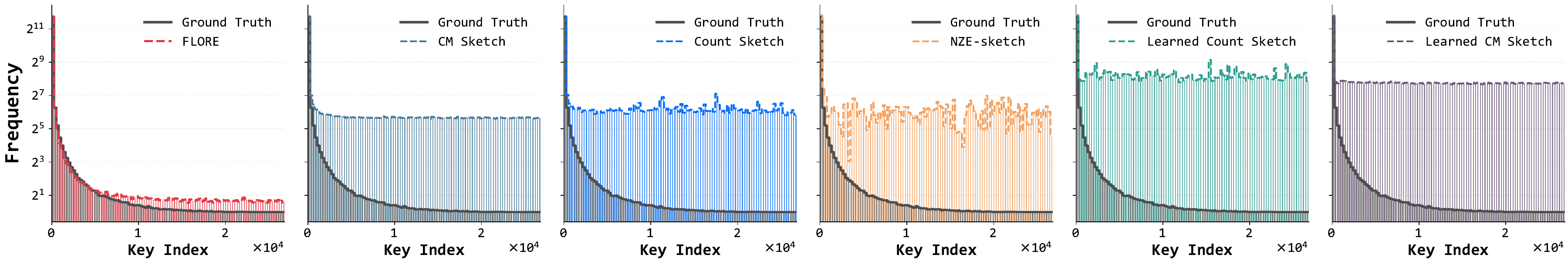}
    }
    \subfigure[{Pareto Distribution}\label{subfig:pareto_dist_32}]{
        \centering
        \includegraphics[width=1.\linewidth]{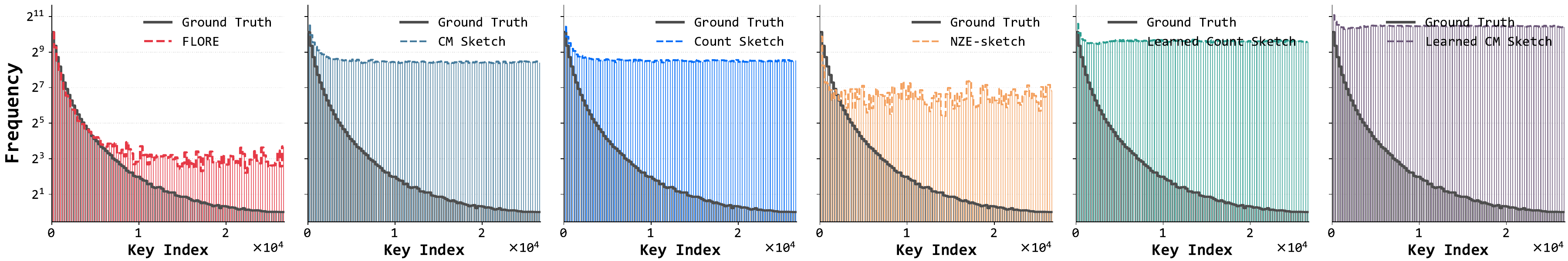}
    }
    \subfigure[{Exponential Distribution}\label{subfig:exponential_dist_32}]{
        \centering
        \includegraphics[width=1.\linewidth]{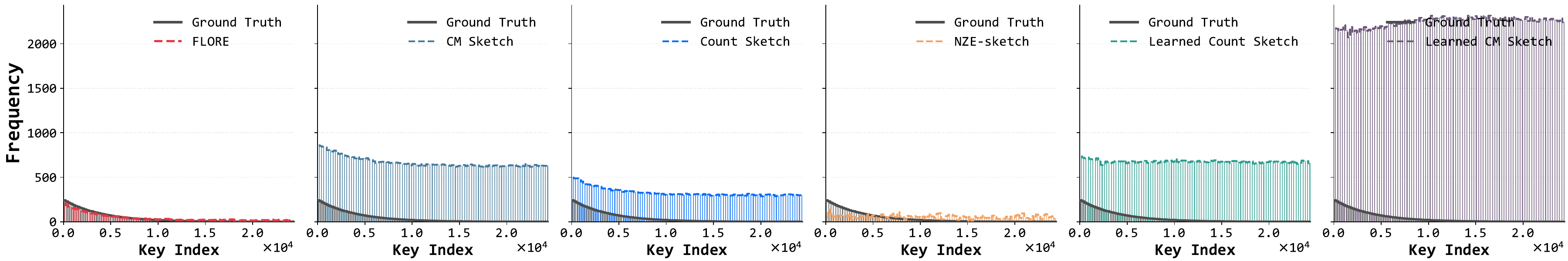}
    }
    \subfigure[{Log-Normal Distribution}\label{subfig:lognormal_dist_32}]{
        \centering
        \includegraphics[width=1.\linewidth]{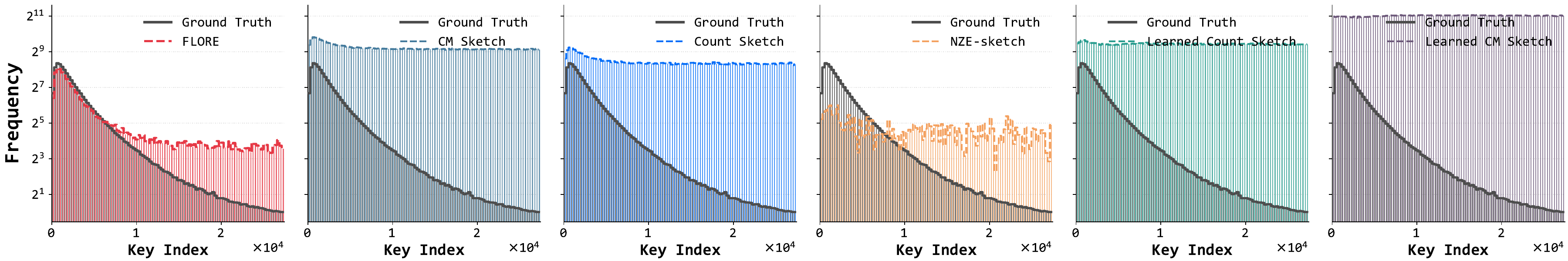}
    }
    \caption{Visualizations of sketch recovery under a fixed \textit{32KB} memory budget. From left to right, we present the recovery results of \texttt{FLORE}, CM, CS, NZE, LCM, and LCS, compared against the ground-truth sample distribution. To enhance readability, every 100 consecutive keys are aggregated in the x-axis.}
    \label{fig:syn_dist_32}
\end{figure}

\begin{figure}
    \centering
    \subfigure[{Zipf-icml Distribution}\label{subfig:zipf_icml_dist_128}]{
        \centering
        \includegraphics[width=1.\linewidth]{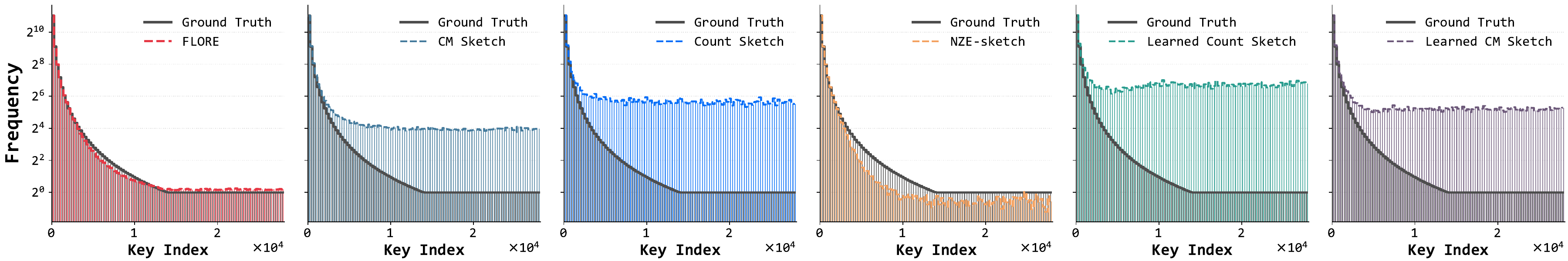}
    }
    \subfigure[{Zipf Distribution}\label{subfig:zipf_dist_128}]{
        \centering
        \includegraphics[width=1.\linewidth]{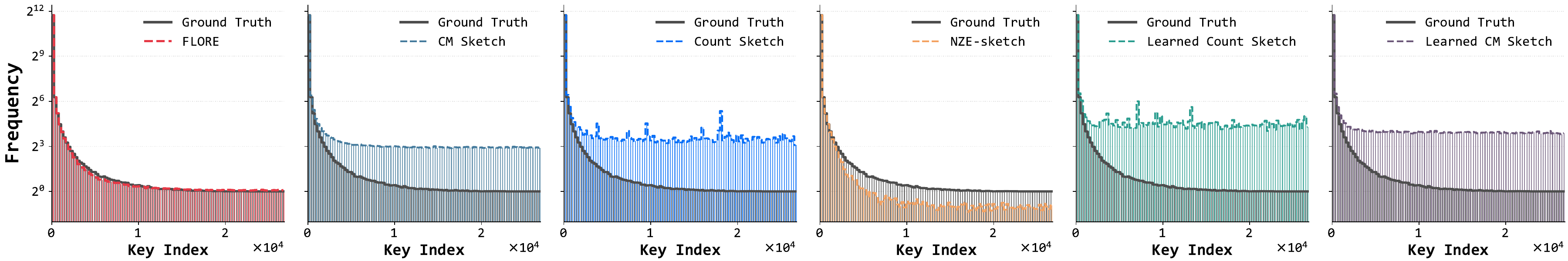}
    }
    \subfigure[{Pareto Distribution}\label{subfig:pareto_dist_128}]{
        \centering
        \includegraphics[width=1.\linewidth]{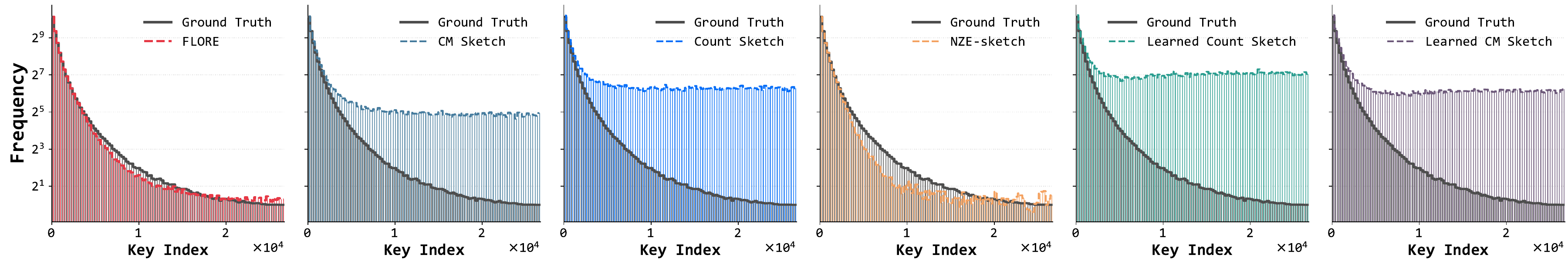}
    }
    \subfigure[{Exponential Distribution}\label{subfig:exponential_dist_128}]{
        \centering
        \includegraphics[width=1.\linewidth]{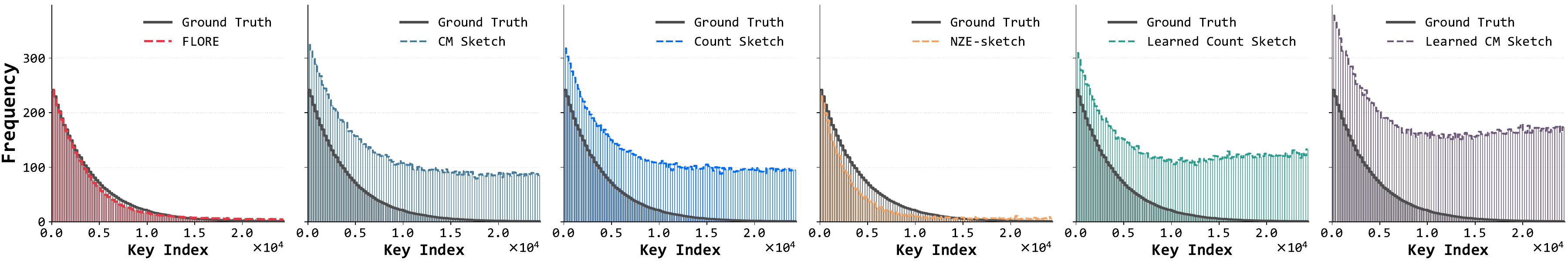}
    }
    \subfigure[{Log-Normal Distribution}\label{subfig:lognormal_dist_128}]{
        \centering
        \includegraphics[width=1.\linewidth]{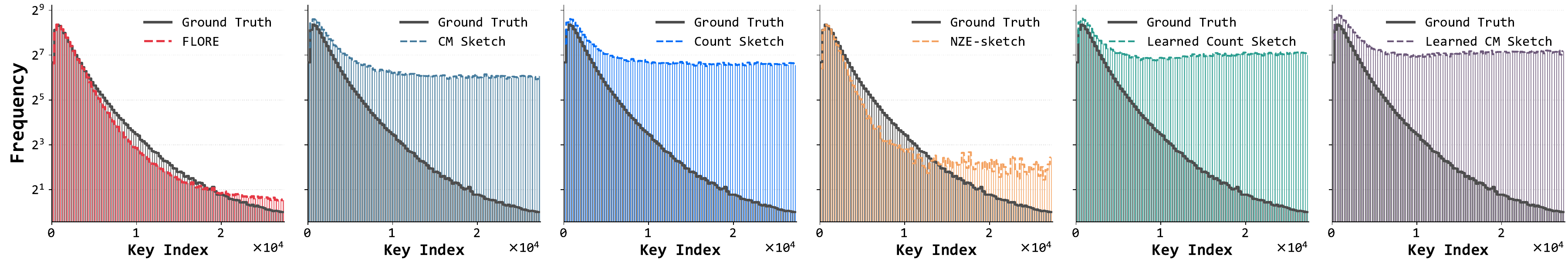}
    }
    \caption{Visualizations of sketch recovery under a fixed \textit{128KB} memory budget. From left to right, we present the recovery results of \texttt{FLORE}, CM, CS, NZE, LCM, and LCS, compared against the ground-truth sample distribution. To enhance readability, every 100 consecutive keys are aggregated in the x-axis.}
    \label{fig:syn_dist_128}
\end{figure}

\begin{figure}
    \centering
    \subfigure[{Zipf-icml Distribution}\label{subfig:zipf_icml_dist}]{
        \centering
        \includegraphics[width=1.\linewidth]{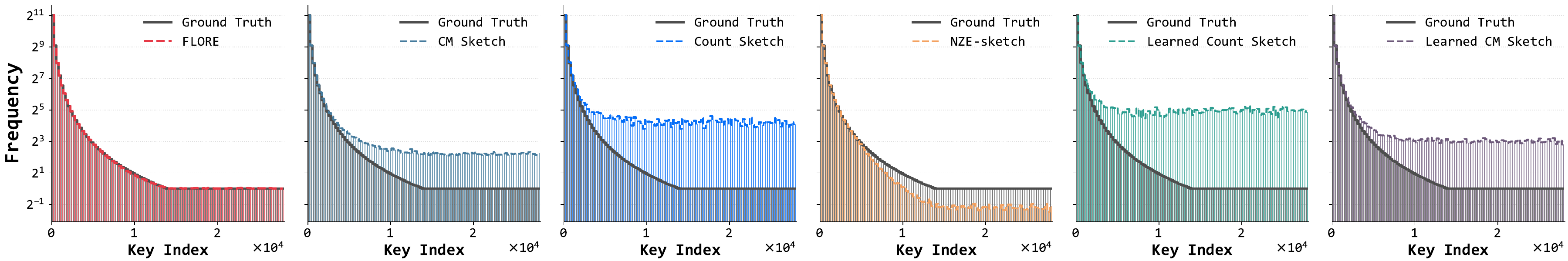}
    }
    \subfigure[{Zipf Distribution}\label{subfig:zipf_dist}]{
        \centering
        \includegraphics[width=1.\linewidth]{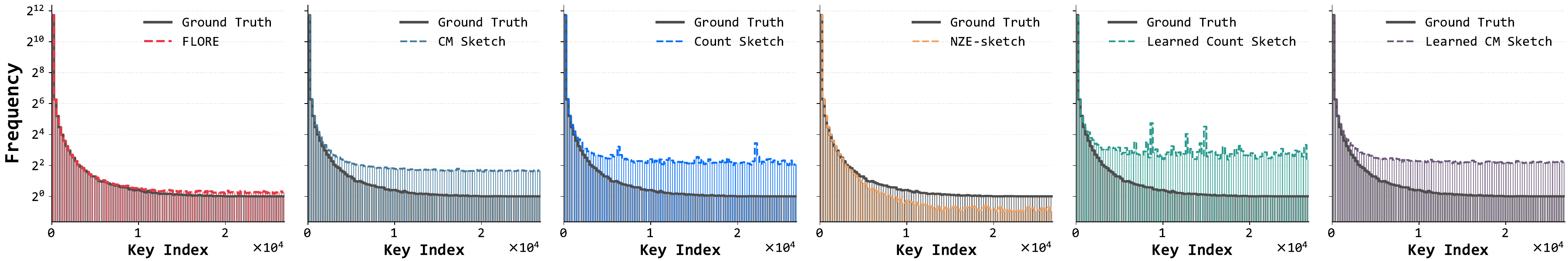}
    }
    \subfigure[{Pareto Distribution}\label{subfig:pareto_dist}]{
        \centering
        \includegraphics[width=1.\linewidth]{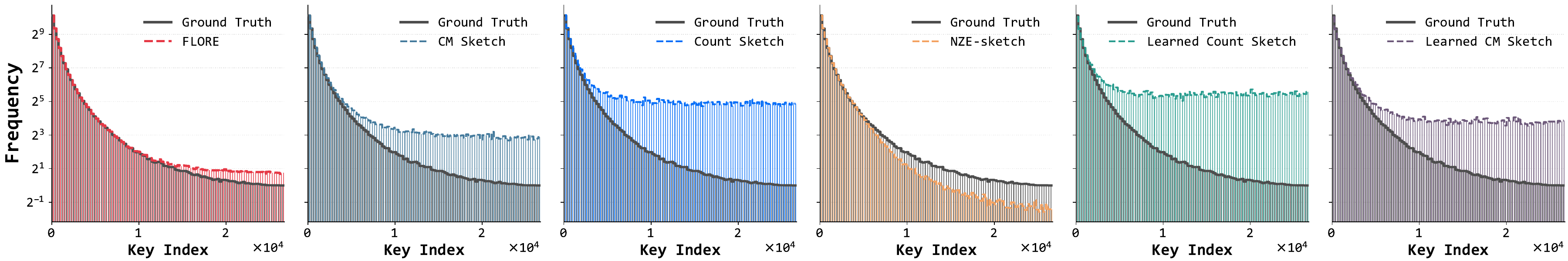}
    }
    \subfigure[{Exponential Distribution}\label{subfig:exponential_dist}]{
        \centering
        \includegraphics[width=1.\linewidth]{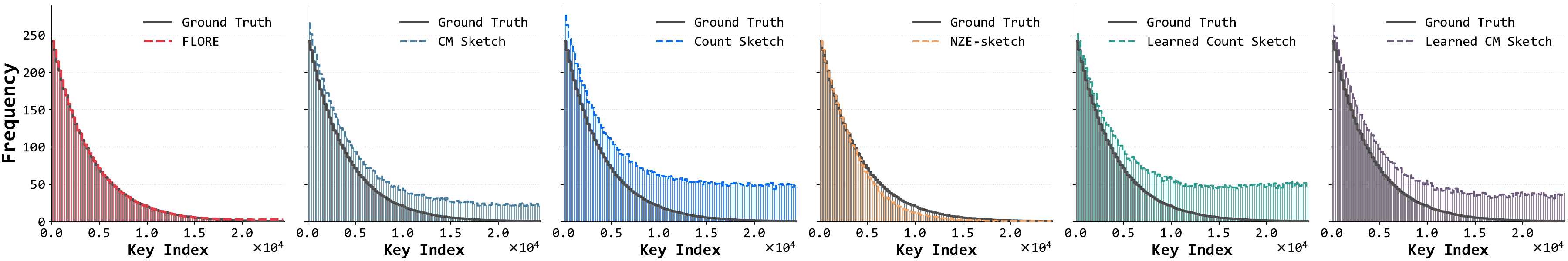}
    }
    \subfigure[{Log-Normal Distribution}\label{subfig:lognormal_dist}]{
        \centering
        \includegraphics[width=1.\linewidth]{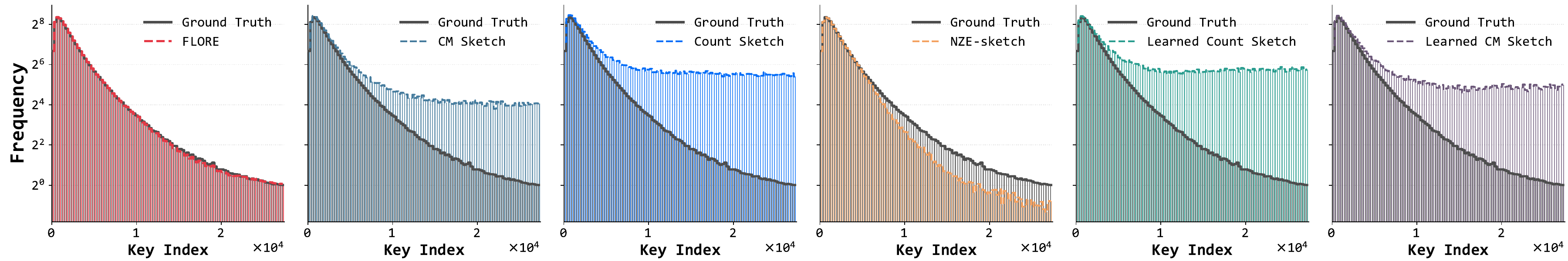}
    }
    \caption{Visualizations of sketch recovery under a fixed \textit{256KB} memory budget. From left to right, we present the recovery results of \texttt{FLORE}, CM, CS, NZE, LCM, and LCS, compared against the ground-truth sample distribution. To enhance readability, every 100 consecutive keys are aggregated in the x-axis.}
    \label{fig:syn_dist_256}
\end{figure}

\end{document}